\documentclass[11pt]{article}
\usepackage{amssymb,amsmath,amsfonts,amsthm,mathtools,color}

\usepackage{mathrsfs}
\usepackage{dsfont}
 \usepackage{bm} 

\usepackage{url} 

\usepackage[overload]{empheq}

\usepackage{comment} 
\usepackage[title,titletoc,header]{appendix}
\usepackage{graphicx}
\usepackage{subcaption}
\usepackage{float} 
\usepackage{dsfont} 

\usepackage{paralist}

\usepackage{cases}

\usepackage{tikz}
\usetikzlibrary{arrows,positioning,shapes.geometric}

\usepackage{algorithm}
\usepackage{algpseudocode}

\usepackage{multirow}

\graphicspath{ {./figures/} }

\usepackage{indentfirst}
\usepackage{multicol}
\usepackage{booktabs}
\usepackage{url}
\usepackage[outdir=./]{epstopdf} 

\usepackage[shortlabels]{enumitem}

\usepackage{hyperref}
\hypersetup{
    colorlinks=true, 
    linktoc=all,     
    linkcolor=blue,  
}
\numberwithin{equation}{section}

\newtheorem{theorem}{Theorem}[section]
\newtheorem{lemma}[theorem]{Lemma}
\newtheorem{proposition}[theorem]{Proposition}
\newtheorem{corollary}[theorem]{Corollary}
\newtheorem{assumption}{Assumption}

\theoremstyle{definition}
\newtheorem{definition}{Definition}[section]

\theoremstyle{remark}
\newtheorem{remark}{Remark}[section]

  \def\a{\alpha}

\def\ms{\medskip}

\def\cA{\mathcal{A}}

\def\cP{\mathcal{P}}

\def\cS{\mathcal{S}}

\def\diff{{\mathrm{d}}}
\def\d{{\mathrm{d}}}

\def\N{{\mathbb{N}}}

\def\R{{\mathbb R}}

\def\Z{{\mathbb{Z}}}

\newcommand{\citep}{\cite}
\newcommand{\citet}{\cite}

\newcommand{\lc}
{\mathrel{\raise2pt\hbox{${\mathop<\limits_{\raise1pt\hbox
{\mbox{$\sim$}}}}$}}}

\newcommand{\gc}
{\mathrel{\raise2pt\hbox{${\mathop>\limits_{\raise1pt\hbox{\mbox{$\sim$}}}}$}}}

\newcommand{\ec}
{\mathrel{\raise2pt\hbox{${\mathop=\limits_{\raise1pt\hbox{\mbox{$\sim$}}}}$}}}

\def\bb{\begin{equation}} \def\ee{\end{equation}}
\def\bbn{\begin{equation*}} \def\een{\end{equation*}}

\def\beqn{\begin{eqnarray}}  \def\eqn{\end{eqnarray}}

\def\beqnx{\begin{eqnarray*}} \def\eqnx{\end{eqnarray*}}

\def\bn{\begin{enumerate}} \def\en{\end{enumerate}}

\def\bd{\begin{description}} \def\ed{\end{description}}



\begin{document}

\title{Efficient Learning for Entropy-Regularized Markov Decision Processes via Multilevel Monte Carlo}

\author{
Matthieu Meunier\thanks{Mathematical Institute, University of Oxford, Oxford OX2 6GG, UK ({\tt matthieu.meunier@maths.ox.ac.uk,
 christoph.reisinger@maths.ox.ac.uk})} 
\and
Christoph Reisinger\footnotemark[1]
\and
Yufei Zhang \thanks{Department of Mathematics, Imperial College London,  London,  UK  ({\tt yufei.zhang@imperial.ac.uk})}
}

\date{}
\maketitle

\begin{abstract}

Designing efficient learning algorithms with complexity guarantees for Markov decision processes (MDPs) with large or continuous state and action spaces remains a fundamental challenge. 
We address this challenge for  entropy-regularized MDPs with Polish state and action spaces, assuming access to a generative model of the environment.

We propose a novel family of multilevel Monte Carlo (MLMC) algorithms that integrate fixed-point iteration with MLMC techniques and a generic stochastic approximation of the Bellman operator. We   quantify the precise  impact of the chosen approximate Bellman operator on the accuracy of the resulting MLMC estimator.
Leveraging this error analysis, we show that using a biased plain MC estimate for the Bellman operator results in quasi-polynomial sample complexity, whereas an unbiased randomized multilevel approximation 
of the Bellman operator
achieves \emph{polynomial} sample complexity in expectation. Notably, these  complexity bounds are \emph{independent of the   dimensions or cardinalities of the state and action spaces}, distinguishing our approach from existing algorithms whose complexities scale with the sizes of these spaces. We validate these theoretical performance guarantees    through numerical experiments.

\end{abstract}


%

\ms


%
\medskip

\section{Introduction}

Value-based reinforcement learning (RL) algorithms aim to estimate the optimal Q-function of a Markov decision process (MDP), which represents the minimal accumulated cost achievable from a given state-action pair \citep{Sutton1998}. 
Agents typically have access to a generative model of the environment, referred to as  an oracle, which takes a state-action pair as input and returns an instantaneous cost along with a next state. By interacting with the oracle, agents explore different actions and refine their strategies to minimize accumulated costs.

The sample complexity of an algorithm is defined as the total number of actions taken and oracle queries made to approximate the optimal Q-function. Since each oracle query is costly, designing efficient algorithms with low sample complexity is essential for reducing computational overhead. However, existing algorithms generally exhibit sample complexity that scales polynomially with the sizes of the state and action spaces, making them inefficient for large or continuous state-action spaces. 
To the best of our knowledge, no existing learning algorithm provides provable sample complexity guarantees for general MDPs \emph{with arbitrary (possibly continuous) state and action spaces}.

In this work, we focus on Monte Carlo (MC) sampling algorithms, which are a popular class of methods for estimating optimal value functions by computing empirical averages over sampled trajectories. 
Various MC sampling algorithms have been proposed for MDPs with \emph{finite action spaces} and arbitrary state spaces,
achieving quasi-polynomial or polynomial sample complexity in terms of the desired accuracy
(see, e.g., 
\citet{kearns2002sparse,grill2019planning, beck2023nonlinearmontecarlomethods}). However, these sample complexity guarantees  depend explicitly on the cardinality of the action space and  grow unbounded for large (particularly continuous) action spaces; see Section \ref{subsection:related-work} for more details.

This work addresses this gap in the context of entropy-regularized MDPs, where the objective is augmented with an entropy term. Unlike prior works such as the paper of \citet{grill2019planning}, which only considers finite action spaces, we allow both the state and action spaces to be general Polish spaces.  
Our key observation is that the Bellman operator of an entropy-regularized MDP involves integration over the action space. Leveraging this insight, we propose several MC algorithms that achieve provable quasi-polynomial or even polynomial 
sample complexity guarantees that are \emph{independent of the dimensions or cardinalities of the state and action spaces}.

\subsection{Outline of Main Results}
\label{subsection:background}

In this section, we provide a road map of the key ideas and
contributions of this work without introducing needless technicalities. The precise assumptions and
statements of the results can be found in Section \ref{section:problem-formulation}.

\paragraph{Entropy-Regularized MDPs.}

Consider an infinite horizon MDP $(\cS,\cA,P,c,\gamma)$,  
where the state space $\cS$ and action space $\cA$ are Polish (i.e., complete separable metric) spaces with possibly   infinite cardinality, $P\in \cP(\cS|\cS\times \cA)$ is the transition probability kernel, $c$ is a bounded cost function, and $\gamma\in [0,1)$ is the discount factor. 
Let $\mu \in \mathcal \cP(\cA)$ denote a reference probability measure and  $\tau>0$ denote a regularization parameter.   
For each  stochastic policy $\pi \in \cP(\cA|\cS)$ and $s\in \cS$, define the regularized value function by\footnote{The  fact that $\cS$ and $\cA$ are Polish spaces  allows for applying the Kolmogorov extension theorem to construct the unique probability measure associated with the kernel $P$ and a policy $\pi$,
ensuring that    \eqref{eq:value-function} is well-defined; see \citet[Section 2.1]{kerimkulov2023fisherrao}.} 
\begin{equation}
\label{eq:value-function}
V^{\pi}(s)= \mathbb E\left[\sum_{n=0}^\infty\gamma^n \Big(c(s_n,a_n) + \tau \operatorname{KL}(\pi(\cdot|s_n)|\mu)\Big)\right],
\end{equation}
where 
  $s_0=s$,
and for all $n\ge 0$, given the state $s_n$,
the action 
$a_n$ is sampled   according to the policy $\pi(\cdot|s_n)$, and the state  
transits to  $s_{n+1}$ according to the distribution  $P(\cdot|s_n,a_n)$.
The term 
$\operatorname{KL}(\pi(\cdot|s)|\mu) 
$ is the    Kullback--Leibler (KL) divergence  of $\pi(\cdot|s)$ with respect to $\mu$, defined as  
$
\operatorname{KL}(\pi(\cdot|s)|\mu)
\coloneqq \int_{A} \ln\frac{\mathrm{d} \pi(\cdot|s)}{\mathrm{d}\mu}(a) \pi(d a|s)  
$ if $\pi(\cdot |s)$ is absolutely continuous with respect to $\mu$,
and infinity otherwise.
 The optimal value function is then given by
\begin{equation*}
    \label{eq:V_star}
    V ^{\star}(s) = \inf_{\pi\in \cP(\cA|\cS)} V ^{\pi}(s), \quad s\in \cS.
\end{equation*}

By the dynamic programming principle  \citep[Appendix B]{kerimkulov2023fisherrao}, both the optimal  function $V^\star$ 
  and the optimal policy that minimizes \eqref{eq:value-function} are 
 given by 
 \begin{equation}
 \label{eq:optimal_value_policy}
        V ^{\star}(s) = - \tau \log \int_{\mathcal A} \exp \left( - \frac{Q ^{\star}(s,a)}{\tau}  \right) \mu (\diff a),
        \quad
        \pi^{\star}  (\diff a|s) = \exp\left(-\frac{ Q^{\star} (s,a)-V^{\star} (s)}{ \tau}\right)\mu(\diff a),
\end{equation}
  where   $Q^\star$ is the 
  optimal state-action value function, also known as the optimal Q-function.
  Moreover,
  one can show that $Q^\star$ is the 
  unique   solution to the following fixed-point equation in $B_b(\cS\times \cA)$:
      \begin{equation}
        \label{eq:optimal-Q-informal}
        Q ^{\star}(s,a)= c(s,a) + \gamma \int_{S} TQ ^{\star}(s ^{\prime}) P(\diff s ^{\prime} | s,a), \quad \forall (s,a)\in \cS\times \cA,
    \end{equation}
    where $T:B_b(\cS\times   \cA)\to B_b(\cS)$ is the soft-Bellman operator defined by
  \begin{equation}
    \label{eq:T-operator-def}
    TQ (s^\prime) =  -\tau \log \int_{\mathcal A} \exp \left( - \frac{Q(s^\prime, a^\prime)}{\tau}   \right) \mu (\diff a^\prime),
\end{equation}
  and $B_b(\cS\times \cA)$
  and $B_b(\cS )$ are the spaces of   bounded measurable functions on $\cS\times \cA$ and $\cS$, respectively. 
  The operator $T$
  is referred to as ``soft",
  following the terminology in  
\citet{van2019composing}, since it is a smooth approximation of the minimum operator.

In this paper,
motivated by the 
  identity \eqref{eq:optimal_value_policy}, we propose and analyze several MC estimators for  $ Q^\star $, using an oracle that generates state transition samples 
  from arbitrary state-action pairs
  and evaluates the corresponding   instantaneous  cost $c$, along with a sampler for the reference measure $ \mu $.

  \paragraph{A Simple Iterative MC Estimator.}

The fixed-point equation 
\eqref{eq:optimal-Q-informal} indicates that for a given initial guess   $Q_0$, 
the following iterates 
$(Q_n)_{n\in \N}$,
$\N\coloneqq  \{0,1,2, \ldots\}$,
given by  
\begin{equation}
    \label{eq:recursion-Q}
    Q_{n+1} (s,a) \coloneqq  c(s,a) + \gamma \int_{\mathcal S} TQ_n (s^\prime)  P(\diff s^\prime | s, a)
\end{equation}
converge  to $Q^\star$ as $n\to \infty$
\citep{kerimkulov2023fisherrao}. Replacing the integrals over $\cS$
 and $\cA$ with empirical averages over sampled data yields   a simple iterative MC estimator of $Q^\star$.
 
 More precisely,
 for any  $n, M, K\in \N^* \coloneqq \{1,2,\ldots \}$, define 
 for all $(s,a)\in \cS\times \cA$,
\begin{equation}
\label{eq:simple_iterate_intro}
    \begin{aligned}
    {Q}_{n,M,K}(s,a) &\coloneqq  c(s,a) + \frac{\gamma}{M} \sum_{i = 1}^M \hat T  {Q}_{n-1,M, K}  \left(S_{s,a}^{(n-1, i)}\right),\\  
    \hat T  {Q}_{n-1,M, K}  (s) &  \coloneqq    -\tau \log \frac{1}{K}\sum_{k = 1}^{K} \exp \left(-\frac{Q_{n-1,M,K}\left(s,A^{(n-1,k)}\right) }{\tau}  \right),
    \end{aligned}
\end{equation}
where $(S ^{(n-1,i)}_{s,a}) _{i=1}^M$ are independent samples from  $P(\cdot | s, a)$, and $(A^{(n-1,k)})_{k=1}^K$
are independent samples from $\mu$.
The estimator \eqref{eq:simple_iterate_intro}
adapts the estimator in \citet{kearns2002sparse} to the present entropy-regularized setting. 

The first main contribution of this work is
to analyze the sample complexity of the estimator defined in \eqref{eq:simple_iterate_intro}. 
In particular,
\begin{itemize}
\item We explicitly quantify the $L^2$ error of the estimator
${Q}_{n,M,\mathbf T_K}$ 
  in terms of $M$, $K$ and $n$
(Theorem \ref{thm:simple-iterative-MC}).
    Leveraging this error bound, 
    we prove  the estimator
  \eqref{eq:simple_iterate_intro}
  achieves accuracy $\varepsilon$ with a quasi-polynomial complexity  of the order 
$\varepsilon^{ -\kappa \log \varepsilon  }$ as $\varepsilon\to 0$, for some $\kappa>0$. 
  This error bound is independent of the cardinality of the action space,
  in contrast to the quasi-polynomial complexity bound in \citet{kearns2002sparse}.

\end{itemize}

The error estimate in Theorem \ref{thm:simple-iterative-MC} also indicates that the estimator \eqref{eq:simple_iterate_intro} cannot achieve polynomial sample complexity.
This is due to the $\mathcal O(M^{-1/2})$ approximation error for the expectation over $\cS$ at each iteration,
resulting in an overall sample complexity of at least
  $\mathcal O( M^{n})$
(Remark \ref{rmk:quasi-polynomial-simple-iterative}).
This motivates us to adopt the multilevel Monte Carlo (MLMC) technique, originally proposed by \citet{giles2015mlmc}, to achieve variance reduction.

\paragraph{MLMC Estimators.}

The MLMC estimators for $Q^\star$ are based on 
the observation that   for any $n\in \N$,
  the iterate  $Q_n$ defined by \eqref{eq:recursion-Q} admits     the following telescoping decomposition:  
$$
\begin{aligned}
Q _n(s,a) &= Q_1(s,a) + \sum_{l = 2}^{n} (Q_{l}(s,a) - Q_{l-1}(s,a)) \\
    &= c(s,a)+\gamma \int_\cS (T Q_0)(s') P(\d s'|s,a)+ \sum_{l = 1}^{n - 1} \gamma \int_\cS 
      (T Q_{l} -T Q_{l-1})(s') P(\d s'|s,a).
    \end{aligned}
$$ 
The  convergence of $(Q_l)_{l\ge 0}$
implies 
the difference $Q_{l} - Q_{l - 1}$ gets smaller as $l$ increases.
Hence by directly
estimating the difference $\int_\cS 
(T Q_{l} -T Q_{l-1})(s') P(\d s'|s,a) $ 
using sampled data,
 fewer samples are required at higher levels to achieve a fixed  overall accuracy,
 which subsequently  results in an improved sample complexity compared to the simple iterative MC estimator \eqref{eq:simple_iterate_intro}.

More precisely,     given $n\in \N$  and $M \in \N^*$, 
for each  $l = 0,\ldots, n$,
we approximate 
  $\int_\cS 
(T Q_{l} -T Q_{l-1})(s') P(\d s'|s,a) $
using  $M^{n-l}$ samples, 
where the number of samples decreases with respect to the level
  $l$. The resulting  MLMC estimator is given by
\begin{equation}
    \label{eq:MLMC-sketch}
    \begin{aligned}
        Q_{n, M} (s,a) & = c(s,a) + \gamma \sum_{i = 1}^{M^n} \frac{1}{M^n} \hat{T}Q_{0}\left(S^{(0,i)}_{s,a}\right)                                 \\
    &
    \quad + \gamma \sum_{l = 1}^{n-1}\sum_{i = 1}^{M^{n-l}}\frac{1}{M^{n-l}}\left[ \hat{T}Q_{ l, M} \left( S^{(l,i)}_{s,a} \right) - \hat{T} Q_{l - 1, M}\left(S^{(l,i)}_{s,a}\right) \right],
    \end{aligned}
\end{equation}
where $(S ^{(l,i)}_{s,a}) _{l,i}$ are independent samples from the distribution $P(\cdot | s, a)$, and $\hat{T}$ is a suitable stochastic approximation of the soft-Bellman operator $T$, 
which may differ from the plain MC approximation given in  \eqref{eq:simple_iterate_intro}.

The MLMC estimator \eqref{eq:MLMC-sketch} differs from the estimator in \citet{beck2023nonlinearmontecarlomethods}, which was developed specifically for unregularized MDPs with finite action spaces.
The key distinction is that \eqref{eq:MLMC-sketch} allows for a   general class of stochastic operators $\hat T$ to approximate 
the soft-Bellman operator $T$, a crucial feature for constructing an MLMC estimator that can accommodate general action spaces.
In contrast, \citet{beck2023nonlinearmontecarlomethods} fix  $\hat T$ as the (exact) Bellman operator for the unregularized MDP, which requires evaluating a given Q-function at all actions and taking the maximum over them. This approach does not scale well to large action spaces and is inapplicable to our setting with general action spaces.

The second main contribution of this work is to quantify the accuracy  of the MLMC estimator \eqref{eq:MLMC-sketch} for a broad class of stochastic operators $\hat T$ and to further optimize its sample complexity for specific choices of $\hat T$. In particular, 
\begin{itemize}
    \item 
    We establish a precise error bound for the MLMC estimator \eqref{eq:MLMC-sketch} in terms of the hyperparameters $n, M$, and the properties of
    the approximation  operator $\hat T$ (Theorem \ref{thm:global-MLMC-error}). The bound reveals that the Lipschitz continuity of the mapping $Q \mapsto \hat{T} Q$  influences error propagation in the recursive construction of the MLMC estimator, while the bias of $\hat T$ introduces an irreducible additive term in the final estimation error.
\item 

We refine the error bound for two specific choices of $\hat T$  and optimize the sample complexities of the resulting MLMC estimators.

The first choice of $\hat T$ is  the plain MC estimator \eqref{eq:simple_iterate_intro}, which serves as a biased approximation of the soft-Bellman operator   $T$ due to the   logarithm function in $T$ \eqref{eq:T-operator-def}. 
We prove that the corresponding  MLMC estimator \eqref{eq:MLMC-sketch}
  achieves a cubic reduction in sample complexity compared to the simple iterative MC estimator \eqref{eq:simple_iterate_intro}, 
  highlighting the advantage of the MLMC technique
  (Theorem \ref{theorem:error-complexity-MC}). 
  However, the  inherent bias in $\hat T$ causes the overall complexity to remain quasi-polynomial (Remark \ref{rmk:quasi-polynomial_MLMC}).
  
The second choice of $\hat{T}$  is an unbiased approximation of the soft-Bellman operator $T$, derived by applying the randomized multilevel Monte Carlo technique from \citet{blanchet2015unbiased} to the soft-Bellman setting; see Definition \ref{def:T-blanchet-glynn}.
We prove that the resulting MLMC estimator \eqref{eq:MLMC-sketch} achieves \emph{polynomial} sample complexity in expectation (Theorem \ref{theorem:error-complexity-BG}). The key step in the analysis   is establishing the Lipschitz continuity of the approximation operator $\hat T$ with respect to the input function $Q$ (Proposition \ref{proposition:contraction-prop-T-blanchet-glynn-l2}).

To the best of our knowledge, this is the first algorithm with polynomial sample complexity guarantee  for regularized MDPs with general state and action spaces. We emphasize that incorporating MLMC techniques into both the fixed-point iteration and the approximation of the soft-Bellman operator is crucial for achieving this polynomial complexity.
  
   \item 

   We examine the performance of the above two MLMC estimators in multi-dimensional linear quadratic control problems.
Our numerical results confirm that the MLMC estimator with a plain MC approximation of 
$T$
exhibits quasi-polynomial complexity, but remains stable even when a small sample size is used to approximate the inner integral in $T$.  In contrast, the MLMC estimator with the unbiased Blanchet--Glynn approximation achieves polynomial complexity with appropriately chosen hyperparameters, but may exhibit numerical instability as the number of fixed-point iterations increases.
\end{itemize}

We summarize in Table \ref{tab:theoretical_summary} the main results obtained for specific estimators.
 \begin{table}[htbp]
\centering
\renewcommand{\arraystretch}{1.3}
\begin{tabular}{|l|c|c|c|}
\hline
\textbf{Estimator} 
& \textbf{Result}
& \textbf{Error Rate} & \textbf{Complexity}  \\ 
\hline
\textbf{Plain MC (iterative)}
&
Theorem \ref{thm:simple-iterative-MC} 
& 
$\mathcal{O}\left(\frac{1}{\sqrt{M}} + \frac{1}{K} + (\gamma L)^n\right)$ &
$\varepsilon^{ \frac{-3 \log \varepsilon}{\log \gamma L}(1 + o(1))}$ 

\\
\hline
\textbf{MLMC  (biased)} & 
Theorem \ref{theorem:error-complexity-MC}
&
$\mathcal{O}\left((\Lambda_M)^n + \frac{1}{K}\right)$ &
$\varepsilon^{\frac{- \log \varepsilon}{\log \gamma L + \delta}(1 + o(1))}$
\\
\hline
\textbf{MLMC  (unbiased)} & 
Theorem \ref{theorem:error-complexity-BG}
&
$\mathcal{O}\left((\Lambda_M)^n\right)$ &
$\mathcal{O}\left(\varepsilon^{-\kappa}\right)$, $\kappa > 0$ 
\\
\hline
\end{tabular}
\caption{Comparison of theoretical properties of estimators
for entropy-regularized MDPs. Here, $M$ is the number of outer samples, $K$ is the number of inner samples for the soft-Bellman approximation,   $n$ is the number of fixed-point iterations,
and $\varepsilon$ is the   accuracy of the estimator. $\Lambda_M < 1$ is a constant depending on $M$, 
$\delta$ is any positive constant, 
and $L$ is a constant that depends on $c$ and $\tau$, for which we assume   $\gamma L < 1$ (see Remark \ref{rem:gammaL}).}
\label{tab:theoretical_summary}
\end{table}

\subsection{Most Related Works}
\label{subsection:related-work}

\paragraph{Monte Carlo Methods for MDPs.}

 In the realm of RL, Monte Carlo sampling has been employed to address the curse of dimensionality for MDPs with \emph{finite action spaces}, dating back to the seminal work of \citet{rust1997usingrandomization}. Algorithms with polynomial sample complexity for MDPs with \emph{finite state and action spaces} were later proposed in \citet{kearns2002near}. Monte Carlo methods became central for \emph{planning} in MDPs, where an agent seeks to estimate the optimal value function for a given state by querying a generative model. The influential paper of \citet{kearns2002sparse} introduced an MC planning algorithm (the sparse sampling algorithm) for MDPs with \emph{finite action spaces and arbitrary state spaces}, achieving quasi-polynomial sample complexity, where the complexity bound explicitly depends on the cardinality of the action space.
In special cases, such as deterministic dynamics \citep{hren2008optimistic} or finite support of the transition probability \citep{szorenyi2014optimistic, jonsson2020planning}, polynomial sample complexities in $\varepsilon^{-1}$ have been achieved. However,  these sample complexity guarantees become exponential
when the state space is infinite and the transitions are not restricted to a finite number of states. Recent works have sought to improve quasi-polynomial complexity to polynomial complexity by incorporating adaptive action selection in the context of regularized MDPs with finite action spaces \citep{grill2019planning}, or by using multilevel Monte Carlo techniques \citep{beck2023nonlinearmontecarlomethods}. Nonetheless, these complexity guarantees still depend explicitly on the cardinality of the action space, and they become infinite for continuous action spaces.

Our work addresses this gap by designing MC algorithms that achieve quasi-polynomial or even polynomial sample complexities for \emph{general (possibly continuous) action and state spaces}, filling a significant gap in the literature.

\paragraph{Entropy-Regularized MDPs.}

Entropy  regularizations have emerged as powerful tools in RL, offering significant benefits across various aspects of algorithm design and performance \citep{geist19atheory}. These techniques are known to stabilize learning \citep{ziebart2008maximum, reisinger2021regularity} and prevent the agent from being trapped in suboptimal policies too early \citep{fox2016tamingnoisereinforcementlearning}. This approach has given rise to popular deep RL methods such as soft actor-critic (SAC) \citep{haarnoja2018sac} and proximal policy optimization (PPO) \citep{schulman2017proximalpolicyoptimizationalgorithms}, which have become staples in modern RL applications. Such regularizations also facilitate the design and study of RL algorithms in continuous time 
(see, e.g., \citep{jia2022b, szpruch2024optimal}), as well as in the multi-agent/mean-field context \citep{cui2021approximately,anahtarci2023q,guo2022entropy}.
The study of entropy-regularized infinite-horizon MDPs with  general action and state spaces has led to important theoretical advancements, particularly in proving the convergence of policy gradient techniques \citep{cen2022fast, kerimkulov2023fisherrao}.
Our estimator shares similarities with techniques used in distributionally robust Q-learning \citep{liu2022distributionally,wang2023finite, wang2024samplecomplexityvariancereduceddistributionally}, where    suitable unbiased estimators 
are employed
to improve  state-of-the-art complexity in the tabular case \citep{wang2024samplecomplexityvariancereduceddistributionally}.

\paragraph{Multilevel Monte Carlo Methods for Fixed-Point Equations.}
Iterative MLMC has been applied  to approximate nonlinear equations with an underlying fixed-point structure 
(see, e.g., \citet{hutzenthaler2020overcoming, 
giles2019generalisedmultilevelpicardapproximations, hutzenthaler2021multilevel}
and the references therein for applications to PDEs,
  as well as \citet{szpruch2019multilevelSDE, hutzenthaler2022multilevel} for applications to McKean-Vlasov SDEs). 
Recently, \citet{beck2023nonlinearmontecarlomethods} introduced multilevel fixed-point iterations for learning the optimal Q-function of unregularized MDPs with a \emph{finite action space}, obtaining a polynomial complexity bound that explicitly depends on the cardinality of the action space.

It is important to note that our problem does not satisfy the assumptions required for the generalized MLMC estimators for fixed-point equations proposed by \citet{giles2019generalisedmultilevelpicardapproximations}.
Indeed, as emphasized earlier, achieving polynomial complexity requires incorporating MLMC techniques into both the fixed-point iteration and the approximation of the soft-Bellman operator.

Our problem also differs from the work of \citet{syed2023optimalrandomized}, which applies MLMC to estimate nested expectations with finite depth. Note that we aim to estimate the fixed point of \eqref{eq:optimal-Q-informal}, which cannot be expressed as a nested expectation of finite depth. More importantly, our estimator requires selecting the depth (corresponding to the level \(n\) in \eqref{eq:MLMC-sketch}) of nested expectations as a function of the error \(\varepsilon\), and we demonstrate that our estimator remains polynomial in \(\varepsilon^{-1}\). In contrast, \citet{syed2023optimalrandomized} provide a polynomial complexity result,  depending exponentially on the fixed depth. At the time of writing, we are not aware of any other model-free reinforcement learning techniques capable of achieving average polynomial complexity in arbitrary state and action spaces without structural assumptions on the underlying MDP.

 Finally, we would like to point out the difference between our MLMC estimator and Q-learning \citep{watkins1992q}. In Q-learning, the Q-values for all state-action pairs are stored (either in a look-up table for the tabular setting or using function approximation in the continuous setting), and they are updated iteratively, typically using one sample transition at a time. 
 The convergence guarantee of Q-learning typically requires finite state and action spaces. In contrast, our estimators compute the Q-value for a specific state-action pair and require sampling multiple transitions from the oracle starting from that pair.
 Our convergence results hold  for MDPs with general state and action spaces. 

\subsection{Notation and Paper Structure}
\label{subsection:notations}
We denote by  $\N = \{0, 1, 2, \cdots\}$    the set of all non-negative integers, and by 
$\N^*$   the set of all positive integers.   
For each measurable space 
$(\mathcal E, \mathcal F_{\mathcal E})$, 
we denote by  $B_b(\mathcal E)$   the set of all   bounded measurable functions $f: \mathcal E\rightarrow \R$, equipped with the  supremum norm $\|\cdot\|_\infty$. 
If $\mathcal E$ is a metric space, then the $\sigma$-algebra considered is the Borel $\sigma$-algebra $\mathcal F_{\mathcal E} = \mathcal B(\mathcal E)$. If $\mathcal E = \prod_{i \in I} \mathcal X_i$ where each $\mathcal X_i$ is endowed with a $\sigma$-algebra and $I$ is countable, then $\mathcal F_{\mathcal E}$ is the product $\sigma$-algebra. Similarly, we equip countable products of topological spaces with the product topology. For Polish spaces $(\mathcal X,\mathcal F_{\mathcal X}), (\mathcal Y, \mathcal F_{\mathcal Y})$, we denote by $\mathcal P (\mathcal X)$ the set of all probability measures on $\mathcal X$, and by $\mathcal P ( \mathcal X | \mathcal Y)$ the set of all Markov kernels  $\pi: \mathcal Y\times \mathcal F_{\mathcal X}   \rightarrow [0,1]$.

Throughout this paper, 
we denote the dependence of a constant on key quantities using the notation 
$C_{(\cdot)}$,
for example,  $C_{(\gamma)}$.

The rest of the paper is organized as follows:  
Section \ref{section:problem-formulation} presents the model assumptions, introduces the iterative MC estimator and various MLMC estimators rigorously, and states the main theoretical results regarding their error bounds and sample complexities.  
Section \ref{section:numerics} provides numerical experiments to illustrate the convergence and stability properties of the MLMC estimators.  
Section \ref{sec:proof-simple-iterative-MC} proves the error bound for the iterative MC estimator.  
Section \ref{section:analysis-of-the-error} proves the error bounds for the MLMC estimators.  
Section \ref{section:analysis-of-the-complexity} proves the sample complexity of the MLMC estimators.

\section{Main Results}
\label{section:problem-formulation}

This section summarizes 
 the model assumptions for  the  MDP, formulates various  MC estimators for the optimal 
$Q$ function, and presents their error bounds and sample complexities.

\subsection{Formulation of  Regularized MDPs}
\label{subsection:assumptions}

This section introduces the probabilistic framework for constructing the MC estimators of the regularized MDPs. 
Throughout this paper, 
we 
consider  an entropy-regularized MDP 
  $\mathcal M = (\mathcal S, \mathcal A, P, c, \gamma, \mu, \tau)$   as   in Section \ref{subsection:background}
  with the following assumption.
  
\begin{assumption}
    \label{assumption:data}
    $\mathcal S$ and $  \mathcal A$ are Polish spaces (i.e.,\   complete separable metric spaces),
    $P\in \cP(\cS|\cS\times \cA)$,
    $c\in B_b(\cS\times \cA)$,
    $\gamma\in [0,1)$,
    $\mu\in \cP(A)$
    and $\tau>0$.
Let
     $c_{\min}, c_{\max}\in [0,\infty)$ be such that $   c_{\min} \leq c(s,a) \leq c_{\max}$ for all $(s,a ) \in \mathcal S\times  \mathcal A$,
     and define 
      $\alpha \coloneqq c_{\min}/(1-\gamma) $
      and $\beta \coloneqq c_{\max}/(1-\gamma)$.
     
\end{assumption}

 Under Assumption \ref{assumption:data},  the optimal Q-function 
$Q^\star\in B_b(\cS\times \cA)   $ for the entropy-regularized MDP 
is    well-defined   and satisfies 
$ \alpha \leq Q^\star(s,a) \leq \beta
$
for all $(s,a)\in \cS\times \cA$.
Moreover, 
for any $Q_0\in B_b(\cS\times \cA)$,
define 
the following fixed point iterates 
 \begin{equation}
 \label{eq:fixed_point}
 Q_n(s,a)\coloneqq c(s,a)+\gamma \int_{S}
    (T Q_n)(s)  P(\diff s ^{\prime} \mid s,a), \quad n\in \N,
 \end{equation}
where 
$T:B_b(\cS\times \cA)\to B_b(\cS)$
is the soft-Bellman operator defined by
 \begin{equation}
 \label{eq:soft_bellman}
   (TQ)(s)\coloneqq -\tau\log \int_{A} \exp\left(- \frac{Q(s,a)}{\tau}\right) \mu(\d a).
\end{equation}
Then
$(Q_n)_{n\in \N}$ 
converges with a linear rate to $Q^\star$ in the space $B_b(\cS\times \cA)$ as $n\to \infty$; see Appendix B in the work of \citet{kerimkulov2023fisherrao}.

 In this paper, we 
 construct MC estimators 
for the optimal Q-function
using   sampled states and actions. 
To this end, let 
$(\Omega,\mathcal F, \mathbb P)$
be a generic probability space that 
 supports all  (countably many) independent random variables used in the   estimator,
 and 
let 
$\Theta = \bigcup_{n \in \N} \Z^n$
be the index set for these independent random variables. 
Note that although the practical implementation of our MLMC estimator involves only finitely many random variables, we define the estimators for all $\theta\in \Theta$
 through an induction process for mathematical convenience
 (see \eqref{eq:def-iterative-mc} and Definition \ref{def:general-MLMC-estimator}).

We assume access to an oracle that generates  independent samples  from the reference measure 
$\mu$
  and the transition kernel 
$P$. 
To ensure the conditional independence of samples from different oracle queries,
 we recall 
the ``noise outsourcing" lemma
\citep[Lemma 2.22]{kallenberg2002foundations}:
given the   kernel $P\in \mathcal P(\cS|\cS\times \cA)$, 
there exists a 
measurable function $f:\cS\times \cA\times [0,1]\to \cS$
such that if  
$U$ is a uniform random variable on $[0,1]$,
$f(s,a, U) $ has   distribution $P(\cdot|s,a)$
for all 
  $(s,a)\in \cS\times \cA$.

\begin{assumption}
    \label{assumption:random-variables}
   \begin{enumerate}[(i)]
       \item\label{item:A_sample}
       $(A ^{\theta})_{\theta \in \Theta}:\Omega\to \cA$ are independent random variables with distribution given by $\mu$.
       \item \label{item:S_sample}
              
       Let $(U^\theta)_{\theta\in \Theta}:\Omega\to [0,1]$ be independent uniform random variables   that are also independent of $(A ^{\theta})_{\theta \in \Theta}$.
       Define
       $S ^{\theta}_{s,a} \coloneqq f(s,a,U^\theta)$
       for all  $(s,a)\in \cS\times \cA$ and $\theta\in \Theta$,
       where $f: \cS\times \cA\times [0,1] \to \cS$
is a measurable function such that 
$S^\theta_{s,a}\coloneqq f(s,a,U^\theta)$ has     distribution $P(\cdot|s,a)$
for all $(s,a)\in \cS\times \cA$.

   \end{enumerate}

\end{assumption}

Assumption  \ref{assumption:random-variables}\ref{item:A_sample}
asserts that one can   sample  from 
the reference measure $\mu$.
This assumption holds    
  for commonly used reference measures,  such as the uniform distribution  \citep{mei2020global, reisinger2021regularity} and Gaussian distributions  
\citep{giegrich2024convergence}.

Assumption  \ref{assumption:random-variables}\ref{item:S_sample}
requires  
that 
the underlying randomness
of sampled state variables 
$(S^\theta_{s,a})_{\theta\in\Theta}$ is represented by some hidden uniform random variables 
$(U^\theta)_{\theta\in \Theta}$.
This   explicit
 representation of the noise in the transition kernel  $P$ ensures
 that for all $(s,a)\in \cS\times \cA$, the samples 
       $(S ^{\theta}_{s,a})_{\theta \in \Theta} $   
       are mutually independent and also independent    from other sources of randomness in our estimator. 
It also ensures a regular conditional probability for our MC estimators, which helps mitigate some measure-theoretical challenges when dealing with continuous state and action spaces.

\subsection{A Simple Iterative MC Estimator  and Its Sample Complexity} 
\label{sec:simple_MC}

We first propose  a simple iterative MC estimator of $Q^\star$, in the spirit of  \citet{kearns2002sparse}. 
The estimator is based on a plain  MC approximation  of the soft-Bellman operator $T$.

\begin{definition}
    \label{def:T-plain-monte-carlo}
    Let $(A ^{\theta})_{\theta \in \Theta}$ be the random variables  in Assumption \ref{assumption:random-variables}. For each $K \in \N$,
    we define the operators $\mathbf T_K =(\hat T^\theta_K)_{\theta\in \Theta}$ 
    such that for all  
    $\theta\in \Theta$ and 
    $Q \in B_b(\mathcal S \times \mathcal A)$, 
    \begin{equation}
        \label{eq:T-plain-monte-carlo}
        \hat{T} ^{\theta}_{K}Q(s) \coloneqq   -\tau \log \frac{1}{K}\sum_{k = 1}^{K} \exp \left(-\frac{Q(s,A^{(\theta,k)}) }{\tau}  \right),\quad s\in \cS.
    \end{equation}

\end{definition}

The estimate of $Q^\star$
is derived by simply 
replacing the operator $T$ in \eqref{eq:fixed_point} by operators of the family
$\mathbf T_K$. 
More precisely, fix  
an initial guess  $Q_0 \in B_b(\cS \times \cA)$ 
such that $\alpha \leq Q_0(s,a) \leq \beta$ for all $(s,a)$.
Define 
the family $({Q}_{n,M,\mathbf T_K}^\theta)_{n\in \N, M\in \N^*,\theta\in \Theta}$ iteratively 
such that for all $(s,a)\in \cS\times \cA$ and $\theta\in \Theta$, 
\begin{equation}
    \label{eq:def-iterative-mc}
    \begin{aligned}
    {Q}_{0,M,\mathbf T_K}^\theta (s,a) &= Q_0(s,a),  
    \\
    {Q}_{n,M,\mathbf T_K}^\theta(s,a) &= c(s,a) + \frac{\gamma}{M} \sum_{i = 1}^M T^{(\theta, i)} {Q}_{n-1,M,\mathbf T_K}^{(\theta, i)} \left(S_{s,a}^{(\theta, i)}\right), \quad   \forall n \geq 1.
    \end{aligned}
\end{equation}
For each $n\in \N^*$,
define the error of 
${Q}_{n,M,\mathbf T_K}^\theta$
by 
$$
E_{n,M,K } = \sup_{(s,a)\in \cS\times \cA} \left( \mathbb E \left[ \left({Q}_{n,M,\mathbf T_K}^\theta(s,a) - Q^\star(s,a) \right)^2 \right]\right)^{1/2},
$$
and define the sample complexity
$\mathfrak C_{n,M, K}$ of the estimator $ {Q}_{n,M,\mathbf T_K}^\theta$
  as the total number of random variables required to evaluate
$ {Q}_{n,M,\mathbf T_K}^\theta$. Notice that $E_{n,M,K}$ is independent of $\theta$ since $(Q^\theta_{n,M,\mathbf T_K}(s,a) - Q^\star(s,a))_{\theta \in \Theta}$ are identically distributed.

The following theorem quantifies the error in terms of $M, n , K$ 
and optimizes the sample complexity
of ${Q}_{n,M,\mathbf T_K}^\theta$ with a given accuracy.
Recall 
  that $\alpha =  (1 - \gamma)^{-1}c_{\min}$ and $ \beta =   (1 - \gamma)^{-1}c_{\max}$.

\begin{theorem}
\label{thm:simple-iterative-MC}
Suppose    Assumptions \ref{assumption:data} and \ref{assumption:random-variables} hold. 
Let $L \coloneqq \exp \left( \tau^{-1}(\beta - \alpha)\right)$ and 
 assume that $\gamma L<1$.
  Then for all $n, M, K \in \N^*$,  
    \begin{equation}
\label{eq:final-error-simple-iterative-mc}
        E_{n,M,K } \leq \frac{\gamma \sqrt{C}}{\sqrt{M}(1 - \gamma L)} + \frac{\gamma (L^\prime)^2}{2 \tau K(1 - \gamma L)} + (\gamma L)^n  \left\|Q_0 - Q^\star\right\|_{\infty},
    \end{equation}
    with
    $
C = \left(\beta-\alpha \right)^2$ 
and 
$ 
L^\prime = \tau (L - 1)$. Moreover, the  corresponding  sample complexity
$\mathfrak C_{n,M, K}$ of $Q^\theta_{n,M,\mathbf T_K}$
is 
 $M^nK^n$.

 In particular, for each 
      $\varepsilon \in (0,1)$,
     by setting 
    \begin{equation}
        \label{eq:n-M-K-simple-iterative-MC}
        n_\varepsilon= \left\lceil \frac{\log \varepsilon - \log (3\left\|Q_0 - Q^\star\right\|_{\infty})}{\log \gamma L}\right\rceil, \quad M_\varepsilon = \left\lceil \frac{9\gamma^2C}{(1-\gamma L)^2\varepsilon^2}\right\rceil, \quad K_\varepsilon = \left\lceil  \frac{3\gamma (L^\prime)^2}{2 \tau(1-\gamma L)\varepsilon}\right\rceil,
    \end{equation}
   it holds that  $E_{n_\varepsilon,M_\varepsilon,K_\varepsilon} \leq \varepsilon$
   for all $\varepsilon\in (0,1)$,
   and 
$\mathfrak C_{n_\varepsilon,M_\varepsilon,K_\varepsilon} = \varepsilon^{ \frac{-3 \log \varepsilon}{\log \gamma L}(1 + o(1))}$ as $\varepsilon\to 0$,
where $o(1)$ denotes a term that vanishes as $\varepsilon\to 0$.
   
\end{theorem}

 The proof of Theorem \ref{thm:simple-iterative-MC} is given in Section \ref{sec:proof-simple-iterative-MC}.
 
 \begin{remark}[Role of regularization]
 Theorem \ref{thm:simple-iterative-MC}
 shows that 
 for regularized MDPs, 
 the  estimator 
\eqref{eq:def-iterative-mc}
achieves accuracy $\varepsilon$ 
with a quasi-polynomial complexity  independent of the cardinalities of the  action spaces.
This stands in contrast to the MC estimator  for unregularized MDPs   in \citet{kearns2002sparse}, where the quasi-polynomial complexity bound  depends explicitly on the action space cardinality and becomes infinity for continuous action spaces.
This improvement arises because entropy regularization leads to a smoothed Bellman operator, eliminating the need to enumerate all actions and compute the maximum over them, as required in the unregularized case.
 \end{remark}

 \begin{remark}[Condition $\gamma L <1$]
 \label{rem:gammaL}
The extra assumption $\gamma L <1$ made in Theorem \ref{thm:simple-iterative-MC} 
holds for   
a sufficiently small
discount factor $\gamma$,
a sufficiently flat cost $c$,
or a sufficiently large regularization parameter $\tau$.
Indeed, 
  for any given bounded cost $c$ and regularization parameter $\tau$, it is satisfied if the discount factor $\gamma$ is sufficiently small. Conversely, for any $\gamma<1$, it is satisfied if either $\tau$ is sufficiently large for given $c$, or $c$ is sufficiently flat
(i.e., $c_{\max}-c_{\min}$ sufficiently small) for a given $\tau$. 
  \end{remark}
  
 \begin{remark}[Error decomposition]
  \label{rmk:quasi-polynomial-simple-iterative}  
The error bound in \eqref{eq:final-error-simple-iterative-mc} quantifies the contributions of three distinct error sources. The first term represents the variance associated with approximating the expectation over the state space \( \mathcal{S} \), the second term accounts for the bias in approximating the soft-Bellman operator \( T \), and the third term reflects the error introduced by the fixed-point iteration.

Theorem \ref{thm:simple-iterative-MC} indicates that 
the simple iterative estimator 
${Q}_{n,M,\mathbf T }^\theta$   cannot achieve a polynomial sample complexity, even if the soft-Bellman operator $T$  can be evaluated exactly.
This is due to the \(\mathcal{O}(M^{-1/2})\) approximation error for the expectation over \( \mathcal{S} \) at each iteration combined with a sample complexity of at least \(\mathcal{O}(M^n)\).
Since both 
$M$ and  $n$  must increase to achieve a higher accuracy, the simple iterative estimator exhibits super-polynomial complexity.
 \end{remark}

In the sequel, we employ   the multilevel Monte Carlo (MLMC) technique, originally proposed in \citet{giles2015mlmc}, to achieve a variance reduction, resulting in an estimator with an average polynomial sample complexity when combined with an unbiased estimation of \(T\).

\subsection{MLMC Estimators and Their Sample Complexities} 

This section uses the MLMC technique  
to    
design more sample efficient estimators. Specifically, observe that for any $n\in \N$,
  the iterate  $Q_n$ defined by \eqref{eq:fixed_point} admits     the following telescoping decomposition:  
$$
\begin{aligned}
Q _n(s,a) &= Q_1(s,a) + \sum_{l = 2}^{n} (Q_{l}(s,a) - Q_{l-1}(s,a)) \\
    &= c(s,a)+\gamma \int_\cS (T Q_0)(s') P(\d s'|s,a)+ \sum_{l = 1}^{n - 1} \gamma \int_\cS 
      (T Q_{l} -T Q_{l-1})(s') P(\d s'|s,a).
    \end{aligned}
$$ 
An MLMC estimator 
of $Q^\star$
can be constructed by estimating the difference 
$$\int_\cS 
(T Q_{l} -T Q_{l-1})(s') P(\d s'|s,a),
$$ 
using sampled data, leading to a variance reduction compared to the standard iterative MC estimator analyzed in Section \ref{sec:simple_MC}.

\subsubsection{Formulation of General MLMC Estimators}

To formulate the MLMC estimator,
we first introduce a generic stochastic approximation \( \mathbf{T} \) of the soft-Bellman operator $T$ defined as follows.

\begin{definition}[Admissible Stochastic Operators]
    \label{def:admissible-stochastic-operator}
    Let   Assumptions \ref{assumption:data} and \ref{assumption:random-variables} hold,
    and let 
    $\mathbf{T} = (T^\theta)_{\theta \in \Theta}$
    be a family of stochastic operators with 
      $T^\theta: B_b(\cS\times \cA) \times \Omega
    \to B_b(\cS) $ for all $\theta\in \Theta$.
    We say that  $\mathbf{T} = (T^\theta)_{\theta \in \Theta} $  is 
     \textit{admissible} if there exists a measurable function $\Phi: \R^{\Z} \times \N \rightarrow \R$ and a family of i.i.d.~$\N$-valued random variables $(K^\theta)_{\theta \in \Theta}$, independent of $(A^\theta, U^\theta)_{\theta \in \Theta}$, such that for all  $\theta\in \Theta$,  $Q  \in  B_b(\mathcal S \times \mathcal A)$,
     and $\omega\in \Omega$,
    \begin{equation*}
        \label{eq:admissible-stochastic-operator}
        \left(T^\theta (Q,\omega)\right) (s) = \Phi \left(\left(Q(s,A^{(\theta,k)}(\omega))\right)_{k \in \Z}, K^\theta(\omega)\right),
        \quad \forall s\in \cS.
    \end{equation*}
    In the sequel, we omit $\omega$ and write $T^\theta (Q,\omega)$
    as $T^\theta Q$ for simplicity. 
   Note that  an admissible $\mathbf{T} $ ensures  that 
   $T^\theta Q: \mathcal S \times \Omega \rightarrow \R$ is measurable for all $Q \in B_b(\mathcal S \times \mathcal A)$ and $\theta\in \Theta$.
\end{definition}

Definition \ref{def:admissible-stochastic-operator} 
provides a unified framework for various stochastic approximations of the soft-Bellman operator considered in this paper.
These approximations may be biased estimators due to the logarithm function in the soft-Bellman operator \eqref{eq:soft_bellman}, and this approximation bias is reflected in the error bound of the 
MLMC estimator (Theorem \ref{thm:global-MLMC-error}).
The variable 
$(K^\theta)_{\theta\in \Theta}$ represents
    the number of samples $(A^{\theta})_{\theta\in \Theta}$ used to approximate the integral over $\cA$, which can be either deterministic or stochastic.
    The  biased plain MC  estimator,
    as defined in Definition \ref{def:T-plain-monte-carlo},   corresponds to 
      $K^\theta \equiv K \in \N^*$,
      while   
   an unbiased  estimator involving   stochastic 
   $(K^\theta)_{\theta\in \Theta}$ will be introduced   in Section \ref{subsection:unbiased-monte-carlo}. 

Given the stochastic operators $\mathbf{T} $ in 
    Definition \ref{def:admissible-stochastic-operator},
    we introduce 
  the   MLMC estimator of $Q^\star$ for the optimal Q-function. 
For each $a \leq b$, we define the truncation function $\rho_a ^{b}: \R\to \R$ by $\rho_a ^{b}( x)= \min(\max(x,a),b)$ for all $x\in \R$.

\begin{definition}[General MLMC Estimator]
    \label{def:general-MLMC-estimator}
    Suppose Assumptions \ref{assumption:data} and \ref{assumption:random-variables} hold,
    and let 
       $\mathbf T = (T^{\theta})_{\theta \in \Theta}$ be an admissible family of stochastic operators (c.f.~Definition \ref{def:admissible-stochastic-operator}).
       Recall that $\alpha = (1-\gamma)^{-1} c_{\min}$,
       $\beta = (1 - \gamma)^{-1} c_{\max}$.  
    Let $Q_0 \in B_b(\mathcal S \times \mathcal A)$ be such that $ \alpha \leq Q_0 \leq \beta$, 
    and define the family of estimators 
    $({Q}_{n,M,\mathbf T}^{\theta})_{n\in \N,M\in \N^*,\theta\in \Theta}$ recursively as follows:
    let $Q ^{\theta}_{0,M, \mathbf T} \coloneqq  Q_0$
    for all $M\in \N^*,\theta\in \Theta$,
    and 
    for all $n\ge 1, M\in \N^*, \theta\in \Theta, (s,a) \in \cS \times \cA$, let
\begin{equation}
        \label{eq:MLMC-estimator-approx-T}
        \begin{aligned}
            \hat{Q}^\theta_{n,M, \mathbf T} (s,a) \coloneqq  & c(s,a) + \gamma  \frac{1}{M^n} \sum_{i = 1}^{M^n} T^{(\theta, 0, i)}Q_{0}\left(S^{(\theta, 0, i)}_{s,a}\right) \\
            & + \gamma \sum_{l = 1}^{n - 1}
            \frac{1}{M^{n-l}}
            \sum_{i = 1}^{M^{n - l}}\left[ T^{(\theta,l,i)}Q_{ l, M, \mathbf T}^{(\theta, l, i)}\left( S^{(\theta, l, i)}_{s,a}\right) - T^{(\theta,l,i)} Q_{l - 1,M, \mathbf T}^{(\theta, -l, i)}\left(S^{(\theta, l, i)}_{s,a}\right) \right],
        \end{aligned}
    \end{equation}
and define $ {Q}^\theta_{n,M, \mathbf T}(s,a) $ by
$
        Q ^{\theta}_{n,M, \mathbf T}(s,a)  \coloneqq \rho_{\alpha} ^{\beta} \left( \hat{Q} ^{\theta}_{n,M, \mathbf T}(s,a) \right).
$
\end{definition}

\begin{remark}[Variance reduction]
\label{rmk:Q_estimator}
The estimator  \eqref{eq:MLMC-estimator-approx-T}
achieves a   variance reduction 
by evaluating 
$ Q_{l,M,\mathbf T}^{(\theta,l,i)}$ and $ Q_{l-1,M,\mathbf T}^{(\theta,-l,i)}$
at the same state-action sample pairs
for   levels $l\in \{1,\ldots, \\ n-1\}$, as indicated by the common superscripts of   $T^{(\theta, l,i)}$ 
   and  $S_{s,a}^{(\theta,l,i)}$. 
   This leverages the asymptotic convergence of $Q_{l,M,\mathbf T}^{(\theta,l,i)}-Q_{l-1,M,\mathbf T}^{(\theta,-l,i)}$ and enables the use of a smaller sample size at higher levels, with $M^{n-l}$   decreasing exponentially in $l$. In contrast, standard Monte Carlo estimators for MDPs evaluate Q-functions using different state-action samples and maintain the same sample size across all iterations (see, e.g., \citet{kearns2002sparse}).

 However,  
 for any given  $(s,a)\in \cS\times \cA$,
 the values 
 $Q_{l,M,\mathbf T}^{(\theta,l,i)}(s,a)$
 and 
 $Q_{l-1,M,\mathbf T}^{(\theta,-l,i)}(s,a)$ are defined using independent samples and can therefore be evaluated in parallel. 
 This can be seen from 
 the different superscripts in  
 $Q_{l,M,\mathbf T}^{(\theta,l,i)}$
 and 
 $Q_{l-1,M,\mathbf T}^{(\theta,-l,i)}$.
 We refer the reader to Lemma \ref{lemma:theta-is-just-an-index} for a detailed account of the role of $\theta$.

\end{remark}

  To implement the   MLMC estimator recursively,
  let   $Q_0$
   be the initial guess for $Q^\star$,
  and $T_{\mathrm{approx}}$ be a procedure  for approximating the soft-Bellman operator,
using samples drawn from the measure  $\mu$ (cf.~Definition \ref{def:admissible-stochastic-operator}).
Using $T_{\mathrm{approx}}$,
we define  
the procedure  $DT_{\mathrm{approx}}$
 for approximating the difference of the soft-Bellman operator evaluated at two different Q-functions in \eqref{eq:MLMC-estimator-approx-T}, 
 which applies 
 $T_{\mathrm{approx}}$ to 
 evaluate  both Q-functions  at the same state-action samples
 to ensure variance reduction (see     Remark \ref{rmk:Q_estimator}).
The   pseudocode for implementing the MLMC estimator is then presented in Algorithm \ref{alg:MLMC}.

\begin{algorithm}[H]

\caption{General MLMC Estimator for Reguarlized MDPs}
\label{alg:MLMC}
\begin{algorithmic}[1]
\Function{$Q_{\rm MLMC}$}{$n,s,a$}
    \If{$n =0 $}
          \State{$\hat{Q} \gets Q_0(s,a)$}
    \Else
     \State{$\hat{Q} \gets  {c(s,a)} $}
      \For{$l = 0, 1, \cdots, n - 1$}

        \State{draw  independent samples
        $(S_i)_{i=1}^{M^{n-l}}$
        from $P(\cdot | s, a)$}
        \If{$l = 0$}
        \State{$\hat{Q} \gets \hat{Q} + \frac{\gamma}{M^{n-l}} 
        \sum_{i=1}^{M^{n-l}}T_{\mathrm{approx}} \left( 
         Q_{\rm MLMC}(l,\cdot,\cdot), S_i\right)$}
        \Else
        \State{$\hat{Q} \gets \hat{Q} + \frac{\gamma}{M^{n-l}}  \sum_{i=1}^{M^{n-l}} DT_{\mathrm{approx}} 
         \left( 
         Q_{\rm MLMC}(l,\cdot,\cdot),
         Q_{\rm MLMC}(l-1,\cdot,\cdot),
         S_i\right) $}
        \EndIf
        \EndFor
        
    \EndIf
   \State  \Return 
    $\min\left(\max\left(\frac{c_{\min}}{1-\gamma},\hat Q\right),\frac{c_{\max}}{1 - \gamma}\right)$
\EndFunction
\end{algorithmic}
\end{algorithm}

 \subsubsection{Error Bounds of General MLMC Estimators}

This section quantifies the error of an MLMC estimator  in Definition \ref{def:general-MLMC-estimator}, assuming that the    approximation operators 
 $\mathbf T = (T^{\theta})_{\theta \in \Theta}$ satisfy suitable boundedness and Lipschitz conditions.
 In the sequel,
 for a given random variable $X:\Omega\to \R$,
 we denote by $\|X\|_{L^2}$
its $L^2$-norm under  the measure $\mathbb P$.

\begin{assumption}
    \label{assumption:lipschitz-condition-l2-stochastic-operator}
    Recall that $\alpha =  (1 - \gamma)^{-1}c_{\min}$, $ \beta =  (1 - \gamma)^{-1} c_{\max}$.  It holds that:
    \begin{enumerate}[(i)]
        \item\label{item:boundedness} 
        For all measurable functions  
$Q: \mathcal S \times \mathcal A \times \Omega \rightarrow [\alpha, \beta]$, 
        $\alpha \le c(s,a) + \gamma\mathbb E T^\theta Q(s) \le \beta $ for all $(s,a) \in \mathcal S \times \mathcal A$.
        \item\label{item:lipschitz-property} There exists   $L> 0$, depending on $\alpha,\beta$ and $\tau$, such that for all 
            bounded measurable functions
             $Q_1, Q_2 : \mathcal S \times \mathcal A \times \Omega \rightarrow \R$
         and 
            random variable
         $S:\Omega\to \cS$                  satisfying 
         \begin{itemize}
             \item 
             for almost sure $\omega\in \Omega$,
                  $\alpha \leq Q_i(s,a, \omega)  \leq \beta$ for all 
                  $i\in \{1,2\}$ and $(s,a) \in \cS \times \cA$,
            \item $S$ follows the distribution  $P(\cdot | s^\prime, a^\prime)$
        for some   $(s^\prime, a^\prime) \in \mathcal S \times \mathcal A$,   
                \item  
for all $\theta \in \Theta$ and $i\in \{1,2\}$, $(Q_i(S,A^{(\theta,k)}))_{k \in \Z}$ are identically distributed,     
                with the random variables $(A^\theta)_{\theta\in \Theta}$ defined in Assumption \ref{assumption:random-variables}, 
              and   $\left(Q_i(S,A^{(\theta,k)})\right)_{ k \in \Z}$ are independent from  the random variable $K^\theta$ 
              given in Definition \ref{def:admissible-stochastic-operator},
         \end{itemize}
we have  
              \begin{equation}
                  \label{eq:lipschitz-condition-l2-stochastic-operator}
                  \left\| T^\theta Q_1(S) - T^\theta Q_2(S) \right\|_{L^2} \leq  L \left\|Q_1\left(S,A^{(\theta, 1)}\right) - Q_2\left(S,A^{(\theta,1)}\right) \right\|_{L^2},
                  \quad \forall 
                  \theta \in \Theta.
              \end{equation}
    \end{enumerate}
\end{assumption}

Assumption \ref{assumption:lipschitz-condition-l2-stochastic-operator}
states that 
the approximation operator 
$T^\theta$ preserves bounded functions when taking expectation,
and is Lipschitz continuous in the $\|\cdot\|_{L^2}$ norm.
These properties 
allow for controlling the rate of error propagation in the recursive construction of the MLMC estimator.
 One can easily show that 
the (exact) soft-Bellman operator
satisfies Assumption \ref{assumption:lipschitz-condition-l2-stochastic-operator} (see  
   Lemmas \ref{lemma:boundedness-of-iterates} and   \ref{lemma:contraction-prop-T-l2}).
We show that both the biased  plain Monte Carlo estimator  
(Lemma \ref{lemma:contraction-prop-T-hat-l2})
  and the unbiased estimator     (Proposition \ref{proposition:contraction-prop-T-blanchet-glynn-l2})
  for the soft-Bellman operator 
  satisfy Assumption \ref{assumption:lipschitz-condition-l2-stochastic-operator}.

 The following theorem quantifies the error     of  the general MLMC estimator $Q^\theta_{n,m,\mathbf T}$,
 for $n\in \N$ and $M\in \N^*$,
 in terms of the  following  $L^2$-norm:
\begin{equation}
    \label{eq:definition-error-l2}
    \begin{aligned}
        E_{n,M, \mathbf T} & \coloneqq \sup_{(s,a) \in \mathcal S \times \mathcal A} \left\| Q^\theta_{n,M, \mathbf T}(s,a) - Q^\star(s,a) \right\|_{L^2}.
    \end{aligned}
\end{equation}
Notice that the error $E_{n,M, \mathbf T}$ is independent of $\theta$, since it only depends on the distributional properties of $Q^\theta_{n,M,\mathbf T}$. In particular, throughout the paper, we will often specialize the expressions only involving the distributional properties of $Q^\theta_{n,M,\mathbf T}$ (such as moments) by taking $\theta = 0$.

\begin{theorem}
    \label{thm:global-MLMC-error}

Suppose Assumptions \ref{assumption:data},
  \ref{assumption:random-variables}, 
  and \ref{assumption:lipschitz-condition-l2-stochastic-operator}
  hold. 
  Let $\left(Q^\theta_{n,M,\mathbf T}\right)_{\theta\in \Theta,n\in \N,M\in \N^*}$ be the MLMC estimators given in Definition \ref{def:general-MLMC-estimator}. 
  Assume that $\gamma L < 1$
  with the constant   $L$ in Assumption \ref{assumption:lipschitz-condition-l2-stochastic-operator},
  and   $M \in \N^*$ satisfies  
  \begin{equation}
    \label{eq:condition-M}
      \gamma L + \frac{1 + \tilde{\gamma} L}{\sqrt{M}} + \frac{\sqrt{\tilde{\gamma} - \gamma}}{M^{1/4}} < 1,
  \end{equation}
with  
       $ \tilde{\gamma} \coloneqq (1 + \max_{1\le k \leq n} \rho_{k,M}) \gamma$ and
    \begin{equation}
        \label{eq:def-rho}
      \rho_{k,M} \coloneqq \sup_{(s,a) \in \cS \times \cA}\max \left\{ \mathbb P \left(\hat{Q}^0_{k,M,\mathbf T}(s,a) < \alpha \right), \mathbb P \left(\hat Q^0_{k,M,\mathbf T} (s,a) > \beta \right)\right\}^{\frac{1}{2}}.
    \end{equation}
  Then  for all $n \in \N$,
    \begin{equation}
        \label{eq:final-MLMC-error}
        E_{n,M,\mathbf T} \leq 
             \frac{3}{2}\left(\max\left( \| Q_0 - Q^\star\|_{\infty},  \tilde{\gamma} \| \sigma_{\mathbf T} \|_{\infty} \right) \left[\gamma L + \frac{1 + \tilde{\gamma} L}{\sqrt{M}} + \frac{\sqrt{\tilde{\gamma} - \gamma}}{M^{1/4}} \right]^n + \frac{\gamma \|\delta_{\mathbf T} \|_{\infty}\sqrt{M} }{\sqrt{M} - \Lambda} \right),
    \end{equation}
     where 
        $\Lambda \coloneqq  \frac{1}{2}\left(1 + L(\tilde{\gamma} + \gamma \sqrt{M}) + \sqrt{\left(1 + L(\tilde{\gamma} + \gamma \sqrt{M})\right)^2 + 4 (\tilde{\gamma} - \gamma)\sqrt{M}}\right)$,
    and
    $\sigma_{\mathbf T},  \delta_{\mathbf T}:\cS\times \cA\to \R$
    are  defined by   \begin{equation}
        \label{eq:def-sigma-t}
        \sigma_{\mathbf T}(s,a)  \coloneqq \mathrm{Var} \left(T^0 Q_0 \left(S^0_{s,a}\right)\right)^{\frac{1}{2}},
        \quad
        \delta_{\mathbf T}(s,a) \coloneqq \left| \mathbb E \left[ T^0Q^\star\left(S^0_{s,a}\right) - TQ ^{\star}\left(S^0_{s,a}\right)\right] \right|.
    \end{equation}
\end{theorem}

The proof is given in Section \ref{section:analysis-of-the-error}.
The condition \eqref{eq:condition-M}  on $M$   ensures that $\sqrt{M} > \Lambda$, so that the upper bound \eqref{eq:final-MLMC-error} is well-defined.
Note that $\tilde \gamma \le 2\gamma$,
  due to the simple bound  $\rho_{k,M}\le 1$.

As indicated by \eqref{eq:final-MLMC-error}, by directly estimating the difference $\int (T Q_{l} - T Q_{l-1})(s') P(\d s'|s,a)$, the MLMC estimator links the dependence on the sample size $M$  and the number $n$ of fixed-point iterations. This is in contrast with the error bound \eqref{eq:final-error-simple-iterative-mc} for the simple iterative MC estimator, where 
$M$ and $n$
 contribute independently to the error.  
Consequently, 
for a fixed sufficiently large $M$, \emph{independent of the desired accuracy level}, 
the first term of 
\eqref{eq:final-MLMC-error}
converges to zero exponentially.
This observation enables the MLMC estimator to attain    an improved sample complexity compared to the simple iterative estimator;
see Remark \ref{rmk:quasi-polynomial_MLMC}.

The error bound \eqref{eq:final-MLMC-error} also indicates the  dependence of the accuracy of the MLMC estimator on the  bias \( \delta_{\mathbf{T}} \) of the  approximation operator \( \mathbf{T} \).  
Consequently, optimizing the sample complexity of the MLMC estimator requires a precise quantification of how the bias of $T^\theta$ depends on the sample size $K^\theta$  
 and optimize it jointly with the parameter  $M$  and the iteration number $n$. In the sequel, we address this issue for $\mathbf T$ being a family of   biased  plain Monte Carlo estimators  
   and a family of unbiased estimators with a randomized sample size.

\subsubsection{Sample Complexity with a Plain Monte Carlo Approximation for \texorpdfstring{$T$}{T}}
\label{subsection:plain-monte-carlo}

In this section, we specialize Theorem \ref{thm:global-MLMC-error}
to a family of MLMC estimators
where $\mathbf T$ is the
plain MC approximation 
of the soft-Bellman operator.

\begin{definition}
    For each $K\in \N$,
    let $ \mathbf{T}_K = (\hat T^\theta_K)_{\theta \in \Theta}$ be defined as in Definition \ref{def:T-plain-monte-carlo},
    and let 
    $({Q}_{n,M, \mathbf{T}_K}^{\theta})_{n\in \N, M\in \N^*,\theta\in \Theta}$ be the MLMC estimators  defined using $ \mathbf{T}_K$ as in 
        Definition \ref{def:general-MLMC-estimator}.
We refer to these estimators as MLMC estimators with biased estimation, or in abbreviation, the MLMCb estimators.
        
\end{definition}

\begin{remark}[$ \mathbf{T}_K$ as a biased estimator]
    \label{remark:bias-plain-MC}

The plain MC approximation $\hat T^\theta_K$ is a biased approximation of  the soft-Bellman operator   $T$ due to the nonlinear function $x\mapsto -\tau  \log x$.
 This bias term $\delta_{\mathbf{T}_K}$   is of the order  $\mathcal O(K^{-1/2})$. In fact, 
 by Jensen's inequality, 
 given independent copies  $(X_i)_{i=1}^K $ 
 of a random variable   $X   $ with appropriate integrability conditions,     \begin{equation*}
        \label{eq:jensen-monte-carlo-sample-averages}
        -\log  (\mathbb E X)
        =-\log  \left(\mathbb E \left(\frac{1}{K}\sum_{i = 1}^KX_i\right)\right) 
        \leq  \mathbb E \left[-\log \left(\frac{1}{K}\sum_{i = 1}^KX_i\right)\right],
  \end{equation*}
and hence in expectation, 
 $\hat T^\theta_K$  over-estimates $T$.

\end{remark}

It is clear that the family $ (\hat T^\theta_K)_{\theta \in \Theta}$ is admissible and corresponds to $K^\theta=K$  in Definition \ref{def:admissible-stochastic-operator}.
Moreover,  $\hat T^\theta_K$  satisfies Assumption \ref{assumption:lipschitz-condition-l2-stochastic-operator} 
with 
$L = \exp \left( \tau^{-1}(\beta - \alpha) \right)$
(see  Lemma \ref{lemma:contraction-prop-T-hat-l2}). Hence, one can apply Theorem \ref{thm:global-MLMC-error} to quantify the accuracy of the MLMCb estimator and further optimize its sample complexity.
Recall 
  that $\alpha =  (1 - \gamma)^{-1}c_{\min}$ and $ \beta =   (1 - \gamma)^{-1}c_{\max}$.
\begin{theorem}
    \label{theorem:error-complexity-MC}
    Suppose  Assumptions \ref{assumption:data} and  \ref{assumption:random-variables} hold, and  
$\gamma L<1$,
with
$L \coloneqq \exp \left( \tau^{-1} (\beta - \alpha)\right)$ 
as in Theorem  \ref{thm:simple-iterative-MC}.
  For all $\varepsilon\in (0,1)$, by setting $n\in \N$, $M\in \N^*$
    and $K\in \N^*$ such that 
 \begin{equation}
        \label{eq:conditions-M-n-MC}
         \Lambda_M \coloneqq \gamma L + \frac{1 + 2 \gamma L}{\sqrt{M}} + \frac{\sqrt{\gamma}}{M^{1/4}} < 1, \quad n \geq \frac{\log \varepsilon - \log D}{\log \Lambda_M}, \quad K \geq \frac{3\gamma (L^\prime)^2 }{2 \tau (1 - \Lambda_{M})\varepsilon} 
    \end{equation}
    with 
    $  
        D \coloneqq \frac{3}{2} \max \left( \beta - \alpha, 2 \gamma L^\prime \right)
    $ and 
    $  L^\prime \coloneqq  \tau (L - 1)$,
      the MLMCb estimator satisfies   
    $$     \sup_{s \in \mathcal S, a \in \mathcal A} \left\| Q^0_{n, M,  \mathbf{T}_K}(s,a) - Q^{\star}(s,a) \right\|_{L^2} \leq \varepsilon, 
    $$ 
    and the corresponding sample complexity satisfies  
    \begin{equation}
        \label{eq:sample-complexity-bound-MC}
	\mathfrak C_{n,M, K} \leq  2^{n+2} K^{n+1}M^{n}.
    \end{equation}
 In particular, choosing 
 $M_0 $,  $n_\varepsilon$ and $K_\varepsilon$ as 
 \[
 \begin{aligned}
  M_0 &= \left\lceil \left(\frac{\sqrt{\gamma} + \sqrt{\gamma + 4(1 - \gamma L)(1 + 2\gamma L)}}{2(1- \gamma L)} \right)^4 \right\rceil, \; 
 n_\varepsilon  = \left\lceil  \frac{\log (\varepsilon /D)}{\log \Lambda_{M_0}}\right\rceil,   \\
        K_\varepsilon &= \left\lceil    \frac{3    \gamma (L^\prime)^2}{2 \tau(1 -\Lambda_{M_0})\varepsilon}  \right\rceil,
\end{aligned}
 \]
 leads to the following complexity bound
    \begin{equation}
        \label{eq:sample-complexity-quasi-polynomial-MC}
 \mathfrak C_{n_\varepsilon,M_0, K_\varepsilon}  \leq C   \varepsilon^{ -\frac{\log \varepsilon}{ \log \Lambda_{M_0}}   - \kappa},
    \end{equation}
    where the constants  $C, \kappa>0$   depend only on $c_{\min}, c_{\max}, \gamma$ and $ \tau$, and are defined in \eqref{eq:constants-K-kappa1-kappa2}.
\end{theorem}

The proof of Theorem \ref{theorem:error-complexity-MC} is given in Section \ref{subsection:specialising-complexity-plainMC}.

\begin{remark}
\label{rmk:quasi-polynomial_MLMC}    
A comparison between   
Theorems \ref{thm:simple-iterative-MC} and 
 \ref{theorem:error-complexity-MC}
reveals  that the MLMC technique achieves a cubic reduction of  the sample complexity
of the simple iterative  estimator, 
\emph{under the same model assumption}.   
Indeed, observe that  $\Lambda_{M_0}$ can be made arbitrarily close to $\gamma L$ by choosing a sufficiently large (but fixed) $M_0$. 
This along with \eqref{eq:sample-complexity-quasi-polynomial-MC} 
indicates that the MLMCb estimator can achieve an asymptotic complexity of order  $\varepsilon^{\frac{- \log \varepsilon}{\log \gamma L + \delta}(1 + o(1))}$ as $\varepsilon\to 0$, for any sufficiently small $\delta>0$.
Consequently,
the MLMC method achieves a cubic reduction in complexity, improving from the 
$\varepsilon^{\frac{-3 \log \varepsilon}{\log \gamma L }}$ complexity of the simple iterative MC estimator in Theorem \ref{thm:simple-iterative-MC}
to approximately 
$\varepsilon^{\frac{- \log \varepsilon}{\log \gamma L }}$.

However,
we note that the MLMC technique alone cannot achieve a polynomial complexity estimator due to the inherent bias of 
$ \mathbf{T}_K$    in approximating the soft-Bellman operator $T$.
As shown in Lemma \ref{lemma:additional-MC-bias}, 
the bias of 
$\mathbf{T}_K$ is of magnitude $\mathcal{O}(K^{-1})$.   
According to  the error bound \eqref{eq:final-MLMC-error},  
 achieving an accuracy $\varepsilon > 0$
 requires 
the number    of fixed-point iterations to   be   $n_\varepsilon = \mathcal{O}(\log (\varepsilon^{-1}))$, while  the sample size $K_\varepsilon$    diverges to infinity as $\varepsilon \to 0$.  
As a result, the total complexity scales as $K_\varepsilon^{n_\varepsilon} = \varepsilon^{-\log K_\varepsilon}$, indicating  a super-polynomial rate as $\varepsilon\to 0$.

 The above observation highlights the challenge of developing estimators with polynomial complexity for MDPs with general action spaces.  
When the action space is finite, one can take \(\mathbf{T}\) as the exact Bellman operator to  eliminate bias, and the MLMC technique then yields an estimator with polynomial runtime \citep{beck2023nonlinearmontecarlomethods}.  
However, this polynomial complexity bound deteriorates as the cardinality of the action space grows and eventually blows up as it tends to infinity, making it inapplicable to general action spaces.

 \end{remark}

\subsubsection{Sample Complexity with an unbiased Approximation for \texorpdfstring{$T$}{T}}

\label{subsection:unbiased-monte-carlo}

 In this section, we combine the MLMC technique with an unbiased approximation of the {nonlinear} soft-Bellman operator \eqref{eq:soft_bellman}, reducing the \emph{quasi-polynomial} complexity of the MLMCb estimator in Section \ref{subsection:plain-monte-carlo} to \emph{polynomial} complexity.

 We construct 
the unbiased approximation of the soft-Bellman operator by exploiting 
   the randomized multilevel technique   proposed  by
  \citet{blanchet2015unbiased}. 
  This approach is based on the following observation, originally made by  
  \citet{mcleish2011general} and \citet{rhee2012new}.   
  Consider 
  a continuous function $g:\R\to \R$
and i.i.d.~samples $(X_i)_{i=1}^\infty$ of an integrable random variable $X$,
by the strong law of large numbers,
\[
\begin{aligned}
g(\mathbb E X) &= \sum_{n = 1}^{\infty} \left(g(\overline{X}_{n + 1}) - g(\overline{X}_n)\right) + g(\overline{X}_1)= \sum_{n = 0}^{\infty} p_n\frac{g(\overline{X}_{n + 1}) - g(\overline{X}_n)}{p_n} + g(\overline{X}_1),
\end{aligned}
\]
where  for all $n$, $\overline{X}_n \coloneqq \frac{1}{n}\sum_{i=1}^n X_i$ and $p_n>0$ is is any sequence satisfying  
$\sum_{i=1}^\infty p_i=1$. This indicates that given  an independent random variable     $N$  with $\mathbb P(N=n)=p_n$ for all $n$,     the estimator 
$ Y\coloneqq p_N^{-1}\left(g(\overline{X}_{N+1}) - g(\overline{X}_N)\right) + g(\overline{X}_1)
$    is an unbiased estimator of $g(\mathbb E X)$. 
\citet{blanchet2015unbiased} further refine this  estimator by using antithetic variates and a random sample size  $N=2^K$, 
where $K$  follows a suitably chosen geometric distribution;
see the work of \citet{bujok2015multilevel, rhee2015unbiased}
for related ideas.

Here we present the precise definition of the Blanchet–Glynn  
type approximation for the soft-Bellman operator $T$, and the resulting MLMC estimator for the optimal Q-function.

\begin{definition}
    \label{def:T-blanchet-glynn}
    Suppose Assumptions \ref{assumption:data} and \ref{assumption:random-variables} hold. 
    Let $(\tilde{K}^{\theta})_{\theta \in \Theta}$ be  a family of independent random variables 
    that is independent of $(A^\theta, U^\theta)_{\theta\in \Theta}$, where $\tilde{K}^\theta$ is geometrically distributed with success parameter   $r \in (1/2, 3/4)$,
    i.e., $p(k)\coloneqq \mathbb P(\tilde{K}^{\theta} =k)=r(1-r)^k$ for all $k\in \N$. 
    Let  
    $g:  (0,\infty)\ni x \mapsto -\tau \log x \in \mathbb R$.

    For any $K \in \N$,
    $\theta \in \Theta$, $s \in \mathcal S$,
    and     $Q \in B_b(\mathcal S \times \mathcal A)$,  define  
    \begin{equation*}
        \label{eq:delta-T-blanchet-glynn}
        \begin{aligned}
             & \Delta_{K}^{\theta}Q(s) =  g \left( \frac{1}{2^{K+1}}\sum_{k = 1}^{2^{K+1}} \exp \left(-Q \left(s,A^{(\theta,k)}\right) / \tau \right) \right)                                                                                                            \\
             & - \frac{1}{2}\left[ g \left( \frac{1}{2^{K}}\sum_{k = 1}^{2^{K}} \exp \left(-Q\left(s,A^{(\theta,2k)}\right) / \tau \right)  \right)  + g \left( \frac{1}{2^{K}}\sum_{k = 1}^{2^{K}} \exp \left(-Q\left(s,A^{(\theta,2k-1)}\right) / \tau \right)  \right) \right],
        \end{aligned}
    \end{equation*}
    and define 
      the Blanchet--Glynn approximation of the soft-Bellman operator by 
    \begin{equation}
        \label{eq:T-blanchet-glynn}
        \tilde{T}^\theta Q (s) \coloneqq  \frac{\Delta_{\tilde{K}^{\theta}}^\theta Q (s)}{p(\tilde{K}^\theta)} +  Q(s,A^{(\theta,0)}).
    \end{equation}

We denote by $({Q}_{n,M, \tilde{\mathbf{T}}}^{\theta})_{n\in \N, M\in \N^*,\theta\in \Theta}$  the MLMC estimators  defined using $\tilde{\mathbf{T}} =(\tilde{T}^{\theta})_{\theta\in \Theta}  $ 
    (cf.~Definition \ref{def:general-MLMC-estimator}).
    These estimators will be referred to as     MLMC estimators with unbiased estimation, or, in abbreviation, the MLMCu estimators.  
\end{definition}

\begin{remark}[Role of parameter $r$]
    \label{remark:r}
   The parameter $ r $ 
   for the geometric distribution 
   determines both the sample complexity and the stability of the Blanchet–Glynn approximation $ \tilde{T}^\theta $.  
   Observe that 
     the expected sample size of $\tilde T^\theta$ is  
    $\mathbb E[2^{\tilde K^\theta+1}]=\sum_{n=0}^\infty 2^{n + 1}p(n) = 2 r \sum_{n=0}^\infty  2^n (1-r)^n = (2r - 1)^{-1}2r$,
    which is finite for $r > 1/2$. 
    The condition $r<3/4$
    ensures that
    the approximation $\tilde T^\theta Q$   has a  finite variance     for any given $Q$ 
    (Proposition \ref{prop:unbiasedness-blanchet-glynn}),
    and the map $Q\mapsto \tilde T^\theta Q$ is Lipschitz  continuous  with respect to the $\|\cdot\|_{L^2}$-norm
    (Proposition \ref{proposition:contraction-prop-T-blanchet-glynn-l2}).

The choice of $ r $ represents a trade-off between the computational cost and the numerical stability of the MLMCu estimator. As $ r $ approaches $ {3}/{4} $, the expected sample size of $ \tilde{T}^\theta $ decreases, leading to lower computational cost. However, this also increases the Lipschitz constant in Proposition \ref{proposition:contraction-prop-T-blanchet-glynn-l2}, resulting in  greater numerical instability of the MLMCu estimator
as the number of fixed-point iteration increases;
see Section \ref{section:numerics} for details.  
    
\end{remark}

It is easy to see that the family 
$\tilde{\mathbf{T}}$  is admissible in  the sense of Definition \ref{def:admissible-stochastic-operator}.
The following proposition shows that 
$\tilde{T}^\theta$ 
is unbiased and 
has finite variance. The proof is given  in Section \ref{subsection:specialising-error-BG}.

\begin{proposition}
        \label{prop:unbiasedness-blanchet-glynn}
  For all $Q \in B_b(\mathcal S \times \mathcal A)$ and   $s \in \mathcal S$, 
  $ \mathbb E \tilde{T}^\theta Q (s) =T Q (s) $, 
  with $T$   defined in \eqref{eq:soft_bellman},
  and 
  $\mathbb E |\tilde{T}^\theta Q (s)|^2 < \infty $. Consequently,
  $\delta_{\tilde{\mathbf T}} \equiv 0$,
 where $\delta_{\tilde{\mathbf T}}$ is defined in \eqref{eq:def-sigma-t}.
\end{proposition}

By Proposition \ref{prop:unbiasedness-blanchet-glynn} and Lemma \ref{lemma:boundedness-of-iterates}, 
the estimator $\tilde T^\theta$ satisfies  Assumption \ref{assumption:lipschitz-condition-l2-stochastic-operator}\ref{item:boundedness}.
The following proposition  proves  that the 
map  
$Q\mapsto \tilde T^\theta Q$
is Lipschitz continuous with respect to the $L^2$-norm,
and hence verifies 
Assumption \ref{assumption:lipschitz-condition-l2-stochastic-operator}.\ref{item:lipschitz-property}
for the Blanchet–Glynn approximation. 
The proof is given in Section \ref{subsection:specialising-error-BG}.
 
\begin{proposition}
    \label{proposition:contraction-prop-T-blanchet-glynn-l2}
 
     Suppose Assumptions \ref{assumption:data} and \ref{assumption:random-variables} hold.      Then the Blanchet–Glynn estimator  
    $\tilde{\mathbf{T}} =(\tilde{T}^{\theta})_{\theta\in \Theta}  $ 
defined in \eqref{eq:T-blanchet-glynn} satisfies Assumption \ref{assumption:lipschitz-condition-l2-stochastic-operator}\ref{item:lipschitz-property}. 
More precisely, 
for all $Q_1,Q_2: \mathcal S \times \mathcal A \times \Omega \rightarrow \R$
and $S:\Omega\to \cS$
satisfying the conditions in  Assumption \ref{assumption:lipschitz-condition-l2-stochastic-operator}\ref{item:lipschitz-property}, and for   $\theta \in \Theta$,
    \begin{equation}
        \label{eq:contraction-prop-T-blanchet-glynn-l2}
        \left\|\tilde{T}^\theta Q_1(S) - \tilde{T}^\theta Q_2(S) \right\|_{L^2} \leq L_{(\alpha, \beta, \tau, r)} \|Q_1(S,A^{(\theta, 1)}) - Q_2(S,A^{(\theta, 1)})\|_{L^2},
    \end{equation}
    where
    $L_{(\alpha, \beta, \tau, r)} = 1 + \sqrt{C_{ (\alpha, \beta, \tau)} \frac{4(1-r)}{3r - 4r^2}} < \infty$, 
    $\alpha,\beta$ are defined in Assumption 
      \ref{assumption:data},
    and $C_{ (\alpha, \beta, \tau)}$ is the constant given by \eqref{eq:lipschitz-constant-c(a,b,t)}.
\end{proposition}

To the best of our knowledge, this is the first result regarding the Lipschitz continuity of the Blanchet–Glynn estimator
with respect to the input random variable.
The proof relies on a second-order Taylor expansion of the   function $g:  (0,\infty)\ni x \mapsto -\tau \log x \in \mathbb R$  
at the corresponding expectations and carefully bounds the
  $L^2$-norm of each residual term.

Proposition 
\ref{proposition:contraction-prop-T-blanchet-glynn-l2} allows for  applying Theorem \ref{thm:global-MLMC-error} to quantify   the   error of the MLMCu estimator for all $n \in \N, M \in \N^*$:
\begin{equation}
        \label{eq:final-MLMC-error-blanchet-Glynn}
        E _{n,M,\tilde{\mathbf T}} \leq \frac{3}{2}\max\left( \| Q_0 - Q^\star\|_{\infty},  2 \gamma \| \sigma_{\tilde{\mathbf T}} \|_{\infty} \right) \left( \gamma L + \frac{1 + 2\gamma L}{\sqrt{M}} + \frac{\sqrt{{\gamma}}}{M^{1/4}} \right)^n,
    \end{equation}
 where $L $ is given in Proposition \ref{proposition:contraction-prop-T-blanchet-glynn-l2} and $\sigma_{\tilde{\mathbf T}}$ is defined as in  \eqref{eq:def-sigma-t}.
Note that the bias 
$\delta_{\tilde{\mathbf T}}\equiv 0$ in
\eqref{eq:def-sigma-t}, due to  
 Proposition \ref{prop:unbiasedness-blanchet-glynn}.

Based on 
the error bound
\eqref{eq:final-MLMC-error-blanchet-Glynn},
 the following theorem determines the values of 
$n$ and $M$   required to achieve a prescribed accuracy $\varepsilon>0$,
and subsequently establishes  a polynomial  complexity of 
the MLMCu estimator.
   The proof is given in Section \ref{subsection:specialsing-complexity-BG}.

\begin{theorem}
    \label{theorem:error-complexity-BG}
Suppose  Assumptions \ref{assumption:data} and   \ref{assumption:random-variables} hold.
Let  
$\left({Q}_{n,M, \tilde{\mathbf{T}}}^{\theta}\right)_{n\in \N, M\in \N^*,\theta\in \Theta}$
be the estimators 
defined in Definition \ref{def:T-blanchet-glynn},
and  $ L=L_{(\alpha,\beta,\tau,r)}$ be the Lipschitz constant  in Proposition \ref{proposition:contraction-prop-T-blanchet-glynn-l2}.
Assume that $\gamma L < 1$.
Then for all $\varepsilon > 0$,
by setting 
  $n , M \in \N^*$ such that 
   \begin{equation}
        \label{eq:conditions-M-n}
        \Lambda_M \coloneqq \gamma L + \frac{1 + 2 \gamma L}{\sqrt{M}} + \frac{\sqrt{\gamma}}{M^{1/4}} < 1, \text{ and  } n \geq \frac{\log \varepsilon - \log D}{\log \Lambda_M} ,
    \end{equation}
with  
 $D \coloneqq \frac{3}{2} (\beta - \alpha) \max(1,2 \gamma L)$,
 the MLMCu estimator 
 satisfies 
    \begin{equation}
        \label{eq:error-varepsilon}
        \sup_{s \in \mathcal S, a \in \mathcal A} \left\| Q^0_{n, M, \tilde{\mathbf T}}(s,a) - Q^{\star}(s,a) \right\|_{L^2} \leq \varepsilon,
    \end{equation}
    and the corresponding sample complexity satisfies 
    \begin{equation}
        \label{eq:sample-complexity-bound-expectation}
	\mathbb E [\mathfrak C_{n, M}] \leq 2 \left(
    \frac{4r}{ 2r - 1}
    \right)^{n+1} M^{n}.
    \end{equation}
   In particular, choosing 
         $ M_0$ and $n_\varepsilon$ as
    \[
        M_0 = \left\lceil \left(\frac{\sqrt{\gamma} + \sqrt{\gamma + 4(1 - \gamma L)(1 + 2\gamma L)}}{2(1- \gamma L)} \right)^4 \right\rceil , \quad n_\varepsilon = \left\lceil  \frac{\log (\varepsilon/D)}{\log \Lambda_{M_0}}\right\rceil,
    \]
    leads to the following complexity bound
    \begin{equation}
        \label{eq:sample-complexity-polynomial-complexity}
        \mathbb E [\mathfrak C_{n_\varepsilon, M_0}] \leq C \varepsilon^{- \kappa},
    \end{equation}
    where the constants $C, \kappa>0 $ depend only on $c_{\min}, c_{\max}, \gamma, \tau$ and $ r$, and are defined in \eqref{eq:def-K-kappa}.
\end{theorem}

 To the best of our knowledge, this is the first MC estimator for MDPs with a polynomial complexity \emph{independent of the 
 dimensions and 
 cardinalities of the state and action spaces}.  
This contrasts with the polynomial complexity bounds established by \citet{grill2019planning} and \citet{beck2023nonlinearmontecarlomethods}, which explicitly depend on the cardinality of the action space and blow  up to infinity for continuous action spaces.

The condition $\gamma L < 1$ in Theorem \ref{theorem:error-complexity-BG} involves a different constant $L$ than in Theorems \ref{thm:simple-iterative-MC} and \ref{theorem:error-complexity-MC}.
This condition  holds if the discount factor $\gamma$ is sufficiently small. Indeed, observe that   the Lipschitz constant $ L $ depends on $ \gamma $ only through $ \alpha = (1 - \gamma)^{-1}c_{\min} $ and $ \beta = (1 - \gamma)^{-1}c_{\max} $. As $ \alpha $ and $ \beta $ remain bounded as $ \gamma \to 0 $, and the map  $ (\alpha, \beta) \mapsto L_{(\alpha,\beta, \tau ,r)} $ is continuous, it follows that $ L $ is bounded as $ \gamma \to 0 $ and  $\lim_{\gamma\to 0} \gamma L = 0 $.
However, it is
unclear whether $\gamma L < 1$
holds for a sufficiently large regularization parameter $\tau$.

We emphasize that achieving polynomial complexity in this setting requires incorporating   MLMC techniques into both the fixed-point iteration and the approximation of the soft-Bellman operator.  
The MLMC approach for the fixed-point iteration reduces variance in estimating expectations over the state space $  \mathcal{S}$ (Remark \ref{rmk:quasi-polynomial-simple-iterative}),
while   the MLMC technique applied to the soft-Bellman operator eliminates the bias of the approximation estimator and reduces variance in estimating expectations over the action space  $\mathcal{A}$ (Remark \ref{rmk:quasi-polynomial_MLMC}).

\section{Numerical Experiments}
\label{section:numerics}
In this section, we examine the performance of the MLMC estimators in multi-dimensional entropy-regularized linear quadratic (LQ) control problems. 
Specifically, we compare the effectiveness of the  MLMCb estimator analyzed in Section \ref{subsection:plain-monte-carlo} and the  MLMCu estimator analyzed in Section \ref{subsection:unbiased-monte-carlo}.\footnotemark 
Our numerical results demonstrate that:
\begin{itemize}
    \item The MLMCu estimator, with appropriately chosen hyperparameters, achieves polynomial complexity, whereas the MLMCb estimator exhibits super-polynomial complexity.

    \item
    The MLMCb estimator is robust with respect to  the sample size used to approximate the soft-Bellman operator, while the MLMCu estimator is sensitive to the choice of $r$
 for the geometric distribution. For large values of $r$, the MLMCu estimator exhibits numerical instability as the number of fixed-point iterations increases.

\end{itemize}

\footnotetext{The   implementation details can be found in Appendix \ref{app:algos}.}

\subsection{Experiment Setup}
\label{subsection:experimental-setup}

We   consider an 
infinite-horizon
entropy-regularized  discounted  LQ   control problem. 
Although this setup does not fully align with our framework due to the unbounded cost, it serves as a benchmark for validating our estimators since the optimal solution is available analytically.
More precisely, let  $d_a, d_s \in \N^*$,   $A \in \R^{d_s \times d_s}, B \in \R^{d_s \times d_a} $ and $ R_1 \in \R^{d_s \times d_s}, R_2 \in \R^{d_a \times d_a}$ be  symmetric positive semidefinite matrices.  Let
$\gamma \in (0,1), \tau > 0$ and $\mu = \mathcal N (0, I_{d_a})$
be a standard   normal distribution on $\mathbb R^{d_a}$,
consider the following minimization problem:
\begin{equation}
    \label{eq:infinite-horizon-lqg}
    \begin{aligned}
        J^\star(s_0)=\min_{\pi \in \mathcal P (\R^{d_a} | \R^{d_s})} J(\pi,s_0) & = \mathbb E \left[ \sum_{t = 0}^{\infty} \gamma^t(s_t^\top R_1 s_t + a_t ^\top R_2 a_t + \tau \mathrm{KL}\left(\pi(\cdot | s_t) \mid \mu\right) \right], 
    \end{aligned}
\end{equation}
subject to  
$s_0  \in \R^{d_s}$,
and  
$$ 
 s_{t + 1}                                       = A s_t + B a_t + w_t, 
\quad 
 a_t\sim \pi(\cdot|s_t),
\quad t \in \N,
$$ 
where $s_0$
is a given initial state, 
$a_t $ is a random  variable 
with distribution $\pi(\cdot|s_t)$,  conditionally independent of $\sigma((a_i)_{i=0}^{t-1}, (s_i)_{i=0}^t)$,
and 
$w_t$ is an independent $d_s$-dimensional standard normal random variable.

By the   dynamic programming principle, 
the optimal value function $J^\star$ of \eqref{eq:infinite-horizon-lqg} is given by $J^\star(s) = s^\top P s + c$,
where $P$ is the unique positive semidefinite solution to  the following algebraic Riccati equation:
\begin{equation}
	\label{eq:are}
	P = R_1 + \gamma A^\top P A - \gamma^2 A^\top P B \left(R_2 + \gamma B^\top P B + \frac{\tau}{2}I_{d_a}\right)^{-1}B^\top P A,
\end{equation}
and $c$ is given by 
\begin{equation}
	\label{eq:c-infinite-horizon-lqg}
	c \coloneqq  \frac{1}{1-\gamma} \left( \gamma \mathrm{tr}(P) + \frac{\tau}{2}\log \mathrm{det}\left( I_{d_a} + \frac{2}{\tau}(R_2 + \gamma B^\top P B) \right) \right).
\end{equation}
Moreover, the optimal policy is given by $\pi(\cdot | s) = \mathcal N (\mu(s), \Sigma)$, where
\begin{equation}
	\label{eq:optimal-policy-infinite-lqg}
	\begin{aligned}
		\mu(s) &= -\gamma \left(R_2 + \gamma B^\top P B + \frac{\tau}{2} I_{d_a}\right)^{-1}B^\top P A s, \\
		\Sigma &= \left(I_{d_a} + \frac{2}{\tau} (R_2 + \gamma B^\top P B)\right)^{-1}.	
	\end{aligned}
\end{equation}
The result follows   from \citet{guo2023fastpolicylearninglinear} by using 
$ \mathrm{KL}\left(\pi(\cdot | s_t) \mid \mu\right) 
=\mathrm{KL} \left(\pi(\cdot | s_t) \mid \mathcal L_{\rm Leb}\right)
-\int_{\R^{d_a}} 
\frac{|a|^2}{2} 
\pi(d a|s)+
\frac{1}{2}\log \mathrm{det} (I_{d_a})
+\frac{d_a}{2} \log (2\pi)
$, where $ \mathcal L_{\rm Leb}$ is the Lebesgue measure on $\R^d$.

In the following, we obtain a reference solution for our experiments by solving the Riccati equation \eqref{eq:are} using the given coefficients. The corresponding optimal \( Q \)-function is denoted as \( Q_{\mathrm{ref}} \).  We choose the parameters    $d_a = d_s = d = 20$, 
$R_d = R_{1,d} = R_{2,d} = I_d / d$,
and 
\[
    A_d = I_d, \quad  B_d = \begin{pmatrix}
        1        & \varepsilon & 0        & 0        & \cdots & 0      \\
        0        & 1        & \varepsilon & 0        & \cdots & 0      \\
        0        & 0        & 1        & \varepsilon & \cdots & 0      \\
        \vdots   & \vdots   & \vdots   & \vdots   & \ddots & \vdots \\
        \varepsilon & 0        & 0        & 0        & \cdots & 1
    \end{pmatrix},
\]
with $\varepsilon = 0.1$,
which ensures  a non-trivial dependence across dimensions.
 
We conduct a series of experiments for both the MLMCb estimator (using the plain Monte Carlo approximation) and the MLMCu estimator (using the Blanchet–Glynn approximation) to compare their performance. We fix  the basis number of outer samples in the MLMC estimator $M=7$
(cf.~Definition \ref{def:general-MLMC-estimator}), and vary the following parameters:
\begin{itemize}
    \item $\gamma\in \{0.4,0.5,0.6\}$, and  $\tau = (1 - \gamma)^{-1}$ for numerical stability;
    \item for the plain Monte Carlo estimator, we choose sample sizes \( K \in \{2, 4, 6\} \) to approximate the Bellman operator;
    \item for the Blanchet--Glynn estimator, we take $r \in \{0.6, 1 - 2^{-3/2}\}$
    for the geometric distribution,
    with $r=1 - 2^{-3/2}\approx 0.646$ being the parameter suggested by \citet{blanchet2015unbiased};
    \item the level $l$ varies in $\{1,2, \ldots, 6\}$.
\end{itemize}
Here, we choose small values of $\gamma$ to examine the asymptotic convergence rates of the MLMC estimators within a reasonable time frame. This allows us to gain insights into the differences between the two estimators while staying within a feasible computational budget.
For each parameter configuration, we  estimate the optimal Q-function at the state-action pair $s_0 = (0 , 0 , \ldots,  0), a_0 = (1,  1 , \ldots,  1)$, and  perform 20 runs in parallel on 20 13th Gen Intel Core i5-13500T CPUs.

\subsection{Numerical Results}
\label{subsection:results}

This section summarizes the results for the cases \( \gamma \in \{ 0.4,0.5\} \). Additional numerical results for the case  $\gamma = 0.6$  are presented in Appendix \ref{appendix:numerical-results}.

 Figure \ref{fig:gamma-04-avg-q-all}  visualizes the average estimate of $Q^\star(s_0,a_0)$ for each configuration of the MLMC estimator for $\gamma = 0.4$, and plot the reference value $Q_{\mathrm{ref}}(s_0,a_0)$ for comparison.
We clearly see that the MLMCu estimators give a better estimate of $Q^\star(s_0,a_0)$ for levels $l \geq 4$. For MLMCb, we observe that as the inner sample size $K$ decreases, the estimation at levels $l \geq 4$ gets worse. This is easily understood in terms of the bias of the plain Monte Carlo estimator, which results in overestimation of the optimal Q-function, as highlighted in Remark \ref{remark:bias-plain-MC}.

\begin{figure}[!ht]
    \centering
    \includegraphics[width=0.8\linewidth]{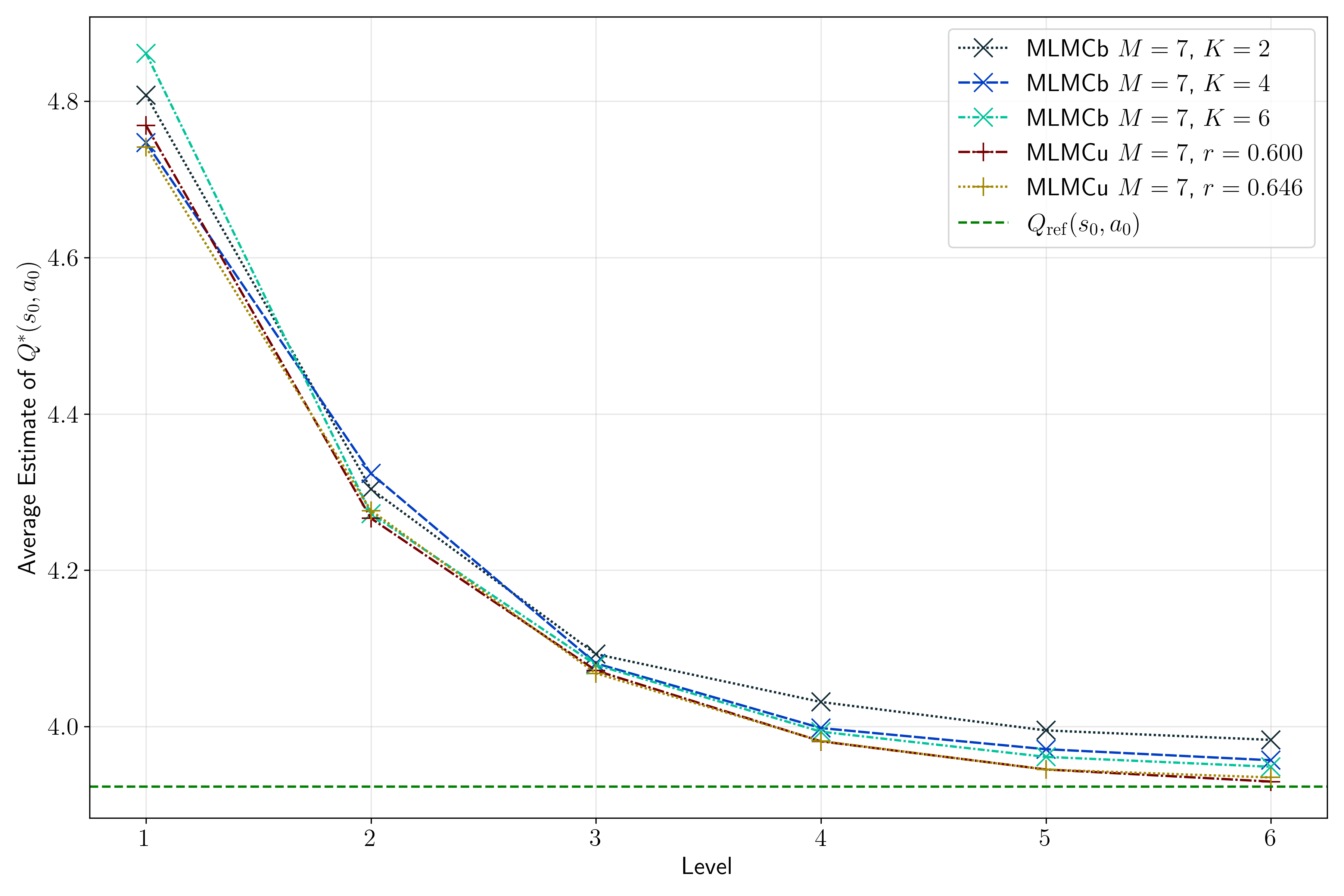}
    \caption{Average estimate of $Q^\star(s_0,a_0)$ over 20 runs for $d=20, \gamma=0.4$.}
    \label{fig:gamma-04-avg-q-all}
\end{figure}

 Figure \ref{fig:gamma-04-rmse-time} visualizes the root mean squared relative error (RMSRE) of the estimates as a function of average compute time for each configuration of the MLMC estimator. Given estimates $(\hat{q}_i)_{i=1}^{20}$ of $Q^\star(s_0,a_0)$ from 20 independent runs, we compute the RMSRE as
\[
    \mathrm{RMSRE} = \sqrt{\frac{1}{20} \sum_{i = 1}^{20} \left(\frac{\hat{q}_i - Q_{\mathrm{ref}}(s_0,a_0)}{Q_{\mathrm{ref}}(s_0,a_0)} \right)^2}.
\]
According to Theorem \ref{theorem:error-complexity-BG}, we expect   a power law relationship between these two quantities for MLMCu. This is confirmed by   the straight lines in Figure \ref{fig:gamma-04-time-error-bg-fit}. In contrast, it seems that the MLMCb estimator does not exhibit a power law, as can be clearly seen for $K = 2$ in Figure \ref{fig:gamma-04-rmse-time}.
This shows 
that the MLMCb error indeed suffers from a super-polynomial complexity, confirming 
Theorem \ref{theorem:error-complexity-MC}. 
This behaviour stems from
the intrinsic bias of the plain Monte Carlo average for approximating $T$ \eqref{eq:T-operator-def}.

\begin{figure}[!ht]
    \centering
    \begin{subfigure}[b]{0.49\linewidth}
        \centering
        \includegraphics[width=\linewidth]{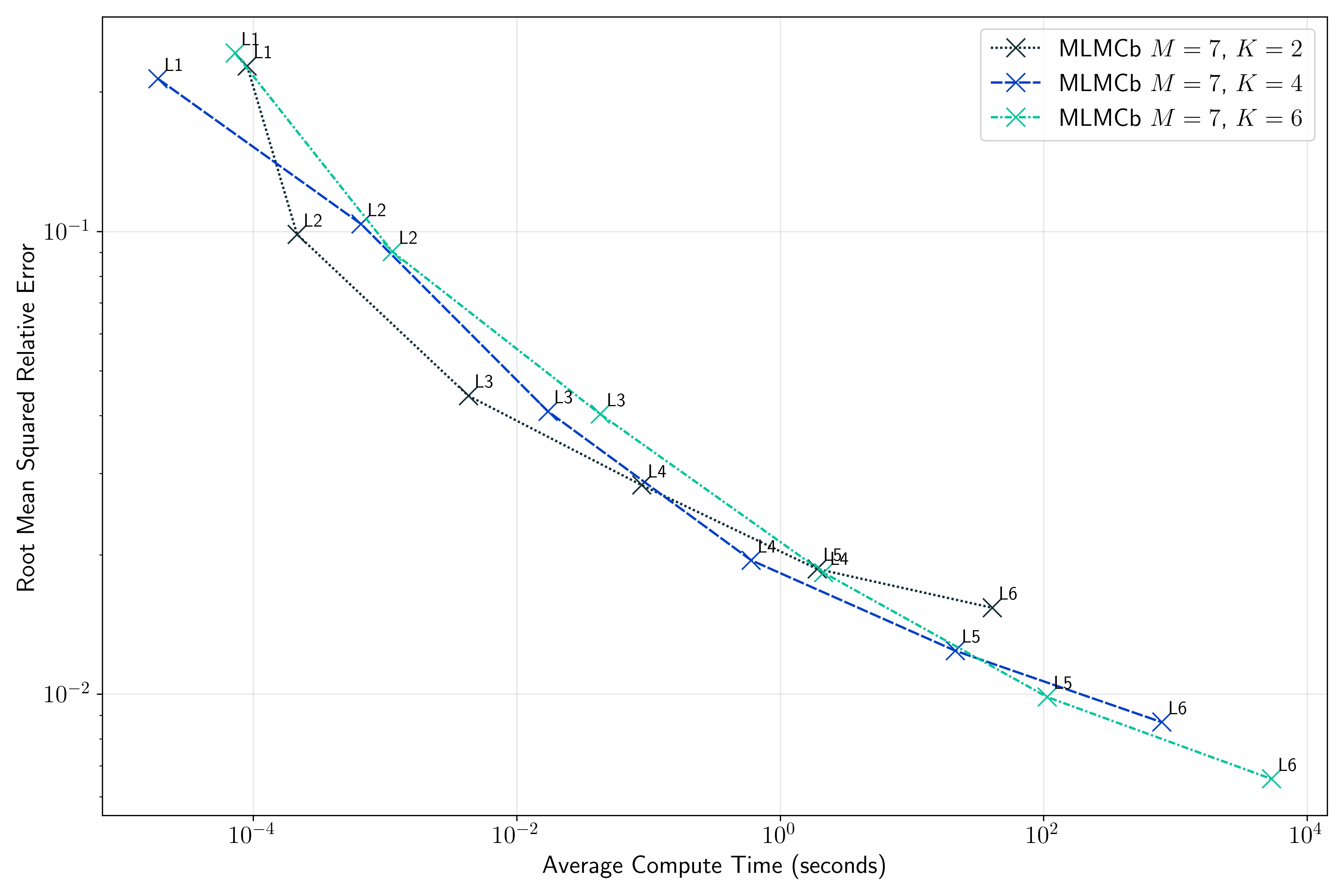}
        \caption{Plain Monte Carlo (MLMCb).}
        \label{fig:gamma-04-time-error-mc}
    \end{subfigure}
    \hfill
    \begin{subfigure}[b]{0.49\linewidth}
        \centering
        \includegraphics[width=\linewidth]{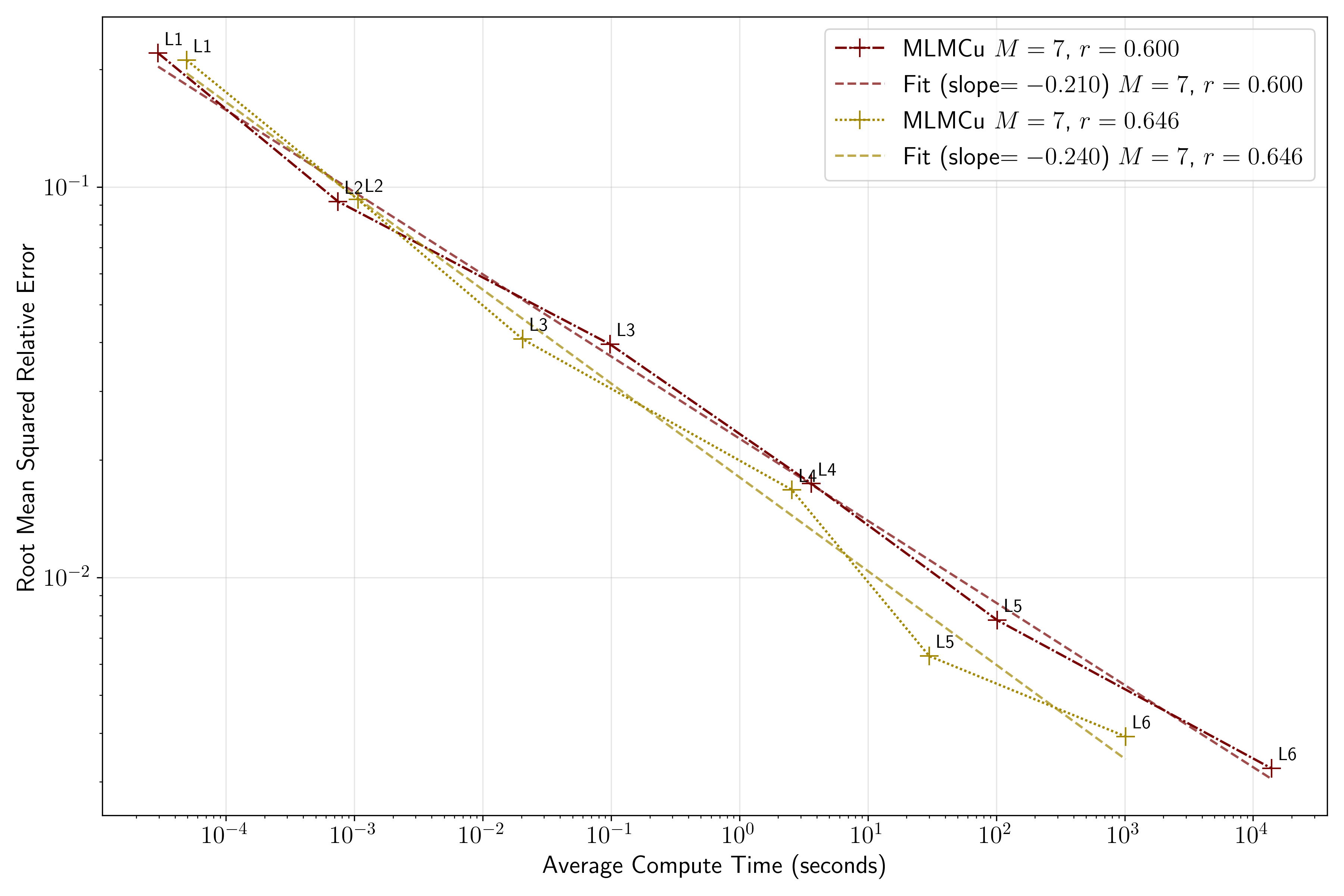}
        \caption{Unbiased Monte Carlo (MLMCu).}
        \label{fig:gamma-04-time-error-bg-fit}
    \end{subfigure}
    \caption{RMSRE as a function of average compute time over 20 runs for $d=20, \gamma = 0.4$  (plotted in a log-log scale). Each point is annotated with the level $n$ used.}
    \label{fig:gamma-04-rmse-time}
\end{figure}

Similar convergence behaviors are observed for a larger value of $\gamma = 0.5$ (at least for MLMCu estimators with sufficiently small $r$), as shown in Figures \ref{fig:gamma-05-avg-q-all} and \ref{fig:gamma-05-rmse-time}.

Moreover, Figures \ref{fig:gamma-05-avg-q-all} and \ref{fig:gamma-05-rmse-time} show that for larger values of \(\gamma\), the MLMCb estimator remains robust with respect to the inner sample sizes \(K\), while the MLMCu estimator requires smaller values of \(r\) (and thus more inner samples) to maintain numerical stability.
Specifically, for \(r = 0.6\), the MLMCu estimator remains stable and achieves polynomial complexity,
while for $r = 0.646$, the MLMCu estimator becomes numerically unstable as the level $n$ increases. 
This instability can be explained through the contraction condition $\gamma L < 1$ in Theorem \ref{theorem:error-complexity-BG} for the MLMCu estimator. To ensure numerical stability of the MLMCu estimator, the Lipschitz constant $L$ of $\tilde T^\theta$ in Proposition \ref{proposition:contraction-prop-T-blanchet-glynn-l2} needs to be less than $\gamma ^{-1}$. 
 As highlighted in Remark \ref{remark:r},
 the Lipschitz constant of $\tilde T^\theta$ 
 increases 
 as \(r\) increases,    which ultimately leads to a violation of this stability condition.

\begin{figure}[!ht]
    \centering
    \includegraphics[width=0.8\linewidth]{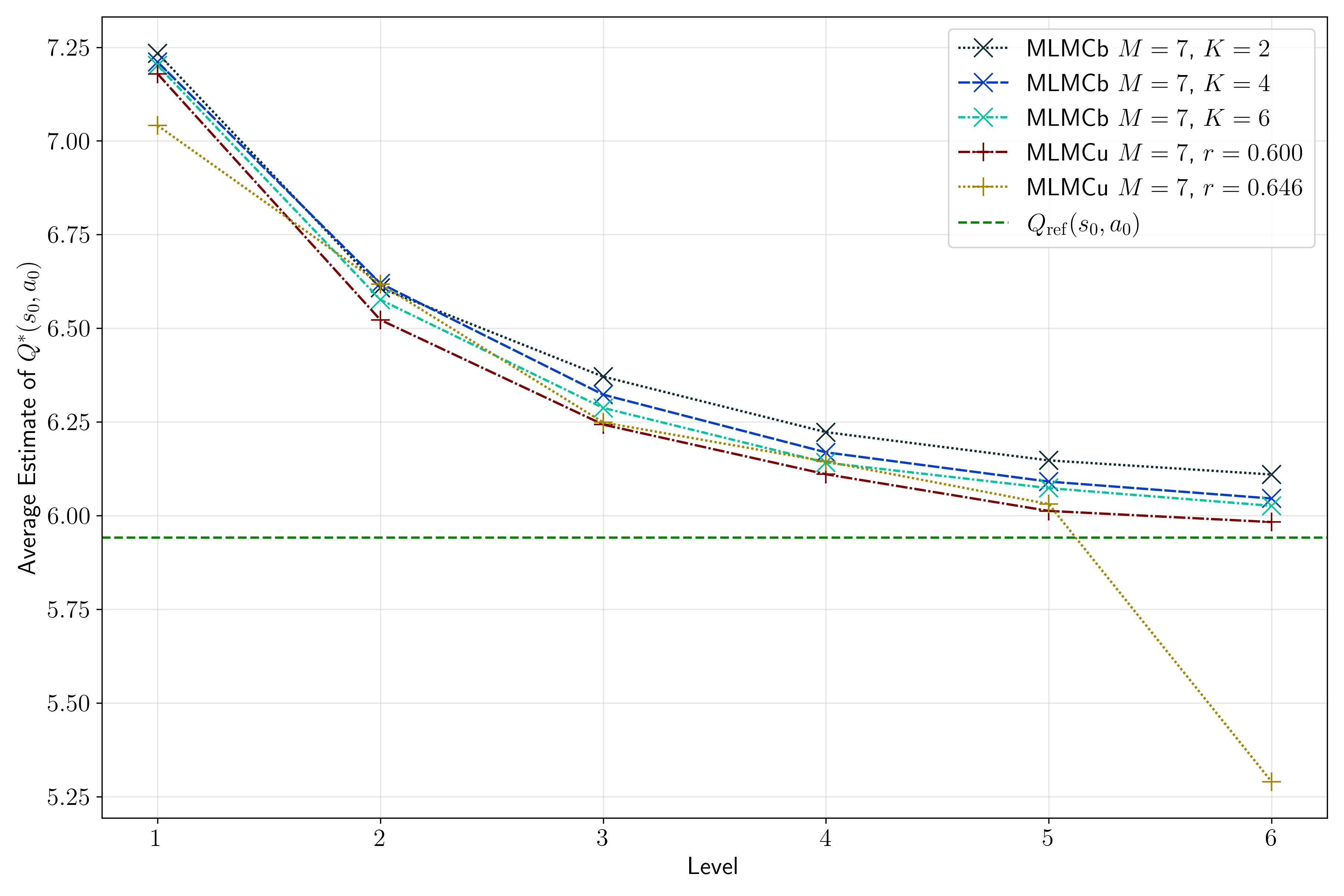}
    \caption{Average estimate of $Q^\star(s_0,a_0)$ over 20 runs for $d=20, \gamma=0.5$.}
    \label{fig:gamma-05-avg-q-all}
\end{figure}

\begin{figure}[!ht]
    \centering
    \begin{subfigure}[b]{0.49\linewidth}
        \centering
        \includegraphics[width=\linewidth]{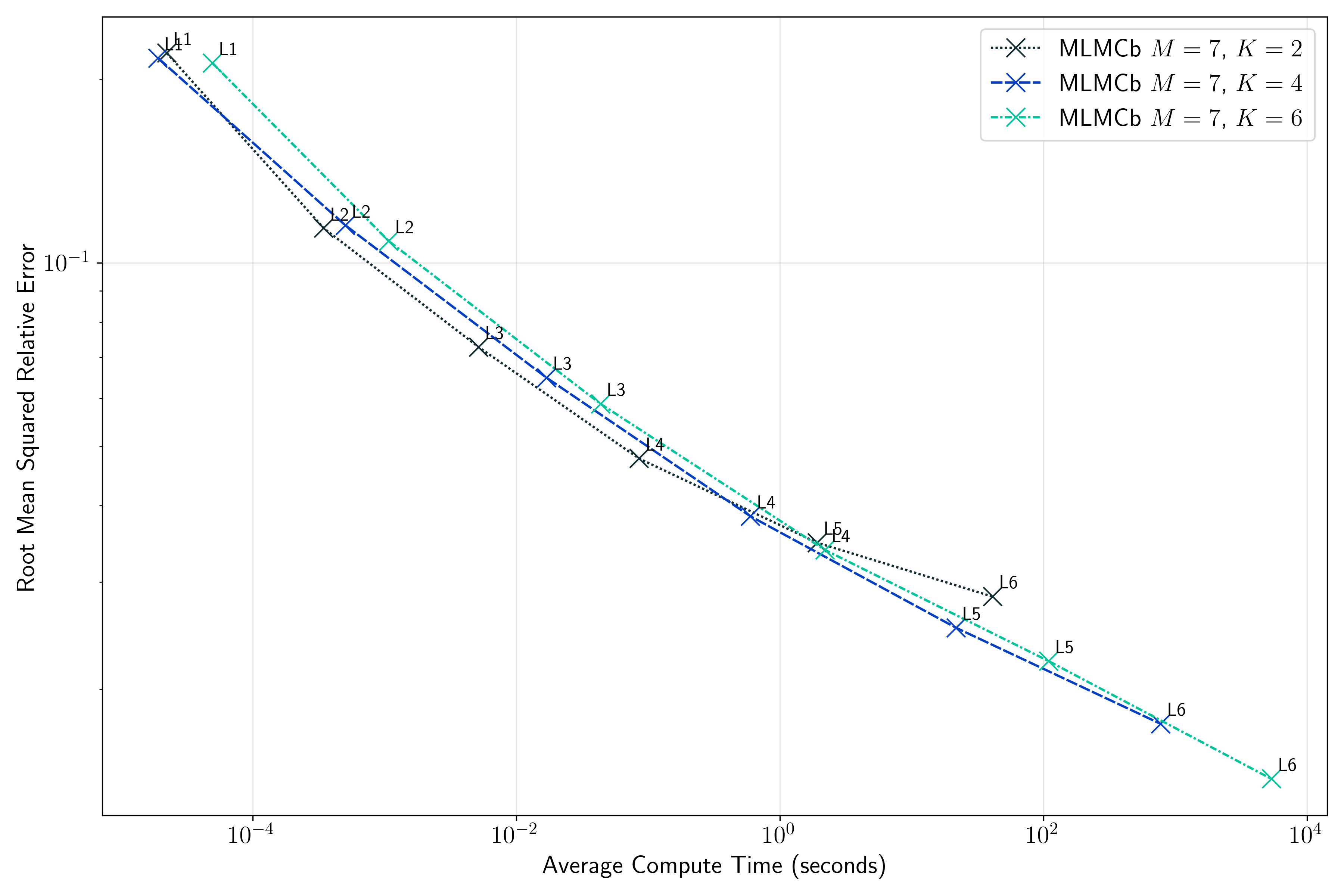}
        \caption{Plain Monte Carlo (MLMCb).}
        \label{fig:gamma-05-time-error-mc}
    \end{subfigure}
    \hfill
    \begin{subfigure}[b]{0.49\linewidth}
        \centering
        \includegraphics[width=\linewidth]{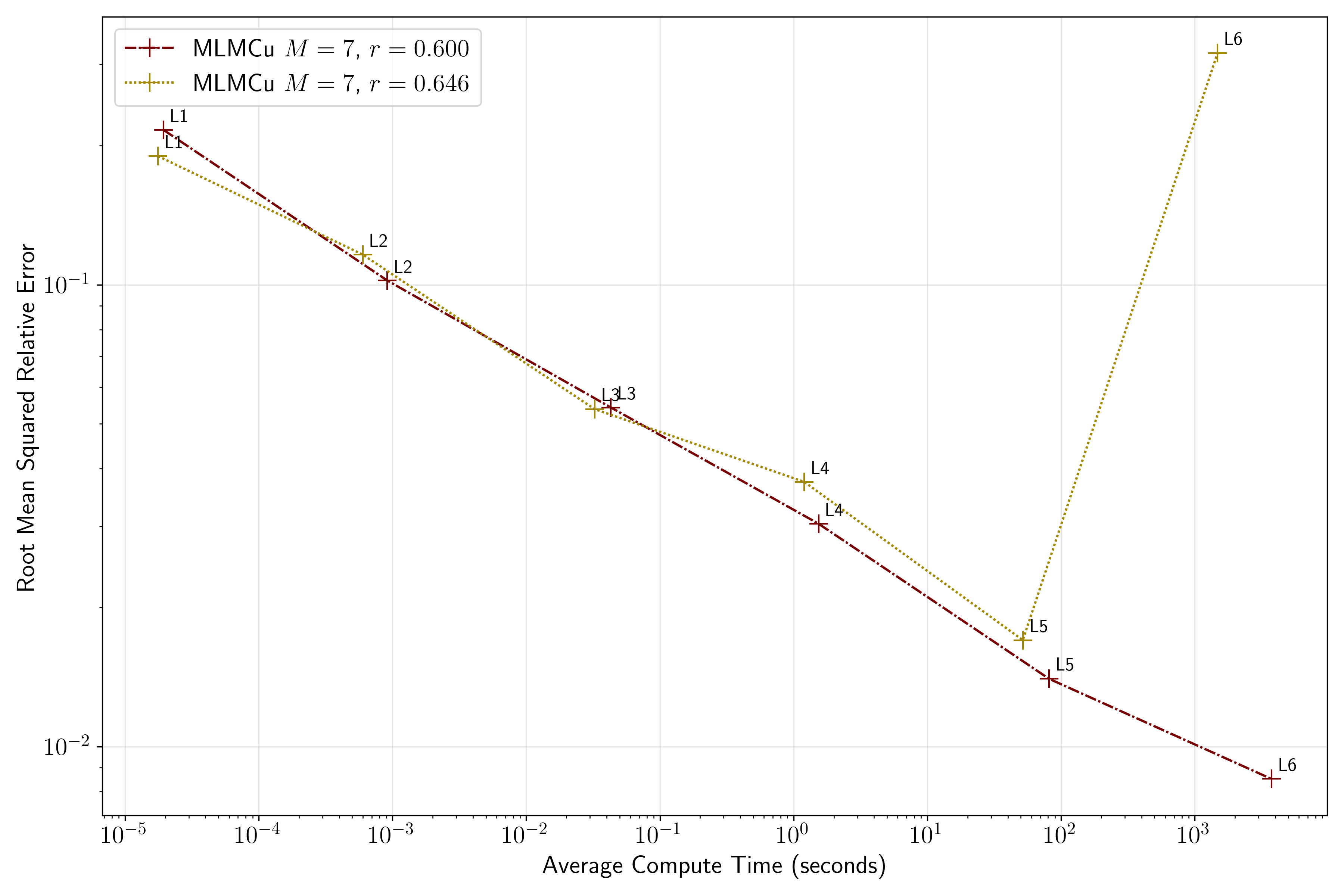}
        \caption{Unbiased Monte Carlo (MLMCu).}
        \label{fig:gamma-05-time-error-bg}
    \end{subfigure}
    \caption{RMSRE as a function of average compute time over 20 runs for $d=20, \gamma = 0.5$ (plotted in a log-log scale).}
    \label{fig:gamma-05-rmse-time}
\end{figure}

\section{Analysis of the Simple Iterative MC Estimator}
\label{sec:proof-simple-iterative-MC}

We first state a series of technical lemmas on the $T$ operator defined in \eqref{eq:T-operator-def} which guide our analysis. All proofs are presented in Appendix \ref{appendix:proofs}.
We begin with a lemma that ensures boundedness when applying $T$ iteratively.
\begin{lemma}
    \label{lemma:boundedness-of-iterates}
    Recall the notations of Assumption \ref{assumption:data}. Let $Q \in B_b(\mathcal S \times \mathcal A)$ such that $ \alpha \leq Q \leq \beta$. Then, for all $(s,s^\prime,a) \in \mathcal S \times \mathcal S \times \mathcal A$, $\alpha \leq c(s,a) + \gamma  TQ (s^\prime) \leq \beta$.
\end{lemma}

We state a Lipschitz property of $T$ with respect to the reference measure $\mu$.
\begin{lemma}
    \label{lemma:contraction-prop-T-l2}
    Let $Q_1,Q_2$ be functions in $B_b(\mathcal S \times \mathcal A)$ such that $\alpha \leq Q_1,Q_2 \leq \beta$ for real constants $\alpha, \beta$. Then, for any $s \in \mathcal S$
    \begin{equation*}
        \label{eq:contraction-prop-T-l2}
        \left|TQ_1(s) - TQ_2(s) \right| \leq e^{(\beta - \alpha)/ \tau} \int_{\mathcal A} \left| Q_1(s,a) - Q_2(s,a)  \right| \mu (\diff a).
    \end{equation*}
\end{lemma}
Notice that this Lipschitz property is different from the usual $\|\cdot\|_{\infty}$ Lipschitz property, where taking the supremum over action spaces ensures a Lipschitz constant equal to $1$.

We now show that the family of plain Monte Carlo estimators satisfies Assumption \ref{assumption:lipschitz-condition-l2-stochastic-operator}.
\begin{lemma}
    \label{lemma:contraction-prop-T-hat-l2}
    For any $K \in \N^*$, the family $(\hat{T}^\theta_K)_{\theta \in \Theta}$ is an admissible family of stochastic operators that satisfy Assumptions \ref{assumption:lipschitz-condition-l2-stochastic-operator} with $L(\alpha, \beta) = \exp(\tau^{-1}(\beta - \alpha))$.
\end{lemma}

\begin{proof}
     Let $\Phi:  \R^{\mathbb Z} \rightarrow \R$ be defined by 
    $$
        \Phi ((q_k)_{k \in \Z}) =  -\tau \log \frac{1}{K}\sum_{k = 1}^{K} \exp \left(- \frac{q_k}{\tau}   \right).
    $$
    By the definition of $\hat{T}^\theta_K$ \eqref{def:T-plain-monte-carlo} we have $\hat{T}^\theta_K Q(s) = \Phi\left( \left(Q(s,A^{(\theta,k)}))\right)_{k \in \Z}\right)$, for any fixed $Q \in B_b (\mathcal S \times \mathcal A)$, therefore $ \mathbf{T}_K \coloneqq (\hat{T}^\theta_K)$ is an admissible family of stochastic operators in the sense of Definition \ref{def:admissible-stochastic-operator}.
    Notice that the proof of Lemma \ref{lemma:boundedness-of-iterates} can be replicated by replacing $T$ by its approximation $\hat{T}^\theta_K$ to get Assumption \ref{assumption:lipschitz-condition-l2-stochastic-operator}.\ref{item:boundedness}. Let $Q_1,Q_2:\mathcal S \times \mathcal A \times \Omega \rightarrow \R$ as in Assumption \ref{assumption:lipschitz-condition-l2-stochastic-operator}.\ref{item:lipschitz-property}.
    Notice that the computation for the proof of Lemma \ref{lemma:contraction-prop-T-l2} also applies when replacing $T$ by $\hat{T}^\theta_K$ and the integral by a sample average, as long as all other hypotheses remain.
    This allows to write
    \begin{equation*}
        \left| \hat{T}^\theta_K Q_1 (S) - \hat{T}^\theta_K Q_2 (S)\right| \leq e^{(\beta - \alpha)/ \tau} \frac{1}{K}\sum_{k = 1}^K|Q_1(S,A^{(\theta,k)}) - Q_2(S,A^{(\theta,k)})|.
    \end{equation*}
    Now, taking the $L^2$ norm and applying the triangle inequality yields
    \begin{equation}
    \label{eq:lipschitz-plain-mc}
        \begin{aligned}
            \left\| \hat{T}^\theta_K Q_1(S) - \hat{T}^\theta_K Q_2(S) \right\|_{L^2} & \leq  e^{(\beta - \alpha)/ \tau} \left\|Q_1(S,A^{(\theta,1)}) - Q_2(S,A^{(\theta, 1)})\right\|_{L^2},
        \end{aligned}
    \end{equation}
    hence $ \mathbf{T}_K$ satisfies Assumption \ref{assumption:lipschitz-condition-l2-stochastic-operator}.\ref{item:lipschitz-property}.
\end{proof}

We now state an upper bound on the bias of the plain Monte Carlo estimators.
\begin{lemma}
    \label{lemma:additional-MC-bias}
    Let $K \in \N^\star$ and define
    $$
     \sigma_{\mathbf T}(s,a)  \coloneqq \mathrm{Var} \left(T^0 Q_0 \left(S^0_{s,a}\right)\right)^{\frac{1}{2}},
        \quad
        \delta_{\mathbf T}(s,a) \coloneqq \left| \mathbb E \left[ T^0Q^\star\left(S^0_{s,a}\right) - TQ ^{\star}\left(S^0_{s,a}\right)\right] \right|.
    $$
    Then  
    $\sigma_{\mathbf T_K}(s,a) \leq L^\prime K^{-1/2}$ and $ \delta_{ \mathbf{T}_K} (s,a) \leq (L^\prime)^2 (2 \tau K) ^{-1}$,
    where $L^\prime  = \tau \left( e^{(\beta - \alpha) / \tau} - 1 \right)$.
\end{lemma}
\begin{proof}
    First, it is clear by the proof of Lemma \ref{lemma:boundedness-of-iterates} that $\hat{T}_K^0 Q_0(S^0_{s,a}),\hat{T}_K^0 Q^\star(S^0_{s,a})$ and $TQ^\star(S^0_{s,a})$ belong to the interval \(\left[ \alpha, \beta  \right]\),
    hence 
    $$
    \exp \left(-\hat{T}_K^0 Q_0(S^0_{s,a}) / \tau\right),\exp \left(- \hat{T}_K^0 Q^\star(S^0_{s,a})/ \tau \right),\exp \left(-TQ^\star(S^0_{s,a}) / \tau \right) \in \left[e^{-\beta / \tau}, e ^{-\alpha/ \tau}\right].
    $$
    The function $g:x \mapsto -\tau \log x$ is Lipschitz on $\left[e^{-\beta / \tau}, e ^{-\alpha/ \tau}\right]$, with a Lipschitz constant given by
$ \tau e^{\beta/ \tau}$.
Let 
$$
    m (S_{s,a}^0) = \int_{\cA} \exp(- Q^\star(S_{s,a}^0, b) / \tau) \mu(\d b), \quad \Sigma_K = \frac{1}{K}\sum_{k = 1}^K \exp(-Q^\star(S_{s,a}^0, A^{(0,k)}) / \tau).
    $$
We have
    \[
    \begin{aligned}
    \sigma^2_{\mathbf T_K}(s,a) &= \mathrm{Var} \left[ g \left( \Sigma_K \right) \right] \\
    & \leq \mathbb E \left|  g \left( \Sigma_K \right) - g \left(\mathbb E \exp \left(-Q(S^0_{s,a}, A^{(0,1)}) / \tau \right) \right)\right|^2 \\
    & \leq \frac{\tau^2 e^{2\beta / \tau}}{K} \mathrm{Var} \left[ \exp \left(-Q \left(S_{s,a}^0,A^{(0,1)}\right) / \tau \right) \right] \\
    & \leq \frac{\tau^2 e^{2\beta / \tau}}{K} \left( e^{-\alpha / \tau} - e^{-\beta / \tau}  \right)^2 = \frac{(L^\prime)^2}{K}.
    \end{aligned}
    \]
     Since $g$ is smooth on $(0,\infty)$ and $m (S_{s,a}^0),\Sigma_K \in [e^{-\beta / \tau}, e^{-\alpha / \tau}]$, we perform a Taylor expansion with Lagrange remainder
    \[
    g(\Sigma_K) = g(m(S_{s,a}^0)) + g^\prime(m(S^0_{s,a})) (\Sigma_K - m(S^0_{s,a})) + \frac{g^{\prime \prime}(\xi_K)}{2} (\Sigma_K - m(S^0_{s,a}))^2,
    \]
    where $\xi_K \in [\Sigma_K , m(S^0_{s,a})]$. Moreover, since $g^{\prime \prime}$ is strictly monotone and continuous, $\xi_K$ is defined uniquely and continuously as a function of $\Sigma_K , m(S^0_{s,a})$. In particular $\xi_K$ is a random variable. Hence
    \[
        \delta_{\mathbf T_K}(s,a) = \left| \mathbb E \left[ g^\prime(m(S^0_{s,a})) (\Sigma_K - m(S^0_{s,a})) + \frac{g^{\prime \prime}(\xi_K)}{2} (\Sigma_K - m(S^0_{s,a}))^2 \right] \right|.
    \]
    The first order term can be rewritten using the tower property
    \[
    \mathbb E \left[ g^\prime(m(S^0_{s,a})) (\Sigma_K - m(S^0_{s,a})) \right] = \mathbb E \left[g^\prime(m(S_{s,a}^0))\mathbb E \left[ (\Sigma_K - m(S^0_{s,a})) \mid S_{s,a}^0\right]\right],
    \]
    and by independence of $A^{(0,k)}$ and $S^0_{s,a}$, we have $\mathbb E \left[ (\Sigma_K - m(S^0_{s,a})) \mid S_{s,a}^0\right] = 0$.
    Hence
    \[
    \begin{aligned}
        \delta_{\mathbf T_K} (s,a) &= \left| \mathbb E \left[ \frac{g^{\prime \prime}(\xi_K)}{2} (\Sigma_K - m(S^0_{s,a}))^2 \right] \right| \\ 
        &\leq \sup_{x \in [e^{-\beta / \tau}, e^{-\alpha / \tau}]} \frac{\tau}{2 x^2} \mathbb E\left[ \left( \Sigma_K - m(S^0_{s,a}) \right)^2 \right]  \\
        &= \frac{\tau e^{2\beta / \tau}}{2K} \mathbb E \left[ \left(\exp( - Q^\star (S^0_{s,a}, A^{(0,1)})) - m(S_{s,a}^0)\right)^2 \right] \\
        & \leq \frac{\tau e^{2 \beta / \tau}}{2K} \left(e^{-\alpha / \tau} - e^{-\beta / \tau}\right)^2 = \frac{(L^\prime)^2}{2\tau K}.
    \end{aligned}
    \]
\end{proof}

\begin{proof}[Proof of Theorem \ref{thm:simple-iterative-MC}]
We consider the   estimator $Q^\theta_{n,M,\mathbf T_K }$ with fixed parameters $M, K \in \N^*$ and drop the indices $M,\mathbf T_K $ for legibility. For all  $(s,a) \in \cS \times \cA$,

\[
	\left\|Q_{n}^\theta(s,a) - Q^\star(s,a) \right\|_{L^2} \leq \sqrt{\mathrm{Var} Q_n^\theta(s,a)} + \left| \mathbb E \left[ Q^\theta_n(s,a) - Q^\star(s,a) \right] \right|.
\]
The variance can be upper bounded by independence
\[
	\mathrm{Var} Q_n^\theta(s,a) = \frac{\gamma^2}{M} \mathrm{Var} \left( \hat{T}^\theta_K Q_{n-1}(s,a) \right) \leq \frac{\gamma^2C}{M},
\]
where $C = \left( \beta - \alpha \right)^2$ since all iterates $Q_{n}, n \in \N$ are bounded in $[\alpha, \beta]$.

The bias can be bounded   using the Bellman equation for $Q^\star $   and the triangle inequality 
\[
	\begin{aligned}
		\left| \mathbb E \left[ Q^\theta_n(s,a) - Q^\star(s,a)\right] \right| &\leq \gamma \mathbb E \left| \hat{T}_K^\theta Q^\theta_{n-1}(S_{s,a}^\theta) - \hat{T}_K^\theta Q^\star(S_{s,a}^\theta)\right| \\
        &+\gamma \left|\mathbb E \left[\hat{T}_K^\theta Q^\star(S_{s,a}^\theta)  - TQ^\star(S^\theta_{s,a})\right]\right| \\
										      &\leq \gamma L \sup_{s,a} \left\|Q^\theta_{n-1}(s,a) - Q^\star(s,a) \|_{L^2} + \gamma \right\| \delta_{\mathbf T}\|_{\infty},
\end{aligned}
\]	
where $L = \exp \left( \tau^{-1}(\beta - \alpha) \right)$ is the Lipschitz constant of $\hat{T}_K^\theta$ given by Lemma \ref{lemma:contraction-prop-T-hat-l2} and $\delta_{\mathbf T}(s,a) = \left|\mathbb E \left[\hat{T}_K^\theta Q^\star(S_{s,a}^\theta)  - TQ^\star(S^\theta_{s,a})\right]\right|$. Hence we get the recursive bound
\[
	E_n \leq \frac{\gamma \sqrt{C}}{\sqrt{M}} + \gamma L E_{n-1} + \gamma \| \delta_{\mathbf T}\|_{\infty},
\]
which implies, assuming that $\gamma L < 1$
\[
	\begin{aligned}
		E_n &\leq \frac{\gamma \sqrt{C}}{\sqrt{M}} \frac{ 1 - (\gamma L)^n}{1 - \gamma L} + \frac{\gamma \| \delta_{\mathbf T}\|_{\infty}(1 - (\gamma L)^n)}{1 - \gamma L} + (\gamma L)^n E_0 \\
		    & \leq \frac{\gamma \sqrt{C}}{\sqrt{M}} \frac{1}{1 - \gamma L} + \frac{\gamma \| \delta_{\mathbf T}\|_{\infty}}{1 - \gamma L} + (\gamma L)^n E_0.
\end{aligned}
	\]
By Lemma \ref{lemma:additional-MC-bias}, we know that $\|\delta_{\mathbf T} \|_{\infty} \leq (L^\prime)^2 (2 \tau K)^{-1}$, where 
$L^\prime =  \tau (L - 1) $ is a constant depending only on $c_{\min}, c_{\max}, \gamma,\tau$, hence the final bound is
\begin{equation*}
	E_{n} \leq \frac{\gamma \sqrt{C}}{\sqrt{M}(1 - \gamma L)} + \frac{\gamma (L^\prime)^2}{2 \tau K(1 - \gamma L)} + (\gamma L)^n E_0,
	\end{equation*}
    with   $E_0\coloneqq \left\|Q_0 - Q^\star\right\|_{\infty}$. 
The sample complexity of this estimator is simply $(MK)^n$. Therefore, assuming that $\gamma L < 1$, we can get an $\varepsilon$-approximation by choosing 
\[
	n = \left\lceil \frac{\log \varepsilon - \log 3E_0}{\log \gamma L}\right\rceil, M = \left\lceil 9\frac{\gamma^2C}{(1-\gamma L)^2\varepsilon^2}\right\rceil, K = \left\lceil 3 \frac{\gamma (L^\prime)^2}{2 \tau (1-\gamma L)\varepsilon}\right\rceil.
\]
We get $E_n \leq \varepsilon/ 3 + \varepsilon/3 + \varepsilon/3 = \varepsilon$ and the sample complexity is of order $\varepsilon^{ \frac{-3 \log \varepsilon}{\log \gamma L}(1 + o(1))}$.
\end{proof}
\section{Proofs of  Error Bounds for MLMC Estimators}
\label{section:analysis-of-the-error}
In this section, we present a rigorous error analysis of the  general MLMC estimator   in Definition~\ref{def:general-MLMC-estimator}, and specialize the error bounds to the MLMCb and MLMCu estimators.

\subsection{Sketch of the Analysis}
The $L^2$ error of the estimator can be decomposed using the triangle inequality
\begin{equation}
    \label{eq:error-decomposition}
    \begin{aligned}
         & \left\| Q ^{0}_{n,M,\mathbf T}(s,a) - Q ^{\star}(s,a) \right\| _{L^2}  \leq \left\| Q ^{0}_{n,M,\mathbf T}(s,a) - \mathbb E \left[ Q ^{0} _{n,M,\mathbf T}(s,a) \right] \right\|_{L^2} \\
         & + \left| \mathbb E \left[ Q ^{0}_{n,M, \mathbf T}(s,a) \right] - \mathbb E \left[ \hat{Q} ^{0}_{n,M, \mathbf T} (s,a) \right] \right|
        + \left| \mathbb E \left[ \hat{Q} ^{0}_{n,M,\mathbf T}(s,a) \right]  - Q ^{\star}(s,a) \right|.
    \end{aligned}
\end{equation}
The three terms can be respectively interpreted as:
\begin{itemize}
    \item the variance of our estimator: the independence between levels is crucial to decompose the variance into a sum of variances at lower levels;
    \item the bias due to the truncation;
    \item the bias due to the number of iterations $n$: as the optimal Q-function is a solution to a fixed point equation, we should expect an exponential factor in $n$. The bias of the operator $T^\theta$ also appears in that term.
\end{itemize}
Once we have estimates for each of these sources of error, we combine them to get a recursive bound on the total error in terms of total errors at lower levels, and then use a discrete Gronwall-type inequality presented in Lemma \ref{lemma:refined-gronwall}.

We now study each term of the upper bound \eqref{eq:error-decomposition}. From now on we work under Assumptions \ref{assumption:data}, \ref{assumption:random-variables} and \ref{assumption:lipschitz-condition-l2-stochastic-operator} unless specified otherwise.

\subsection{Distributional Properties of the General MLMC Estimator}
\label{subsection:technal-lemmas}

We state a lemma which ensures that $\theta$ only contributes as an index in the definition of the general MLMC estimator given by \eqref{eq:MLMC-estimator-approx-T} and state some measurability properties of the general MLMC estimator in the framework of admissible stochastic operators.
\begin{lemma}
    \label{lemma:theta-is-just-an-index}
    Suppose Assumptions \ref{assumption:data} and \ref{assumption:random-variables} hold. Let $\mathbf T$ be an admissible family of stochastic operators in the sense of Definition \ref{def:admissible-stochastic-operator}, let $Q^\theta_{n,M,\mathbf T}$ be defined as in Definition \ref{def:general-MLMC-estimator}. Then, for all $(s,a) \in \mathcal S \times \mathcal A, n \in \N, M \in \N^*, \theta \in \Theta$:
    \begin{itemize}
	    \item[(i)] There exists a Polish space $\mathcal U_n$, a measurable function $f_n : \mathcal S \times \mathcal A \times \mathcal U_n \rightarrow \R$ and random variables $U_n^\theta$ valued in $\mathcal U_n$ such that $Q_{n,M,\mathbf T}^\theta(s,a) = f_n(s,a,U_n^\theta)$ for all $\theta \in \Theta, (s,a) \in \mathcal S \times \mathcal A$. Moreover, these random variables can be taken such that if $\theta \in \Z^m$, $U^{\theta}_n$ is independent of $\left((S^{\theta^\prime}_{s,a})_{(s,a)\in \mathcal S \times \mathcal A}, A^{\theta^\prime} \right)_{\theta^\prime \in \cup_{m^\prime < m + 2} \Z^{m^\prime}}$. In particular, $Q^\theta_{n,M,\mathbf T}(s,a)$ is a random variable.
        
        \item[(ii)] $\sigma \left(Q^\theta_{n,M,\mathbf T}(s,a)\right) \subseteq \sigma \left(\left(A^{(\theta,\theta^\prime)}\right)_{\theta^\prime \in \bigcup_{m \geq 3}\Z^m}, \left(S^{(\theta,\theta^\prime)}_{s^\prime, a^\prime}\right)_{\theta^\prime \in  \bigcup_{m \geq 2}\Z^m, s^\prime \in \mathcal S, a^\prime \in \mathcal A}\right)$;
        
        \item[(iii)] if $\theta \in \Z^m$, $Q^\theta_{n,M,\mathbf T}(s,a)$ is independent from $\left(A^{\theta^\prime}\right)_{\theta^\prime \in \bigcup_{m^\prime \leq m + 2}\Z^{m^\prime}}$ and from \\ $\left(S^{\theta^\prime}_{s^\prime, a^\prime}\right)_{\theta^\prime \in  \bigcup_{m^\prime \leq m + 1}\Z^{m^\prime}, s^\prime \in \mathcal S, a^\prime \in \mathcal A}$;
        
        \item[(iv)] if $\theta \neq \theta^\prime$ for $\theta, \theta^\prime \in \Z^m$, then for all $n^\prime \in \N, M^\prime \in \N^*$, $Q^\theta_{n,M,\mathbf T}(s,a)$ and $Q^{\theta^\prime}_{n^\prime, M^\prime, \mathbf T}(s,a)$ are independent;
        \item[(v)] $\left(Q^{\theta^\prime}_{n,M,\mathbf T} \right)_{\theta^\prime \in \Theta}$ are identically distributed.
    \end{itemize}
\end{lemma}
\begin{proof}
	Fix $M \in \N^*, \theta \in \Theta$. We start by proving (i). For $n = 0$, $Q_0$ is measurable and bounded by assumption, so we just take $f_0 = Q_0, \mathcal U_0 = \emptyset$. In order to show the property for $n=1$, we first show that $T^\theta Q_0 (S_{s,a}^\theta)$ can be represented as $g_1(s,a,U_1)$ for some random variable $U_1$ valued in a Polish space to be defined. Using Assumption \ref{assumption:random-variables} and Definition \ref{def:admissible-stochastic-operator} gives the following representation of $ T^\theta Q_0(S_{s,a}^\theta)$,
	\begin{equation}
	\label{eq:representation-TQ0}
	T^\theta Q_0(S_{s,a}^\theta) = \Phi \left(\left(Q_0 \left(S_{s,a}^\theta,A^{(\theta,k)} \right)\right)_{k \in \Z}, K^\theta\right) = \Phi \left( \left(f_0 \left(f \left(s,a,U^\theta \right),A^{(\theta, k)}\right)\right)_{k \in \Z}, K^\theta\right).
	\end{equation}
	Hence, by setting $g_1(s^\prime,a^\prime,u,(a_k)_{k \in \Z},k)) = \Phi\left(\left(f_0(f(s,a,u),a_k)\right)_{k \in \Z},k\right)$, $\mathcal U_1 = [0,1] \times \mathcal A ^{\Z} \times \N$ and $U_1^\theta = (U^\theta, (A^{(\theta,k)})_{k \in \Z}, K^\theta)$, we have $T^\theta Q_0 (S_{s,a}^\theta) = g_1(s,a,U_1^\theta)$. It is easy to check that $g_1$ is measurable, and that $\mathcal U_1$ is Polish as a countable product of Polish spaces. Moreover, $U^\theta_1$ is independent of $\left((S^{\theta^\prime}_{s,a} )_{(s,a) \in \mathcal S \times \mathcal A}, A^{\theta^\prime}\right)$ for $\theta^\prime \in \Z^{m^\prime}$ with $m^\prime < m$ where $\theta \in  \Z^m$. Finally, notice that by definition $\sigma \left(Q_{1,M,\mathbf T}^\theta(s,a) \right) \subseteq \sigma \left( T^{\theta^\prime}Q_0(S^{\theta^\prime}_{s,a}): \theta^\prime \in \Z^{m^\prime}\right)$ for $m^\prime = m + 2$, hence this shows that $Q_{1,M,\mathbf T}^\theta(s,a) = f_1(s,a,U_1^\theta)$ for some measurable function $f_1$, for a random variable $U^\theta_1$ independent of $\left((S^{\theta^\prime}_{s,a})_{(s,a)\in \mathcal S \times \mathcal A}, A^{\theta^\prime} \right)$.  The proof extends to any $n \geq 2$ by recursion. The proof of (ii)-(v) is exactly as in Lemma 3.9 in \citet{hutzenthaler2021multilevel}.
\end{proof}
 This allows to only consider the computation of the MLMC estimator for $\theta = 0$.

\subsection{Analysis of the Monte Carlo Error}
\label{subsection:analysis-mc-error}
We first analyse the variance of the estimator given by Definition \ref{def:general-MLMC-estimator}.
\begin{proposition}
    \label{prop:variance-monte-carlo}
    Let $(s,a) \in \mathcal S \times \mathcal A$, and let $n,M$ be positive integers. Assume that $\mathbf T$ is an admissible family of stochastic operators satisfying Assumption \ref{assumption:lipschitz-condition-l2-stochastic-operator} with $L = L(\alpha, \beta)$.
    Then for any $\theta \in \Theta$,
    \[
        \mathrm{Var} \left(Q _{n,M,\mathbf T} ^{\theta}(s,a)\right) = \mathrm{Var} \left(Q _{n,M,\mathbf T} ^{0}(s,a) \right) \leq \mathrm{Var} \left(\hat{Q} _{n,M,\mathbf T} ^{0}(s,a)\right),
    \]
    and we have
    \begin{equation*}
        \label{eq:monte-carlo-error-final}
        \mathrm{Var} \left(\hat{Q} _{n,M,\mathbf T} ^{0}(s,a)\right)\leq \frac{\gamma^2 \sigma_{\mathbf T}(s,a)^2}{M^n}+ \sum_{l = 1}^n \frac{(\gamma L)^2}{M^{n-l}} \left\| Q_{l,M,\mathbf T} ^{(0,l)} \left(S_{s,a}^{0},A^{0} \right) - Q_{l-1,M,\mathbf T} ^{(0,-l)} \left(S_{s,a}^0,A^{0} \right) \right\|^2_{L^2},
    \end{equation*}
    where $ \sigma_{\mathbf T}(s,a)^2 = \mathrm{Var} \left(T^0 Q_0 (S^0_{s,a}) \right)$.
\end{proposition}

\begin{proof}
    Throughout the proof, we assume that $\sigma_{\mathbf T}(s,a) < \infty$, otherwise the result is trivial. Observe that $\mathrm{Var} (Q _{n,M,\mathbf T} ^{\theta}(s,a)) = \mathrm{Var}(Q _{n,M,\mathbf T} ^{0}(s,a))$ is an immediate consequence of Lemma \ref{lemma:theta-is-just-an-index}.(v). Since $Q^0_{n,M,\mathbf T}$ is a truncated version of $\hat{Q}^0_{n,M,\mathbf T}$, it is clear that $\mathrm{Var}(Q _{n,M,\mathbf T} ^{0}(s,a)) \leq \mathrm{Var}(\hat{Q} _{n,M,\mathbf T} ^{0}(s,a))$. Moreover, it is easy to see that $\left(T^{\theta^\prime}Q^{\theta^{\prime\prime}}_{l,M,\mathbf T}\left(S_{s,a}^{\theta^\prime}\right)\right)_{\theta^\prime, \theta^{\prime \prime} \in \Z^m}$ are i.i.d.~for any fixed $m,l,M,s$ and $a$. Combining that with Lemma \ref{lemma:theta-is-just-an-index}.(iv), we can use independence to get
    \begin{align*}
        \mathrm{Var} \left(Q^0_{n,M,\mathbf T}(s,a) \right) & \leq  \mathrm{Var}\left(\hat{Q}^0_{n,M,\mathbf T}(s,a)\right) \\
            & = \frac{\gamma^2}{M ^{n}} \mathrm{Var}\left[T^{(0,0,1)}Q_0 \left(S^0_{s,a} \right)\right] \\
            & + \sum_{l = 1}^{n-1} \frac{\gamma^2}{M ^{n - l}} \mathrm{Var} \left[ T^{(0,l,1)}Q _{l,M,\mathbf T}^{(0,l,1)} \left(S_{s,a}^{(0,l,1)} \right)   - T^{(0,l,1)}Q _{l,M,\mathbf T}^{(0,-l,1)}\left(S_{s,a}^{(0,l,1)} \right) \right].
    \end{align*}
    By distributional equality, we have $\mathrm{Var}[T^{(0,0,1)}Q_0(S^0_{s,a})] = \sigma_{\mathbf T}(s,a)^2$. Using the fact that $\| X \|^2_{L^2} \geq \mathrm{Var}(X)$ and the Lipschitz property \eqref{eq:lipschitz-condition-l2-stochastic-operator}, we have 
    \begin{equation*}
        \begin{aligned}
            &\mathrm{Var} \left(\hat{Q} _{n,M,\mathbf T} ^{0}(s,a)\right) \\
            & \leq \frac{\gamma^2}{M ^{n}} \sigma _{\mathbf T}(s,a) ^{2}     + \sum_{l = 1}^{n} \frac{\gamma^2}{M ^{n - l}} \left\| T^{(0,l,1)} Q_{l,M,\mathbf T} ^{(0,l,1)}\left(S^{(0,l,1)}_{s,a}\right) - T^{(0,l,1)} Q_{l-1,M,\mathbf T} ^{(0,-l,1)} \left(S^{(0,l,1)}_{s,a}\right) \right\|_{L^2}^2\\
            & \leq \frac{\gamma^2}{M^n}\sigma_{\mathbf T}(s,a)^2  + \sum_{l = 1}^n \frac{(\gamma L)^2}{M^{n-l}} \left\| Q_{l,M,\mathbf T} ^{(0,l,1)}\left(S_{s,a}^{(0,l,1)},A^{(0,l,1,1)}\right) - Q_{l-1,M,\mathbf T} ^{(0,-l,1)}\left(S_{s,a}^{(0,l,1)},A^{(0,l,1,1)}\right) \right\|^2_{L^2} .
        \end{aligned}
    \end{equation*}
    Finally, observing that $Q_{l,M,\mathbf T} ^{(0,l,1)}(S_{s,a}^{(0,l,1)},A^{(0,l,1,1)}) - Q_{l-1,M,\mathbf T} ^{(0,-l,1)}(S_{s,a}^{(0,l,1)},A^{(0,l,1,1)})$ has the same distribution as $Q_{l,M,\mathbf T} ^{(0,l)}(S_{s,a}^{0},A^{0}) - Q_{l-1,M,\mathbf T} ^{(0,-l)}(S_{s,a}^0,A^{0})$ (Lemma \ref{lemma:theta-is-just-an-index}) concludes the proof.
\end{proof}
\begin{remark}
    
    Notice that, because we eventually want to take a supremum outside of the expectation, we need to rely on a Lipschitz property similar to the one satisfied by $T$ in Lemma \ref{lemma:contraction-prop-T-l2}. Hence Assumption \ref{assumption:lipschitz-condition-l2-stochastic-operator}.\ref{item:lipschitz-property} is crucial to carry out the recursive analysis, and we cannot simply rely on a $\| \cdot \|_{\infty}$-Lipschitz property of $T^\theta$.
\end{remark}
\subsection{Analysis of the Truncation Error}
\label{subsection:analysis-truncation}
We now look at the following error term corresponding to the truncation error in \eqref{eq:error-decomposition},
\begin{equation*}
    \label{eq:truncation-error}
    \delta^{\mathrm{trunc}}_{n,M,\mathbf T}(s,a) \coloneqq \left| \mathbb E \left[ Q ^{0}_{n,M,\mathbf T}(s,a) \right] - \mathbb E \left[ \hat{Q} ^{0}_{n,M,\mathbf T} (s,a) \right] \right|.
\end{equation*}

\begin{proposition}
    \label{proposition:truncation-cauchy-schwarz}
    We have for any $ \in \N, M \in \N^*, \theta \in \Theta$,
    \begin{equation*}
        \label{eq:truncation-upper-bound-cauchy-schwarz}
        \delta^{\mathrm{trunc}}_{n,M,\mathbf T}(s,a)^2 \leq \max\left\{\mathbb P \left(\hat{Q}^\theta_{n,M, \mathbf T}(s,a) < \alpha \right), \mathbb P \left(\hat Q^\theta_{n,M, \mathbf T} (s,a) > \beta \right)\right\} \mathrm{Var}\left(\hat{Q}^\theta_{n,M, \mathbf T}(s,a)\right).
    \end{equation*}
\end{proposition}

\begin{proof}
    We assume that $\mathrm{Var}(\hat{Q}^\theta_{n,M, \mathbf T}(s,a)) < \infty$, otherwise the result is trivial. Observe that the upper bound is independent of $\theta$ due to Lemma \ref{lemma:theta-is-just-an-index}. For \((s,a) \in \mathcal S \times \mathcal A\), dropping the indices \(M,\mathbf T, \theta\)
    \[
        \delta^{\mathrm{trunc}}_n(s,a) = \left| \mathbb E \left[ \mathbf 1 _{\hat{Q}_n(s,a) < \alpha} \left(\alpha - \hat{Q}_n(s,a)\right)  + \mathbf 1 _{\hat{Q}_n(s,a) > \beta}\left(\beta - \hat{Q}_n(s,a)\right)\right] \right|.
    \]
    Notice that the two terms are of opposite sign, and observe further that $\mathbb E \hat Q_n(s,a) = c(s,a) + \gamma\mathbb E T^0  Q_{n - 1}(S^0_{s,a}) \in [\alpha, \beta]$ by Assumption \ref{assumption:lipschitz-condition-l2-stochastic-operator}.\ref{item:boundedness}.
   By the Cauchy--Schwarz inequality,
    \begin{align*}
        \mathbb E \left[ 1 _{\hat{Q_n}(s,a) < \alpha} \left(\alpha - \hat{Q}_n(s,a)\right) \right]^2 & \leq   \mathbb E \left[ 1 _{\hat{Q_n}(s,a) < \alpha} \left(\mathbb E \left[\hat Q_n(s,a)\right] - \hat{Q}_n(s,a)\right) \right]^2 \\
        & \leq \mathbb P \left(\hat{Q}_n(s,a) < \alpha \right) \mathrm{Var}\left(\hat{Q}_n(s,a)\right).
    \end{align*}
    Similarly, we have
    \[
        \mathbb E \left[1 _{\hat{Q}_n(s,a) > \beta}\left(\beta - \hat{Q}_n(s,a)\right) \right]^2 \leq \mathbb P \left(\hat Q_n (s,a) > \beta \right)\mathrm{Var} \left(\hat{Q}_n(s,a)\right).
    \]
    Therefore, we have the following upper bound on the truncation error
    \[
        \delta^{\mathrm{trunc}}_{n}(s,a)^2 \leq \max \left\{ \mathbb P \left(\hat{Q}_n(s,a) < \alpha \right), \mathbb P \left(\hat Q_n (s,a) > \beta \right)\right\} \mathrm{Var}\left(\hat{Q}_n(s,a)\right).
    \]
\end{proof}
\subsection{Analysis of the Bias}
\label{subsection:analysis-bias}
We now look at the bias of the untruncated estimator
\[
    \delta ^{\mathrm{bias}}_{n,M, \mathbf T} \coloneqq \left| \mathbb E \left[\hat{Q} ^{0} _{n,M,\mathbf T}(s,a)\right] - Q ^{\star}(s,a) \right|.
\]
\begin{proposition}
    \label{lemma:bias-upper-bound}
    For any $(s,a) \in \mathcal S \times \mathcal A$,
    \begin{equation*}
        \label{eq:bias-upper-bound}
        \delta ^{\mathrm{bias}}_{n,M, \mathbf T} \leq \gamma L  \left\| Q^0_{n-1,M,\mathbf T} \left(S^0_{s,a}, A^{0} \right) - Q^\star \left(S^0_{s,a}, A^{0}\right) \right\|_{L^2} + \gamma \delta_{\mathbf T}(s,a),
    \end{equation*}
    where $\delta_{\mathbf T}$ is defined by \eqref{eq:def-sigma-t}.
\end{proposition}

\begin{proof}
    For $s \in \mathcal S, a \in \mathcal A, n \geq 1$, we have by a simple telescoping argument, using Assumption \ref{assumption:lipschitz-condition-l2-stochastic-operator}, the Bellman equation for  $Q^\star$  and the Cauchy--Schwarz inequality,
    \begin{equation*}
        \label{eq:bias-error}
        \begin{aligned}
            \delta ^{\text{bias}}_{n,M, \mathbf T} & = \gamma \left| \mathbb E \left[T^0Q_{n-1,M,\mathbf T}^{0}\left(S^0_{s,a}\right) \right] - \mathbb E \left[ TQ ^{\star}\left(S_{s,a}^0\right)\right] \right|        \\
      & \leq \gamma \left| \mathbb E \left[ T^0Q_{n-1,M,\mathbf T}^{0}\left(S^0_{s,a}\right) \right] - \mathbb E \left[ T^0Q ^{\star}\left(S^0_{s,a}\right)\right] \right|    \\  
      & + \gamma \left| \mathbb E \left[ T^0Q^\star \left(S^0_{s,a}\right) \right] - \mathbb E \left[TQ ^{\star}\left(S^0_{s,a}\right)\right] \right|                                 \\
                     & \leq \gamma L  \mathbb E \left|Q^0_{n-1,M,\mathbf T} \left(S^0_{s,a}, A^{(0,1)} \right) - Q^\star \left(S^0_{s,a}, A^{(0,1)}\right)\right|  + \gamma \delta_{\mathbf T} (s,a) \\
                    & \leq \gamma L  \left\| Q^0_{n-1,M,\mathbf T} \left(S^0_{s,a}, A^{(0, 1)}\right) - Q^\star \left(S^0_{s,a}, A^{(0,1)}\right) \right\|_{L^2} + \gamma \delta_{\mathbf T}(s,a).
        \end{aligned}
    \end{equation*}
    Finally, notice that $Q^0_{n-1,M,\mathbf T} (S^0_{s,a}, A^{(0, 1)}) - Q^\star (S^0_{s,a}, A^{(0,1)})$ and $Q^0_{n-1,M,\mathbf T} (S^0_{s,a}, A^{0}) - \\ Q^\star (S^0_{s,a}, A^{0})$ have the same distribution by Lemma \ref{lemma:theta-is-just-an-index}.
\end{proof}

\subsection{Putting Everything Together: Global Error}
\label{subsection:putting-everything-together}
We first prove a lemma to help us work with a supremum over the state and action spaces.
\begin{lemma}
    \label{lemma:from-S-to-sup-s}
    Let $Q: \mathcal S \times \mathcal A \times \Omega \rightarrow \R$ be a measurable bounded function. Suppose there exists a Polish space $\mathcal X$, a random variable $X$ valued in $\mathcal X$ and a measurable function $g: \mathcal S \times \mathcal A \times \mathcal X \rightarrow \R$ such that $Q(s,a, \omega) = g(s,a,X(\omega))$ for all $(s,a,\omega) \in \mathcal S \times \mathcal A \times \Omega$. Let $(s,a) \in \mathcal S \times \mathcal A$, let $S$ be a random variable distributed according to $P(\cdot | s, a)$ and let $A$ be a random variable distributed according to $\mu$ such that $S,A$ are independent of $X$. Then
    \[
        \mathbb E[Q(S,A)] \leq \sup_{s^\prime \in \mathcal S, a^\prime \in \mathcal A} \mathbb E[Q(s^\prime,a^\prime)].
    \]
\end{lemma}
\begin{proof}
	Since $Q$ is assumed to be bounded, one can assume that $g$ is bounded without loss of generality. By independence of $X$ and $(S,A)$, we can apply Proposition 1.12 in \citet{da2014stochastic} and write $\mathbb E[Q(S,A) \mid S,A] = \psi(S,A)$, where $\psi(s^\prime,a^\prime) = \mathbb E[g(s^\prime,a^\prime,X)] = \mathbb E[Q(s^\prime,a^\prime)]$. Hence, we have
	$$
	\begin{aligned}
		\mathbb E [Q(S,A)] &= \mathbb E \left[ \mathbb E [Q(S,A) \mid S,A]\right] 
        = \mathbb E [\psi(S,A)] 
          \\ &\leq \mathbb E \left[\sup_{s^\prime \in \mathcal S, a^\prime \in \mathcal A} \psi(s^\prime,a^\prime)\right] 
				  = \sup_{s^\prime \in \mathcal S, a^\prime \in \mathcal A} \psi(s^\prime,a^\prime) = \sup_{s^\prime \in \mathcal S, a^\prime \in \mathcal A} \mathbb E [Q(s^\prime,a^\prime)].
	\end{aligned}
	$$
\end{proof}
We are now ready to give the upper bound on $E_{n,M,\mathbf T}$.
\begin{proof}[Proof of Theorem \ref{thm:global-MLMC-error}]
    Combining Propositions \ref{prop:variance-monte-carlo}, \ref{proposition:truncation-cauchy-schwarz} and \ref{lemma:bias-upper-bound} yields
    \begin{equation*}
        \label{eq:full-error}
        \begin{aligned}
             & \left\| Q ^{0}_{n,M,\mathbf T}(s,a) - Q ^{\star}(s,a) \right\| _{L^2}  \leq \left\| Q ^{0}_{n,M,\mathbf T}(s,a) - \mathbb E \left[ Q ^{0} _{n,M,\mathbf T}(s,a) \right] \right\|_{L^2}                                                                           \\
             & + \left| \mathbb E \left[ Q ^{0}_{n,M,\mathbf T}(s,a) \right] - \mathbb E \left[ \hat{Q} ^{0}_{n,M,\mathbf T} (s,a) \right] \right|
            + \left| \mathbb E \left[ \hat{Q} ^{0}_{n,M,\mathbf T}(s,a) \right]  - Q ^{\star}(s,a) \right|                                                                                                                                                                      \\
             & \leq \frac{(1 + \rho_{n,M})\gamma \sigma_{\mathbf T}(s,a)}{\sqrt{M ^{n}}} + \sum_{l = 1}^{n - 1}\frac{(1 + \rho_{n,M})\gamma L}{\sqrt{M ^{n - l}}}  \left\| Q_{l,M,\mathbf T} ^{(0,l)} \left(S_{s,a}^0,A^{0} \right) - Q_{l-1,M,\mathbf T} ^{(0,-l)} \left(S_{s,a}^0,A^{0}\right)\right\|_{L^2} \\
             & + \gamma L \left\|Q^0_{n-1,M,\mathbf T} \left(S^0_{s,a}, A^{(0,1)}\right) - Q^\star \left(S^0_{s,a}, A^{(0,1)}\right) \right\|_{L^2} + \gamma \delta_{\mathbf T}(s,a),
        \end{aligned}
    \end{equation*}
    where $\rho_{n,M}$ is defined by \eqref{eq:def-rho}.
    Notice that under the notations of Lemma \ref{lemma:theta-is-just-an-index}, $U^{(0,l)}_l$, $U^{(0,-l)}_{l-1}$ for $l =1, \ldots, n-1$ and $U^0_{n-1}$ are independent of $(S_{s,a}^0,A^{0})$. Recall that $\tilde \gamma = (1 + \max_{k \leq n} \rho_{k,M})\gamma$. By Lemma \ref{lemma:from-S-to-sup-s}, denoting $e_{n,M,\mathbf T}(s,a) \coloneqq \| Q ^{0}_{n,M,\mathbf T}(s,a) - Q ^{*}(s,a) \|_{L^2}$, the global error becomes
    \begin{equation*}
        \label{eq:full-error-sup-s}
        \begin{aligned}
            e_{n,M,\mathbf T}(s,a) & \leq \frac{\tilde \gamma \sigma_{\mathbf T}(s,a)}{\sqrt{M ^{n}}} + \sum_{l = 1}^{n - 1}\frac{\tilde \gamma L}{\sqrt{M ^{n - l}}}  \sup_{(s^\prime,a^\prime) \in \mathcal S \times \mathcal A}\left\| Q_{l,M,\mathbf T} ^{(0,l)}(s^\prime,a^\prime) - Q_{l-1,M,\mathbf T} ^{(0,-l)}(s^\prime,a^\prime)\right\|_{L^2} \\
                                   & + \gamma L \sup_{(s^\prime,a^\prime) \in \mathcal S \times \mathcal A}\left\|Q^0_{n-1,M,\mathbf T} (s^\prime, a^\prime) - Q^\star (s^\prime, a^\prime) \right\|_{L^2} + \gamma \delta_{\mathbf T}(s,a).
        \end{aligned}
    \end{equation*}
    Now, by a triangle inequality, we have
    \[
        \left\| Q_{l,M,\mathbf T} ^{(0,l)}(s^\prime,a^\prime) - Q_{l-1,M,\mathbf T} ^{(0,-l)}(s^\prime,a^\prime)\right\|_{L^2} \leq e_{l,M,\mathbf T}(s^\prime,a^\prime) + e_{l-1,M,\mathbf T}(s^\prime,a^\prime),
    \]
    which, by denoting $E_{n,M,\mathbf T} = \sup_{s,a} e_{n,M,\mathbf T}(s,a)$, yields
    \begin{equation}
        \label{eq:final-recursive-error}
        \begin{aligned}
        M ^{n / 2} E _{n,M,\mathbf T} &\leq \tilde \gamma \|\sigma _{\mathbf T}\|_{\infty} + \tilde \gamma L (1 + \sqrt{M}) \sum_{l = 0}^{n-2} M ^{l / 2} E _{l,M,\mathbf T} \\
        &+ L(\tilde{\gamma} + \gamma\sqrt{M}) M^{\frac{n-1}{2}} E_{n-1,M,\mathbf T} + \gamma \left \|\delta_{\mathbf T}\right\|_{\infty} M^{n/2}.
        \end{aligned}
    \end{equation}
    Now recall that the zeroth level error is given by $\|Q_0 - Q^\star\|_{\infty} = \sup_{s,a} |Q_0(s,a) - Q^\star(s,a)|$. 
    We now aim to apply the Gronwall-type inequality of Lemma \ref{lemma:refined-gronwall}, corresponding to a special case of Corollary 2.3 in \citet{hutzenthaler2022multilevel}, with
    \[
    \begin{aligned}
    a_l &= M^{l/2} E_{n,M,\mathbf T} \geq 0, \quad l = 0, \ldots, n,  \\
    \quad b_1 &= \max\left( \| Q_0 - Q^\star\|_{\infty},  \tilde{\gamma} \| \sigma_{\mathbf T} \|_{\infty}\right) \geq 0, \quad b_2 = \gamma\| \delta_{\mathbf T} \|_{\infty} \geq 0, \quad b_3 = \sqrt{M} \geq 0,\\
    \lambda_1 &=  L (\tilde {\gamma} + \gamma\sqrt M) \geq 0, \quad \lambda_2 = \tilde{\gamma} L (1 + \sqrt{M}) - \lambda_1 = (\tilde{\gamma} - \gamma) \sqrt M \geq 0. 
    \end{aligned}
    \]
    The fully recursive bound \eqref{eq:final-recursive-error} implies that for any $l = 0, \ldots, n$, we have $a_l \leq b_1 + b_2 b_3^n + \sum_{k = 0}^{l-1} \lambda_1 a_k + \lambda_2 a_{k-1}$. Moreover, the constant $\Lambda$ is given by
    \[
    \Lambda = \frac{1}{2} \left(1 + L(\tilde{\gamma} + \gamma \sqrt{M}) + \sqrt{\left(1 + L(\tilde{\gamma} + \gamma \sqrt{M})\right)^2 + 4 (\tilde{\gamma} - \gamma)\sqrt{M}} \right),
    \]
    and notice that we have
    \begin{equation}
    \label{eq:bound-lambda}
    \frac{\Lambda}{\sqrt{M}} = \frac{1}{2} \left(\gamma L + \frac{1 + \tilde{\gamma}L}{\sqrt{M}} + \sqrt{\left(\gamma L + \frac{1 + \tilde{\gamma}L}{\sqrt{M}}\right)^2 + 4\frac{\tilde{\gamma} - \gamma}{\sqrt{M}}}\right) \leq
    \gamma L + \frac{1 + \tilde{\gamma} L}{\sqrt{M}} + \frac{\sqrt{\tilde{\gamma} - \gamma}}{M^{1/4}} < 1,
    \end{equation}
    as a consequence of \eqref{eq:condition-M}.
    Hence we can apply Lemma \ref{lemma:refined-gronwall} to get
    \begin{equation}
    \label{eq:refined-gronwall-recursive-error}
        M^{n/2} E_{n,M,\mathbf T} \leq 
        \frac{3}{2} \left(\max\left( \| Q_0 - Q^\star\|_{\infty},  \tilde{\gamma} \| \sigma_{\mathbf T} \|_{\infty}\right) \Lambda^n + \gamma\| \delta_{\mathbf T} \|_{\infty} \frac{ M^{\frac{n+1}{2}} - \sqrt{M}\Lambda^n}{\sqrt{M} - \Lambda} \right). 
    \end{equation}
    
   We can divide \eqref{eq:refined-gronwall-recursive-error} by $M^{n/2}$ to get the following upper bound on the error
    \begin{multline}
        E_{n,M,\mathbf T} \leq \\ 
             \frac{3}{2}\left(\max\left( \| Q_0 - Q^\star\|_{\infty},  \tilde{\gamma} \| \sigma_{\mathbf T} \|_{\infty} \right) \left[\gamma L + \frac{1 + \tilde{\gamma} L}{\sqrt{M}} + \frac{\sqrt{\tilde{\gamma} - \gamma}}{M^{1/4}} \right]^n  + \gamma \|\delta_{\mathbf T} \|_{\infty}\frac{\sqrt{M} }{\sqrt{M} - \Lambda} \right),
    \end{multline}
    corresponding to the desired inequality \eqref{eq:final-MLMC-error}.
\end{proof}

\subsection{Specializing with a Plain Monte Carlo Estimator}
\label{subsection:specialising-error-plain-mc}

We now discuss the specialization of our general MLMC estimator to the plain Monte Carlo estimator for the regularized Bellman operator, that is the MLMCb estimator.
Lemma \ref{lemma:contraction-prop-T-hat-l2} shows that the family of plain Monte Carlo estimators satisfies Assumption \ref{assumption:lipschitz-condition-l2-stochastic-operator}.
 
The following corollary follows directly from 
Theorem \ref{thm:global-MLMC-error}, Lemmas \ref{lemma:contraction-prop-T-hat-l2} and \ref{lemma:additional-MC-bias},
the  bound \eqref{eq:bound-lambda}, and $\tilde{\gamma} \leq 2 \gamma$.

\begin{corollary}
    \label{cor:error-plain-monte-carlo}
Suppose Assumptions \ref{assumption:data} and \ref{assumption:random-variables} hold.
Let $L = \exp \left(\tau^{-1} (\beta - \alpha)\right)$, 
 $n \in \N$, and $M \in \N^*$ satisfy
  \begin{equation}
    \label{eq:condition-M-MC}
      \Lambda_M \coloneqq \gamma L + \frac{1 + 2 \gamma L}{\sqrt{M}} + \frac{\sqrt{\gamma}}{M^{1/4}} < 1.
  \end{equation}
If $ \mathbf{T}_K = (\hat{T}^\theta_K)_{\theta \in \Theta}$ for some $K \in \N^*$, then the   error of the MLMCb estimator is   bounded by
    \begin{equation*}
        \label{eq:final-MLMC-error-plain-monte-carlo}
        E _{n,M, \mathbf{T}_K} \leq
        \frac{3}{2}\left(\max\left( \| Q_0 - Q^\star\|_{\infty},  2 \gamma \| \sigma_{ \mathbf{T}_K} \|_{\infty} \right) \left(\Lambda_M\right)^n + \frac{\gamma (L^\prime)^2}{2 \tau K}\frac{1 }{1 - \Lambda_M} \right),
    \end{equation*}
    where  $L^\prime \coloneqq   \tau ( L - 1 )$, and $  \sigma_{ \mathbf{T}_K}$ is defined  in Theorem \ref{thm:global-MLMC-error}.
   
\end{corollary}

\subsection{Specializing with the Blanchet--Glynn Estimator}
\label{subsection:specialising-error-BG}

We first prove the unbiasedness of the Blanchet--Glynn estimator defined by \eqref{eq:T-blanchet-glynn}.

\begin{proof}[Proof of Proposition \ref{prop:unbiasedness-blanchet-glynn}]
    
    For the unbiasedness, we refer to Theorem 1 in \citet{blanchet2015unbiased}, and therefore only need to check the following assumptions:
    \begin{itemize}
        \item growth of $g$: since $Q$ is bounded and $g$ is locally Lipschitz, we can restrict $g$ to a closed interval on which it has linear growth;
        \item local differentiability: $g$ is clearly twice continuously differentiable in a neighborhood of $\mathbb E e^{-Q(s,A)/ \tau}$;
        \item finite 6th moment: since $Q$ is bounded, we clearly have $\mathbb E |e^{-Q(s,A)/ \tau}|^6 < \infty$.
    \end{itemize}
    Therefore $\tilde{T}^\theta Q(s)$ has finite variance and is indeed unbiased.
\end{proof}

Before proving Proposition \ref{proposition:contraction-prop-T-blanchet-glynn-l2},
we first present two  technical lemmas,
which will be used to prove the desired Lipschitz property of the Blanchet–Glynn estimator.  
The proofs are given in 
  Appendix \ref{appendix:proofs}. 
\begin{lemma}
    \label{lemma:technical-taylor-lemma}
    Let $g(x) = -\tau \log (x)$ for all $x>0$,
    and   $(a_n)_{n\in \N},(b_n)_{n\in \N},(a^\prime_n)_{n\in \N},(b^\prime_n)_{n\in \N}$ be sequences of real numbers with values in $[A,B]$, with $0 < A < B < \infty$, such that $\lim_{n\to \infty} a_n=\lim_{n\to \infty} b_n= m $, and $\lim_{n\to \infty} 
 a^\prime_n = \lim_{n\to \infty} b^\prime_n= m^\prime  $. Define $c_n \coloneqq \frac{a_n + b_n}{2}$ and $c^\prime_n \coloneqq \frac{a^\prime_n + b^\prime_n}{2}$
 for all $n\in \N$. Then
 for all $n\in \N$,
    \begin{equation}
        \label{eq:technical-taylor-lemma}
        \begin{aligned}
            D(a_n,b_n,a^\prime_n,b^\prime_n)\coloneqq                  & \left| g \left( c_n \right) - g \left( c^\prime_n \right) - \frac{1}{2} \left[ g(a_n) - g(a_n ^\prime) + g(b_n) - g(b^\prime_n)  \right] \right| \\
            \leq \tau C_2 \sum_{x \in \{a,b,c\}} & \big[(|x_n - x_n^\prime| + |m - m^\prime|)(x_n - m)^2
            \\
                                                                & + \left|(x_n - m + x^\prime_n - m^\prime)(x_n - x^\prime_n - (m - m^\prime))\right|\big],
        \end{aligned}
    \end{equation}
    where   $C_2 \coloneqq \max(C_1, A^{-2})$, with $C_1$ being a constant depending only on $A,B$ given by \eqref{eq:lipschitz-constant-c1}.
\end{lemma}

\begin{lemma}
    \label{lemma:lip-delta-fixed-n}
    Let $Q_1,Q_2 \in B_b(\mathcal S \times \mathcal A)$ such that $\alpha \leq Q_1,Q_2 \leq \beta$ for real constants $\alpha < \beta$. Let $s \in \mathcal S$. Then for all $\theta \in \Theta$, for all $N \in \N$,
    \begin{equation}
        \label{eq:lip-delta-fixed-n}
        \| \Delta_N^\theta Q_1(s) - \Delta_N ^\theta Q_2(s) \|^2_{L^2} \leq C_3 2^{-2N}\int_{\mathcal A}|Q_1(s,a) - Q_2(s,a)|^2 \mu (\diff a),
    \end{equation}
    where $C_3 > 0$ is a constant depending on $\alpha,\beta,\gamma$ given by \eqref{eq:lipschitz-constant-c(a,b,t)}.
\end{lemma}

We are now ready to prove the Lipschitz property of the Blanchet--Glynn estimator.

\begin{proof}[Proof of Proposition \ref{proposition:contraction-prop-T-blanchet-glynn-l2}]
   One can replicate the proof of Lemma \ref{lemma:lip-delta-fixed-n} with random $Q_1,Q_2$ and a random variable $S$ instead of $s$, which yields
    \[
        \| \Delta_N^\theta Q_1(S) - \Delta_N ^\theta Q_2(S) \|^2_{L^2} \leq C_{(\alpha, \beta, \tau)} 2^{-2N}\mathbb \|Q_1(S,A^{(\theta, 1)}) - Q_2(S, A^{(\theta, 1)})\|_{L^2}^2,
    \]
    where $C_{(\alpha,\beta,\gamma)} = C_3$.
    By the triangle inequality, we have
    \[
        \begin{aligned}
            \left\|\tilde{T}^\theta Q_1(S) - \tilde{T}^\theta Q_2(S) \right\|_{L^2} & \leq \left\|\frac{\Delta_{\tilde K^\theta}^\theta Q_1(S) - \Delta_{\tilde K ^ \theta} ^\theta Q_2(S)}{p(\tilde{K}^\theta)} \right\|_{L^2}   \\
            &+ \left\| Q_1(S,A^{(\theta, 0)}) - Q_2(S,A^{(\theta, 0)}) \right\|_{L^2}.
        \end{aligned}
    \]
    Now, to conclude the proof, we have by conditioning on $\tilde K^\theta$,
    \[
        \begin{aligned}
            \left\|\frac{\Delta_{\tilde K^\theta}^\theta Q_1(S) - \Delta_{\tilde K ^ \theta} ^\theta Q_2(S)}{p(\tilde{K}^\theta)} \right\|^2_{L^2} & = \sum_{N = 0}^{\infty} \frac{\| \Delta_N^\theta Q_1(S) - \Delta_N ^\theta Q_2(S) \|^2_{L^2}}{p(N)}                                   \\
            & \leq C_{(\alpha, \beta, \tau)} \|Q_1(S,A^{(\theta,1)}) - Q_2(S,A^{(\theta, 1)}) \|^2_{L^2} \sum_{N = 0}^{\infty} \frac{2^{-2N}}{p(N)}.
        \end{aligned}
    \]
    Since $\tilde K ^\theta$ is geometric with parameter $r$, we have
    $p(N) = r (1 - r)^{N}$, 
    and $2^{2N}p(N) \geq r 2^{N / 2}$, which shows that $\sum_{N = 0}^\infty \frac{2^{-2N}}{p(N)} < \infty$. Hence we have the desired Lipschitz property \eqref{eq:contraction-prop-T-blanchet-glynn-l2} with
    \begin{equation}
        \label{eq:lipschitz-constant-c4}
        L_{ (\alpha, \beta, \tau, r)} = 1 + \sqrt{C_{(\alpha, \beta, \tau)} \sum_{N = 0}^\infty \frac{1}{r(4(1-r))^N}} = 1 + \sqrt{C_{(\alpha, \beta, \tau)} \frac{4(1-r)}{3r - 4r^2}}.
    \end{equation}
\end{proof}

\section{Proofs of  Sample Complexities   for MLMC Estimators}

\label{section:analysis-of-the-complexity}
In this section, we perform a rigorous analysis of the sample complexity of the MLMCb estimator and derive a quasi-polynomial bound. We then show that the unbiasedness of the MLMCu estimator enables to get a polynomial sample complexity in expectation.

\subsection{Complexity of the MLMCb Estimator}
\label{subsection:specialising-complexity-plainMC}

\begin{proof}[Proof of Theorem \ref{theorem:error-complexity-MC}]
Observe that 
when using the plain Monte Carlo approximations $ \mathbf{T}_K = (\hat T^\theta_K)_{\theta \in \Theta}$, computing a realization of $\hat T^\theta_{K}Q(s)$ exactly requires  $K$ samples from  $\mu$ given by $A^{(\theta, 1)}, \allowbreak A^{(\theta, 2)}, \ldots, A^{(\theta,K)}$.
Let $\mathfrak C_{n,M, K}$ be the total number of independent random variables needed to compute a sample of $Q ^{0}_{n,M, \mathbf{T}_K}(s,a)$. We have for $M \geq 1$
\begin{equation}
    \label{eq:complexity-analysis}
    \begin{aligned}
        \mathfrak C_{n,M, K} & = M ^{n}(K + 1) + \sum_{l = 1}^{n-1} M ^{n - l}\left( 1
        + K(\mathfrak C _{l,M,K} + \mathfrak C _{l-1,M,K} + 1)\right)                                                                      \\
                            & = \sum_{l = 0}^n M^l (K + 1) + K\sum_{l = 1}^{n-1} M ^{n - l} (\mathfrak C _{l,M,K} + \mathfrak C_{l-1,M,K}) \\
                            & \leq (K+1) \frac{M^{n+1} - 1}{M - 1} + K(1 + M^{-1}) \sum_{l = 1}^{n-1} M^{n-l} \mathfrak C_{l,M,K},
    \end{aligned}
\end{equation}
which implies
\begin{equation*}
    \label{eq:complexity-upper-bound}
    M^{-n} \mathfrak C_{n,M, K} \leq (K + 1) \frac{M}{M-1} + K (1 + M^{-1}) \sum_{l = 0}^{n-1}M^{-l} \mathfrak C_{l,M,K}.
\end{equation*}
By the discrete Gronwall inequality of Lemma \ref{lemma:discrete-gronwall} we get
\begin{equation}
    \label{eq:gronwall-cost}
    \mathfrak C_{n,M,K} \leq (K + 1) \frac{M}{M-1} (1 + K (1 + M^{-1}))^n M^n.
\end{equation}

Then for all $n \in \N,M \geq 2, K \in \N^*$,
    \begin{equation}
    \label{eq:complexity-bound-MC-proof}
    \mathfrak C_{n,M,K} \leq 2(K + 1)(2 K)^n M^n \leq 2^{n+2}K^{n + 1}M^{n}.
    \end{equation}
     Using the upper bound on $\sigma_{\mathbf T_K}$ of Lemma \ref{lemma:additional-MC-bias} yields
    \[
\frac{3}{2}\max\left( \| Q_0 - Q^\star\|_{\infty},  2\gamma \| \sigma_{\mathbf T_K} \|_{\infty} \right) \leq \frac{3}{2} \max \left( \beta - \alpha, 2 \gamma \frac{L^\prime}{\sqrt{K}} \right) \leq \frac{3}{2} \max \left( \beta - \alpha, 2 \gamma L^\prime \right) = D.
    \]
    Now, taking $M \geq 2, n \in \N, K \in \N^*$ satisfying \eqref{eq:conditions-M-n-MC}, noticing that $(\sqrt{M} - \Lambda)^{-1}\sqrt{M} \leq (1 - \Lambda_M)^{-1}$ and applying Corollary \ref{cor:error-plain-monte-carlo} gives us the desired $\varepsilon$ error
    \[
    \sup_{(s,a) \in \cS \times \cA} \| Q^0_{n,M, \mathbf{T}_{K}}(s,a) - Q^\star(s,a) \|_{L^2} \leq \frac{3}{2}\left(\frac{\varepsilon}{3} + \frac{\varepsilon}{3}\right) =\varepsilon.
    \]
    Now, take $(M,n,K) = (M_0,n_\varepsilon,K_\varepsilon)$ with
    \[
    \begin{aligned}
        M_0 &= \left\lceil \left(\frac{\sqrt{\gamma} + \sqrt{\gamma + 4(1 - \gamma L)(1 + 2\gamma L)}}{2(1- \gamma L)} \right)^4 \right\rceil \\
        n_\varepsilon &= \left\lceil  \frac{\log \varepsilon - \log D}{\log \Lambda_{M_0}}\right\rceil, \quad  
        K_\varepsilon = \left\lceil 3 \frac{\gamma (L^\prime)^2}{2 \tau (1 - \Lambda_{M_0}) \varepsilon} \right\rceil
    \end{aligned}
    \]
    where $\Lambda_{M_0} = \gamma L + M_0^{-1/2}(1 + 2 \gamma L) + M_0^{-1/4}\sqrt{\gamma} < 1$. 
    As such, $M_0$ is just a function of $\gamma, \tau, c_{\min}, c_{\max}$.
    Let $\tilde{C}(M_0) \coloneqq  (2\tau(1 - \Lambda_{M_0}))^{-1}3\gamma (L^\prime)^2 $. 
    With that choice of $M,n,K$, the complexity bound \eqref{eq:complexity-bound-MC-proof} can be written as
    \[
    \mathfrak C_{n_0,M_0,K_0} \leq 2^{-l/l^\prime + 3}(\tilde{C}(M_0)+1)^{-l / l^\prime + 2}M_0^{-l / l^\prime + 1}
    \varepsilon^{- \log(\varepsilon) / l^\prime + \frac{2l + \log 2 + \log M_0 + \log (\tilde{C}(M_0) + 1)}{l^\prime} - 4},
    \]
    where $l = \log D, l^\prime = \log \Lambda_{M_{0}}$, and we have used the following upper bounds 
    \[
        n_{\varepsilon} \leq \frac{\log \varepsilon - \log D}{\log \Lambda_{M_0}} + 1, \quad K_{\varepsilon} \leq \varepsilon^{-1}\left(\tilde{C}(M_0) + 1\right),
    \]
    since $\varepsilon < 1$.
    This gives the  complexity bound \eqref{eq:sample-complexity-quasi-polynomial-MC} with the following constants
    \begin{equation}
        \label{eq:constants-K-kappa1-kappa2}
        \begin{aligned}
            C &= 2^{-l/l^\prime + 3}(\tilde{C}(M_0)+1)^{-l / l^\prime + 2}M_0^{-l / l^\prime + 1} > 0, \\
            \kappa & = 4 - \frac{2l + \log 2 + \log M_0 + \log (\tilde{C}(M_0) + 1)}{l^\prime} > 4.
        \end{aligned}
    \end{equation}
\end{proof}

\subsection{Complexity of the MLMCu Estimator}
\label{subsection:specialsing-complexity-BG}
The sample complexity now becomes a random variable due to the random sample size of $\tilde{T}^\theta$. Recall that $\tilde{K} + 1$ is a geometric random variable with parameter $r \in (1/2,3/4)$ such that $2^{\tilde K + 1}$ corresponds to the number of i.i.d.~copies of $\mu$ needed to compute a realization of $\tilde{T}^\theta Q (s)$ for any fixed $\theta,Q,s$. In particular, we need to handle the recursion in the cost carefully, as the sample complexity $\mathfrak C^\theta_{n,M}$ of the MLMCu estimator $Q^\theta_{n,M}$ is now stochastic.
Notice that the random variable $K^{(\theta,l,i)}$ is independent of both $\mathfrak C^{(\theta,l,i)} _{l,M}$ and $\mathfrak C ^{(\theta,-l,i)}_{l-1,M}$. In particular, letting $\mathfrak C_{n,M} = \mathbb E \mathfrak C^\theta_{n,M}$ and $K = \mathbb E 2^{\tilde K ^\theta + 1} + 1$, we recover \eqref{eq:complexity-analysis}. Now, $K$ does not depend on $n$, hence we get a polynomial bound in expectation. Notice how the choice of the parameter $r > 1/2$ is important here, since it ensures that $K = \mathbb E 2^{\tilde K + 1} = \sum_{n \geq 0}2^{n + 1}p(n) < \infty$.  

We are now ready to prove Theorem \ref{theorem:error-complexity-BG}.

\begin{proof}[Proof of Theorem \ref{theorem:error-complexity-BG}]
First, notice that, by choosing $Q_1 = Q_{0}$ and $Q_2 \equiv q \in [\alpha,\beta]$ in Proposition \ref{proposition:contraction-prop-T-blanchet-glynn-l2}, we have 
$ 
\sigma_{\mathbf{\tilde{T}}}(s,a) \leq L (\beta - \alpha),
$ 
which implies 
\[
\frac{3}{2}\max\left( \| Q_0 - Q^\star\|_{\infty},  2 \gamma \| \sigma_{\tilde{\mathbf T}} \|_{\infty} \right) \leq  \frac{3}{2}(\beta - \alpha) \max(1,2 \gamma L) = D. 
\]
    The error bound \eqref{eq:error-varepsilon} follows directly from \eqref{eq:final-MLMC-error-blanchet-Glynn} and the condition \eqref{eq:conditions-M-n} for $M$ and $ n$.
    For the complexity bound \eqref{eq:sample-complexity-bound-expectation}, given the definition of the Blanchet--Glynn estimator (Definition \ref{def:T-blanchet-glynn}), the number of independent variables needed to compute a realization of $\tilde{T}^\theta Q(s)$ for a fixed function $Q \in B_b(\mathcal S \times \mathcal A)$ is \(2^{\tilde{K}^\theta + 1} + 1\), therefore the total number of random variables one needs to sample in order to compute a realization of $Q^\theta_{n,M,\tilde{\mathbf T}}$,   denoted by $\mathfrak C_{n,M}^\theta$, is
    \begin{equation*}
        \label{eq:sample-complexity-def}
        \mathfrak C_{n,M}^\theta = \sum_{i = 1}^{M^n}\left(1 + 2^{\tilde K^{(\theta,0,i)} + 1} + 1\right) + \sum_{l = 1}^{n-1} \sum_{i = 1}^{M^{n-l}} \left( 1
        + \left(2^{\tilde K^{(\theta,l,i)} + 1} + 1 \right)\left(\mathfrak C^{(\theta,l,i)} _{l,M} + \mathfrak C ^{(\theta,-l,i)}_{l-1,M} + 1\right)\right).
    \end{equation*}
    Now, notice that Lemma \ref{lemma:theta-is-just-an-index}.(iii) implies that $K^{(\theta,l,i)}$ is independent of $\mathfrak C^{(\theta,l,i)} _{l,M}$ and $\mathfrak C ^{(\theta,-l,i)}_{l-1,M}$, therefore, by taking expectation on this sum of nonnegative quantities we get
    \[
        \begin{aligned}
            \mathbb E \mathfrak C_{n,M}^\theta & = \sum_{i = 1}^{M^n}\left(1 + \mathbb E \left[2^{\tilde K^{(\theta,0,i)} + 1} + 1\right] \right) \\
                                              & + \sum_{l = 1}^{n-1} \sum_{i = 1}^{M^{n-l}}\left( 1
            + \mathbb E\left[ 2^{\tilde K^{(\theta,l,i)} + 1} + 1\right]\left(\mathbb E \mathfrak C^{(\theta,l,i)} _{l,M} + \mathbb E\mathfrak C ^{(\theta,-l,i)}_{l-1,M} + 1\right)\right).
        \end{aligned}
    \]
    Let $K = \mathbb E\left[2^{\tilde K^\theta + 1} + 1\right] = \mathbb E\left[2^{\tilde K^0 + 1} + 1\right] = (2r - 1)^{-1}2r$. The condition $r > 1/2$ ensures that $K < \infty$, and notice that $K$ only depends on $r$. Hence, by writing $\mathfrak C_{n,M} = \mathbb E \mathfrak C^\theta_{n,M}$, we have
    \[
        \begin{aligned}
            \mathfrak C _{n,M} & = M ^{n}(K + 1) + \sum_{l = 1}^{n-1} M ^{n - l}\left( 1
            + K(\mathfrak C _{l,M} + \mathfrak C _{l-1,M} + 1)\right)                                                                    \\
                              & = \sum_{l = 0}^n M^l (K + 1) + K\sum_{l = 0}^{n-1} M ^{n - l} (\mathfrak C _{l,M} + \mathfrak C_{l-1,M}) \\
                              & \leq (K+1) \frac{M^{n+1} - 1}{M - 1} + K(1 + M^{-1}) \sum_{l = 0}^{n-1} M^{n-l} \mathfrak C_{l,M}.
        \end{aligned}
    \]
    We can apply Gronwall's inequality from Lemma \ref{lemma:discrete-gronwall} to get \eqref{eq:gronwall-cost}, i.e.,\
    \[
        \mathfrak C_{n,M} \leq (K + 1) \frac{M}{M-1} (1 + K (1 + M^{-1}))^n M^n \leq 2 \mathcal C_{\mathrm{num}}(r)^{n+1} M^{n},
    \]
    which yields \eqref{eq:sample-complexity-bound-expectation} with $\mathcal C_{\mathrm{num}}(r) \coloneqq (2r - 1)^{-1}4r = 2K$. Finally, we take $(M,n) = (M_0,n_\varepsilon)$ given by 
    \[
        M_0 = \left\lceil \left(\frac{\sqrt{\gamma} + \sqrt{\gamma + 4(1 - \gamma L)(1 + 2\gamma L)}}{2(1- \gamma L)} \right)^4 \right\rceil , \: n_\varepsilon = \left\lceil  \frac{\log \varepsilon - \log D}{\log \Lambda_{M_0}}\right\rceil,
    \]
    where $\Lambda_{M_0} = \gamma L + M_0^{-1}(1 + 2 \gamma L) + M_0^{-1/4}\sqrt{\gamma}$. As such, $M_0$ is just a function of $\gamma, \tau,r, c_{\min}, \\c_{\max}$.
    This leads to the average polynomial sample complexity bound \eqref{eq:sample-complexity-polynomial-complexity} with the following expressions for $C$ and $ \kappa$:
    \begin{equation}
    \label{eq:def-K-kappa}
        \begin{aligned}
            C & := 2 \mathcal C_{\mathrm{num}}(r) (M_0 \mathcal C_{\mathrm{num}}(r))^{\log D / \log \Lambda_{M_0}}, \quad 
            \kappa       := - \frac{\log( M_0 \mathcal C_{\mathrm{num}}(r))}{\log \Lambda_{M_0}}.
        \end{aligned}
    \end{equation}
\end{proof}

\section{Conclusion and Future Work}

In this paper, we propose several Monte Carlo (MC) algorithms for estimating the optimal Q-function of regularized MDPs with Polish state and action spaces, and establish their sample complexity guarantees   independently of the dimensions and cardinalities of the state and action spaces.

 We begin by proving that a simple iterative MC algorithm achieves quasi-polynomial sample complexity. 
 To improve performance, we introduce a general framework for constructing multilevel Monte Carlo (MLMC) estimators which combine fixed-point iteration, MLMC techniques, and a suitable stochastic approximation of the Bellman operator. 
 We quantify the $L^2$ error of the MLMC estimator in terms of the properties of the chosen approximate Bellman operator. 
 Building on this error analysis, we show that using a biased plain MC estimate for the Bellman operator results in an MLMC estimator that achieves a cubic reduction in sample complexity compared to the simple iterative MC estimator, though it still suffers from quasi-polynomial complexity due to the inherent bias. 
We then adapt a debiasing technique from \citet{blanchet2015unbiased} to construct an unbiased randomized multilevel approximation of the Bellman operator. 
 The resulting MLMC estimator achieves polynomial sample complexity in expectation, providing the first polynomial-time estimator for general action spaces. 
 Along the way, we also prove the Lipschitz continuity of the Blanchet--Glynn estimator with respect to the input random variable, which is a result of independent interest.

A natural extension of this work is to investigate efficient sampling strategies from the policy induced by the estimated Q-function. 
Moreover, the proposed estimators could be generalized to more complex settings, such as partially observable or mean-field MDPs. 
Finally, integrating MLMC techniques with other (deep) reinforcement learning approaches presents a promising direction for future research.

\section{Acknowledgements and Disclosure of Funding}
MM is supported by the EPSRC Centre for Doctoral
Training in Mathematics of Random Systems: Analysis, Modelling and Simulation (EP/S023925/1). CR is partially supported by EPSRC grant EP/Y028872/1. The authors are grateful for comments and suggestions made by Mike Giles and Justin Sirignano.

\appendix

\section{Discrete Gronwall Inequalities}
\label{appendix:discrete-gronwall}
As hinted by the recursive form of the estimator \eqref{eq:MLMC-estimator-approx-T}, we analyse its error and complexity by means of discrete Gronwall inequalities. One can refer to \citet{agarwal2000difference} for a detailed account of such inequalities. We state a version of it which we rely upon for the analysis of the complexity of the MLMC estimators.

\begin{lemma}[Discrete Gronwall inequality]
    \label{lemma:discrete-gronwall}
    Let $(u_n)_{n \in \N}$ be a real sequence. Let $0 \leq n_0 \leq n_1$ and let $b, w_0, \cdots, w_{n_1 - n_0 - 1} \geq 0$ such that for all $k \in \{n_0, n_0 + 1, \cdots, n_1\}$,
    \[
        u_k \leq b + \sum_{j = n_0}^{k - 1} w_{j - n_0} u_{j}.
    \]
    Then, for all $k \in \{n_0, n_0 + 1, \cdots, n_1\}$,
    \[
        u_k \leq b \prod_{j = n_0}^{k-1}(1 + w_{j - n_0}).
    \]
\end{lemma}

We also rely on a refined version of this inequality for the error of the general MLMC estimator. This result is a special case of Corollary 2.3 of \citet{hutzenthaler2022multilevel}.

\begin{lemma}[Refined Gronwall-type inequality] 
\label{lemma:refined-gronwall}
Let $N \in \N$, let $\left(a_n\right)_{0 \leq n \leq N},\lambda_1, \lambda_2, b_1,\\ b_2,b_3\in[0, \infty)$ satisfy for all $n \in \{0,\ldots, N\}$ that
$$
a_n \leq b_1 + b_2 b_3^n +\sum_{k=0}^{n-1}\left[\lambda_1 a_k+\lambda_2 a_{k-1}\right],
$$
where $a_{-1} = 0$.
Let $\Lambda=\frac{(1+\lambda_1)+\sqrt{(1+\lambda_1)^2+4 \lambda_2}}{2}$. 
Then for all $n \in \mathbb{N}$, we have 
$$
a_n \leq \begin{cases}\frac{3}{2} \Lambda^n b_1+\frac{3}{2} b_2 n \Lambda^n & \text { if } b_3=\Lambda, \\ 
\frac{3}{2} \Lambda^n b_1+\frac{3}{2} b_2 n \Lambda^n +\frac{3 b_2\left(b_3^{n+1}-b_3 \Lambda^n\right)}{2\left(b_3-\Lambda\right)} & \text { else. }\end{cases}
$$

\end{lemma}

\section{Proofs of Technical Lemmas}
\label{appendix:proofs}

\begin{proof}[Proof of Lemma \ref{lemma:boundedness-of-iterates}]
    Recall that $0 \leq \alpha = (1 - \gamma)^{-1}c_{\min}$ and $\beta = (1 - \gamma)^{-1}c_{\max}$. 
    Let $Q \in \mathcal B_b(\mathcal S \times \mathcal A)$ such that  $\alpha \leq Q  \leq \beta$.
    The lower bound follows easily from $TQ \geq 0$ and $c\ge c_{\min}$.
    For the upper bound, we have for all  $(s,s^\prime,a) \in \mathcal S \times \mathcal S \times \mathcal A$,
    \begin{align*}
        T Q (s^\prime) & = -\tau \log \int_{\mathcal A} \exp \left( - \tau^{-1}Q(s^\prime, a)\right) \mu (\diff a)                                                                           \\
                       & \leq -\tau \log \int_{\mathcal A} \exp \left(- \frac{c_{\max}}{\tau (1 - \gamma)}  \right) \mu (\diff a) = -\tau \log \exp \left(-\frac{c_{\max}}{\tau (1 - \gamma)} \right) = \frac{c_{\max}}{1 - \gamma},
    \end{align*}
    and hence
   $ 
        c(s,a) + \gamma TQ(s^\prime) \leq c_{\max} \left(1 + (1 - \gamma)^{-1}\gamma\right) = (1 - \gamma)^{-1}c_{\max}.
    $     
\end{proof}

\begin{proof}[Proof of Lemma \ref{lemma:contraction-prop-T-l2}]
    We have
    \begin{equation*}
        \label{eq:contraction-prop-T-l2-long}
        \begin{aligned}
            TQ_1(s) - TQ_2(s) & = - \tau \log \frac{\int_{\mathcal A} e^{-Q_1(s,a) / \tau} \mu (\diff a)}{\int_{\mathcal A} e ^ {- Q_2(s,a) / \tau} \mu (\diff a)}                  \leq \tau \left( \frac{\int_{\mathcal A} e ^ {- Q_2(s,a) / \tau} \mu (\diff a)} {\int_{\mathcal A} e^{- Q_1(s,a) / \tau} \mu (\diff a)}- 1 \right)     \\
                          & =\tau \frac{\int_{\mathcal A} (e^{-Q_2(s,a)/ \tau} - e^{-Q_1(s,a) / \tau}) \mu (\diff a)}{\int_{\mathcal A} e^{- Q_1(s,a) / \tau} \mu (\diff a)}         \\
                          & \leq \tau e^{(\beta - \alpha) / \tau}\int_{\mathcal A} \left( e^{-(Q_2(s,a) - \alpha) / \tau} - e^{-(Q_1(s,a) - \alpha) / \tau} \right) \mu (\diff a),
        \end{aligned}
    \end{equation*}
    where we have used the concavity of the logarithm $- \log (x) \leq x^{-1} - 1$, and $e^{-Q_1(s,a) / \tau} \geq e^{-\beta / \tau}$. Now, since $x \mapsto e^{-x}$ is 1-Lipschitz on $[0,\infty)$ and that $Q_1,Q_2 \geq \alpha$, we have
    \begin{equation*}
        \begin{aligned}
            \label{eq:contraction-prop-T-l2-short}
            TQ_1(s) - TQ_2(s) & \leq \tau e^{(\beta - \alpha) / \tau} \int_{\mathcal A}\left| \frac{Q_2(s,a) - Q_1(s,a)}{\tau} \right| \mu (\diff a) \\
                          & = e^{(\beta - \alpha)/ \tau} \int_{\mathcal A} |Q_1(s,a) - Q_2(s,a)| \mu (\diff a).
        \end{aligned}
    \end{equation*}
Performing the same computation for  $TQ_2(s) - TQ_1(s)$  yields the desired result.
\end{proof}

\begin{proof}[Proof of Lemma \ref{lemma:technical-taylor-lemma}]
    Recall that $g(x) = -\tau \log (x)$. By factoring out $\tau > 0$, we can assume without loss of generality that $\tau = 1$. 
    By a second-order Taylor expansion with Lagrange remainder, we get hold of $\xi_c,\xi^\prime_c$ such that
    \[
        \begin{aligned}
            g(c_n)        & = g(m) + g^\prime(m) (c_n - m) + \frac{g^{\prime \prime}(\xi_n^c)}{2}(c_n - m)^2,                                                     \\
            g(c^\prime_n) & = g(m^\prime) + g^\prime(m^\prime) (c^\prime_n - m^\prime) + \frac{g^{\prime \prime}(\xi_n^{c^\prime})}{2}(c^\prime_n - m^\prime)^2,
        \end{aligned}
    \]
    and $\xi_n^c \in [m, c_n], \xi_n^{c^\prime} \in [m^\prime, c_n^\prime]$. Similarly, we can get hold of $\xi_n^a \in [m,a_n], \xi_n^{a^\prime} \in [m^\prime, a^\prime_n], \xi_n^b \in [m,b_n], \xi_n^{b^\prime} \in [m^\prime, b^\prime_n]$. Now, plugging these Taylor expansions in the definition of $D(a_n,b_n,\\a^\prime_n,b^\prime_n)$, the first order terms cancel and we are left with
    \begin{equation}
        \label{eq:D-taylor-expanded}
        \begin{aligned}
            D(a_n,b_n,a^\prime_n,b^\prime_n) & = \frac{g^{\prime \prime}(\xi_n^c)}{2}(c_n - m)^2 - \frac{g^{\prime \prime}(\xi_n^{c^\prime})}{2}(c^\prime_n - m^\prime)^2                                  \\
                                             & - \frac{1}{2} \sum_{x \in \{a,b\}}\frac{g^{\prime \prime}(\xi_n^x)}{2}(x_n - m)^2 - \frac{g^{\prime \prime}(\xi_n^{x^\prime})}{2}(x^\prime_n - m^\prime)^2.
        \end{aligned}
    \end{equation}
    We focus our attention to one of the terms. We have
    \begin{equation}
        \label{eq:D-term-refactored}
        \begin{aligned}
            \frac{g^{\prime \prime}(\xi_n^c)}{2}(c_n - m)^2 - \frac{g^{\prime \prime}(\xi_n^{c^\prime})}{2}(c^\prime_n - m^\prime)^2 & = \frac{g^{\prime \prime}(\xi_n^c) - g^{\prime \prime}(\xi_n^{c^\prime})}{2}(c_n - m)^2             \\
        & - \frac{g^{\prime \prime}(\xi_n^{c^\prime})}{2}\left[(c^\prime_n - m^\prime)^2- (c_n - m)^2\right].
        \end{aligned}
    \end{equation}
    Now, notice how $g^{\prime \prime}(\xi^c_n)$ can be written as a function of $(c_n, m)$, more specifically let $\phi: \R_+^* \times \R_+^* \rightarrow \R$ be defined by
    \begin{equation*}
        \label{eq:definition-phi}
        \phi (x,y) = \begin{cases}
            \frac{g(x) - g (y) + g^\prime (y)(y - x)}{(y-x)^2} & \text{if } x \neq y \\
            \frac{g^{\prime \prime}(x)}{2}                      & \text{ otherwise}.
        \end{cases}
    \end{equation*}
    Notice that $g^{\prime \prime}(\xi_n^c) - g^{\prime \prime}(\xi_n^{c^\prime}) = 2 (\phi (c_n, m) - \phi (c_n^\prime, m^\prime))$. We aim to prove that $\phi$ is locally Lipschitz for the $L^1$ norm in $\mathbb R^2$. It is therefore sufficient to prove that $\nabla \phi$ is locally bounded, which we prove now.

    We now compute the partial derivative of $\phi$ with respect to $x$,
    \[
        \partial_x \phi (x,y) = \frac{(g^\prime(x) - g^\prime (y))(y - x)^2 + 2 (y - x)(g(x) - g(y) + g^\prime (y)(y - x))}{(y-x)^4}.
    \]
    The only critical points are when $x = y$. Fix $y$ and let's write a Taylor expansion of the numerator of $\partial_x \phi(x,y)$ as $x \rightarrow y$,
    \[
        \begin{aligned}
             & g^\prime(x) - g^\prime (y))(y - x)^2 + 2 (y - x)(g(x) - g(y) + g^\prime (y)(y - x))                                                                        \\
             & = -g^{\prime \prime}(y)(y - x)^3 + \frac{g^{(3)}(y)}{2} (y - x)^4  - 2g^{\prime}(y)(y - x)^2 + g^{\prime \prime }(y)(y - x)^3 - \frac{g^{(3)}(y)}{3}(y - x)^4 \\
             & + 2 g^\prime(y)(y - x)^2 + O((y - x)^5)                                                                                             \\
             & = \frac{g^{(3)}(y)}{6}(y) (y - x)^4 + O ((y - x)^5).
        \end{aligned}
    \]
    Hence since $g^{(3)}$ is locally bounded, it indeed shows that $\partial_x \phi$ is locally bounded. Similarly, we can show that $\partial_y \phi$ is locally bounded, proving that $\nabla \phi$ is locally bounded and that $\phi$ is indeed locally Lipschitz.
    Eventually, this shows that 
    \begin{equation}
        \label{eq:lip-phi}
        |g^{\prime \prime}(\xi_n^c) - g^{\prime \prime}(\xi_n^{c^\prime})| \leq C_1(|c_n - c_n^\prime| + |m - m^\prime|),
    \end{equation}
    for a constant $C_1 > 0$ given by 
    \begin{equation}
        \label{eq:lipschitz-constant-c1}
        C_1 \coloneqq \sup_{(x,y), (x^\prime, y^\prime) \in [A,B]^2} \frac{|\phi(x,y) - \phi(x^\prime, y^\prime)|}{|x - x^\prime| + |y - y^\prime|},
    \end{equation}
    depending only on $A,B$.
    Finally, notice that
    \begin{equation}
        \label{eq:a^2-b^2=(a+b)(a-b)}
        (c^\prime_n - m^\prime)^2- (c_n - m)^2 = (c_n - m + c^\prime_n - m^\prime)(c_n - c^\prime_n - (m - m^\prime)).
    \end{equation}
    Now, since $|g^{\prime \prime}(x)| = x^{-2} \leq A^{-2}$ on $[A,B]$, using \eqref{eq:lip-phi} and \eqref{eq:a^2-b^2=(a+b)(a-b)}, we can upper bound \eqref{eq:D-term-refactored} with
    \begin{equation*}
        \label{eq:D-term-upper-bound}
        \begin{aligned}
            \left| \frac{g^{\prime \prime}(\xi_n^c)}{2}(c_n - m)^2 - \frac{g^{\prime \prime}(\xi_n^{c^\prime})}{2}(c^\prime_n - m^\prime)^2 \right| \leq C_2(A,B) \big[(|c_n - c_n^\prime| + |m - m^\prime|)(c_n - m)^2 \\
                + \left|(c_n - m + c^\prime_n - m^\prime)(c_n - c^\prime_n - (m - m^\prime)) \right|\big].
        \end{aligned}
    \end{equation*}
    where $C_2(A,B) = \max(C_1(A,B), A^{-2})$ . Finally, one can derive the same bound for the terms corresponding to $a_n$ and $b_n$ in \eqref{eq:D-taylor-expanded}, hence using the triangle inequality yields \eqref{eq:technical-taylor-lemma}.
\end{proof}

\begin{proof}[Proof of Lemma \ref{lemma:lip-delta-fixed-n}]
    We fix $\theta \in \Theta$ and drop the $\theta$ superscript for clarity. Specifically, we write $A_k := A^{(\theta,k)}$. Then define, for arbitrary $Q^\prime \in B_b (\mathcal S \times \mathcal A)$
    \begin{align*}
        SQ^\prime(2^{N + 1})  &= \frac{1}{2^{N+1}}\sum_{k = 1}^{2^{N+1}} \exp \left(-Q^\prime(s,A_k) / \tau \right),  \\
        S_EQ^\prime(2^N)      &= \frac{1}{2^{N}}\sum_{k = 1}^{2^{N}} \exp \left(-Q^\prime(s,A_{2k}) / \tau \right),   \quad S_OQ^\prime (2^N)     = \frac{1}{2^{N}}\sum_{k = 1}^{2^{N}} \exp \left(-Q^\prime(s,A_{2k-1}) / \tau \right),
    \end{align*}
    $E/O$ referencing the even / odd indices used in the sum. Notice that for any $Q^\prime$, $SQ^\prime(2^{N + 1}) = \frac{1}{2}(S_EQ^\prime(2^{N}) + S_OQ^\prime(2^{N}))$. We now examine the difference $\Delta_N^\theta Q_1(s) - \Delta_N ^\theta Q_2(s)$ which we decompose in 3 terms
    \begin{equation*}
        \label{eq:delta-decomp}
        \Delta_N^\theta Q_1(s) - \Delta_N ^\theta Q_2(s) = D(N) - \frac{1}{2}(D_{E}(N) + D_{O}(N)),
    \end{equation*}
    where $D_N = g(SQ_1(2^{N + 1})) - g (SQ_2(2^{N+1}))$, $D_{E}(N) = g(S_EQ_1(2^N)) - g (S_EQ_2(2^N))$, and $D_{O}(N)$ is defined likewise.
    Since the sums at which $g$ is evaluated are lower-bounded by $A =e^{-\beta / \tau}$ and upper-bounded by $B = e^{-\alpha / \tau}$, we can apply Lemma \ref{lemma:technical-taylor-lemma} to $a_N = S_EQ_1(2^N), \quad b_N = S_OQ_1(2^N), \quad a^\prime_N = S_E Q_2(2^N), \quad b^\prime_N = S_O Q_2(2^N)$,
    which gives an upper bound like \eqref{eq:technical-taylor-lemma} with the constant given by $ \tau C_2(e^{-\beta / \tau}, e^{-\alpha / \tau})$.
    Notice that 
    $
    m = \int_{\mathcal A} e^{-Q_1(s,a)/ \tau} \mu (\diff a)$ 
    and 
    $\quad m^\prime = \int_{\mathcal A} e^{-Q_2(s,a)/ \tau} \mu (\diff a)
    $.
    We now show that each of the terms in the upper bound \eqref{eq:technical-taylor-lemma} can be bounded in squared $L^2$ norm by 
    $$
    \tau^{-2} 2^{-2N}D_{(\alpha,\tau)}\int_{\mathcal A} |Q_1(s,a) - Q_2(s,a)|^2 \mu (\diff a),
    $$ 
    with $D_{(\alpha,\tau)}$ denoting a numerical constant only depending on $\alpha$ and $\tau$. 

    \textbf{Term corresponding to $|a_N - a_N^\prime||a_N - m|^2$}.
    For clarity, let $q_k = Q_1(s,A_{2k}), r_k = Q_2(s,A_{2k})$. We have
    \[
        \begin{aligned}
            \mathbb E |a_N - a_N^\prime|^2|a_N - m|^4 & = \mathbb E \left( \frac{1}{2^N}\sum_{k = 1}^{2^N} e^{-q_k / \tau} - e^{-r_k/ \tau} \right)^2 \left( \frac{1}{2^N}\sum_{k = 1}^{2^N} e^{-q_k / \tau} - m \right)^4    \\
            & \leq \mathbb E \left( \frac{e^{-\alpha / \tau}}{\tau 2^N}\sum_{k = 1}^{2^N} |q_k - r_k| \right)^2 \left( \frac{1}{2^N}\sum_{k = 1}^{2^N} e^{-q_k / \tau} - m \right)^4 \\
            & = \frac{e^{-2\alpha / \tau}}{\tau^2 2^{6N}}  \sum_{k_1, \cdots, k_6} \mathbb E |q_{k_1} - r_{k_1}||q_{k_2} - r_{k_2}| \prod_{i = 3}^6 (e^{-q_{k_i} / \tau} - m),
        \end{aligned}
    \]
    where the sum is taken over all $k_1,k_2,k_3,k_4,k_5,k_6 \in \{ 1, \cdots, 2^N\}$. Let $K = 2^N$, we now aim to prove that at most $K^4$ (up to a numerical factor) of the terms in the sum are non zero. Notice that the non-zero terms correspond to $6$-tuples $(k_1, \ldots, k_6)$ such that for all $i \in \{3,4,5,6\}$, there exists $j \neq i, k_i = k_j$.
    It suffices to find an upper-bound of the cardinality of the set
    \[
        \mathfrak S_K := \{(k_1,\cdots, k_6) \in \N^6: 1 \leq k_i \leq K, \forall i \in \{3,4,5,6\}, \exists j \neq i, k_i = k_j \}.
    \]
    It is easy to see that any tuple $\mathbf k \in \mathfrak S_K$ can contain at most $4$ distinct integers. Therefore, we have \(| \mathfrak S _K | \leq 4^6 K^4\).
    Now, for any $k \in \mathfrak S_{K}$, we have
    \[
        \mathbb E |q_{k_1} - r_{k_1}||q_{k_2} - r_{k_2}| \prod_{i = 3}^6 (e^{-q_{k_i} / \tau} - m) \leq e^{-4\alpha/\tau} \mathbb E |q_1 - r_1|^2.
    \]
    Therefore,
    \begin{equation}
        \label{eq:lip-delta-fixed-n-first-bound}
        \begin{aligned}
            \mathbb E |a_N - a_N^\prime|^2|a_N - m|^4 & = \mathbb E \left( \frac{1}{2^N}\sum_{k = 1}^{2^N} e^{-q_k / \tau} - e^{-r_k/ \tau} \right)^2 \left( \frac{1}{2^N}\sum_{k = 1}^{2^N} e^{-q_k / \tau} - m \right)^4            \\
            & \leq \frac{e^{-2\alpha/\tau}}{\tau^2 2^{6N}}  \sum_{\mathbf k \in \mathfrak S_{2^N}} \mathbb E |q_{k_1} - r_{k_1}||q_{k_2} - r_{k_2}| \prod_{i = 3}^6 (e^{-q_{k_i} / \tau} - m) \\
            & \leq \frac{D_{(\alpha,\tau)} }{\tau^2 2^{2N}} \mathbb E |q_1 - r_1|^2 = \frac{D_{(\alpha,\tau)} }{\tau^2 2^{2N}} \int_{\mathcal A} |Q_1(s,a) - Q_2(s,a)|^2 \mu (\diff a),
        \end{aligned}
    \end{equation}
with $D_{(\alpha,\tau)} = 4^6 e^{-6\alpha/\tau}$ on the last line.

    \textbf{Term corresponding to $|m - m^\prime||a_N - m|^2$}.
    We have
    \begin{equation}
        \label{eq:lip-delta-fixed-n-second-bound}
        \begin{aligned}
            \mathbb E |m - m^\prime|^2|a_N - m|^4 & = |m - m^\prime|^2 \mathbb E \left( \frac{1}{2^N} \sum_{k = 1}^{2^N} e^{-q_k / \tau} - m \right)^4 \\
                                                  & \leq |m - m^\prime|^2 \frac{2^4 e^{-4\alpha / \tau}}{2^{2N}}                                        \\
                                                  & \leq \frac{D_{(\alpha,\tau)}}{\tau^2 2^{2N}} \int_{\mathcal A} |Q_1(s,a) - Q_2(s,a)|^2 \mu (\diff a),
        \end{aligned}
    \end{equation}
    with $D_{(\alpha,\tau)} = 2^4 e^{-6\alpha / \tau}$, where we have used $|m - m^\prime| \leq \tau^{-1}e^{-\alpha / \tau} \int |Q_1(s,a) - Q_2(s,a)| \mu (\diff a)$ and Jensen's inequality.

    \textbf{Term corresponding to $|a_N - m + a_N^\prime - m^\prime||a_N - a_N^\prime - (m - m^\prime)|$}.
    We have
    \[
        \begin{aligned}
             & \mathbb E |a_N - m + a_N^\prime - m^\prime|^2|a_N - a_N^\prime - (m - m^\prime)|^2                                                                                                                             \\
             & = \frac{1}{2^{4N}}\sum_{k_1,k_2,k_3,k_4} \mathbb E \prod_{i = 1,2}(e^{ -q_{k_i} / \tau} - m + e^{- r_{k_i} / \tau } - m^\prime) \prod_{i = 3,4} (e^{- q_{k_i} / \tau} - e^{- r_{k_i}/ \tau} - (m - m^\prime)),
        \end{aligned}
    \]
    where the sum is taken over all $k_1,k_2,k_3,k_4 \in \{ 1, \cdots, 2^N\}$. Notice that each of the factors within each product has $0$ expectation. It is then easy to see that at most $2^4 2^{2N}$ terms are non-zero in the sum. Moreover, we have the following bound
    \[
        \begin{aligned}
            \mathbb E \prod_{i = 1,2} & (e^{ -q_{k_i} / \tau} - m + e^{- r_{k_i} / \tau } - m^\prime) \prod_{i = 3,4} (e^{- q_{k_i} / \tau} - e^{- r_{k_i}/ \tau} - (m - m^\prime)) \\
                              & \leq 4 e^{-2\alpha / \tau} \mathbb E |e^{- q_{k_i} / \tau} - e^{- r_{k_i}/ \tau} - (m - m^\prime)|^2              \leq \frac{D_{(\alpha,\tau)}}{\tau^2} \mathbb E |q_1 - r_1|^2,
        \end{aligned}
    \]
    with $D_{(\alpha,\tau)} = 4 e^{-4\alpha / \tau}$.
    Therefore,
    \begin{equation}
        \label{eq:lip-delta-fixed-n-third-bound}
        \mathbb E |a_N - m + a_N^\prime - m^\prime|^2|a_N - a_N^\prime - (m - m^\prime)|^2  \leq \frac{D(\alpha, \beta)}{\tau^2 2^{2N}} \int_{\mathcal A}|Q_1(s,A) - Q_2(s,A)|^2 \mu (\diff a).
    \end{equation}

    The same bounds \eqref{eq:lip-delta-fixed-n-first-bound}, \eqref{eq:lip-delta-fixed-n-second-bound}, and \eqref{eq:lip-delta-fixed-n-third-bound} can be derived when replacing $a_N$ by $b_N$ and $c_N$. Therefore, combining these bounds and Lemma \ref{lemma:technical-taylor-lemma}, we get the desired result \eqref{eq:lip-delta-fixed-n} with the following Lipschitz constant
    \begin{equation}
        \label{eq:lipschitz-constant-c(a,b,t)}
        C_{(\alpha, \beta, \tau)} \coloneqq  3 C_2(e^{-\alpha / \tau}, e^{-\beta/ \tau}) \max(4^6 e^{-6\alpha/\tau}, 2^4 e^{-4\alpha / \tau}).
    \end{equation}
\end{proof}

\section{Implementations of Plain Monte Carlo and Blanchet--Glynn Estimators}
\label{app:algos}

Algorithm \ref{alg:MLMC} requires subroutines $T_{\mathrm{approx}}$ and $DT_{\mathrm{approx}}$ for approximating the soft-Bellman operator.   Algorithms \ref{alg:naive-mc} and \ref{alg:bg}  implement the plain Monte Carlo and Blanchet--Glynn approximations, corresponding to MLMCb and MLMCu estimators, respectively.

\begin{algorithm}
    \caption{Approximation of $T$ based on plain Monte Carlo average}\label{alg:naive-mc}
    \begin{algorithmic}
        \Require{$\tau > 0, K \in \N^*, \mu \in \mathcal P (A)$},
        \Procedure{$T_{\mathrm{MC}}$}{$Q,s$}
        \State{generate $K$ i.i.d.~samples from $\mu$: $A_1, \ldots, A_K$}
       \State{$\hat{S} \gets 
       \frac{1}{K}\sum_{k=1}^K 
 \exp(-Q(s,A_k) / \tau)$}
        \State\Return $-\tau \log \hat{S}$
        \EndProcedure
        \Procedure{$DT_{\mathrm{MC}}$}{$Q_1, Q_2,s$}
        \State{generate $K$ i.i.d.~samples from $\mu$: $A_1, \ldots, A_K$}
        \State{$\hat{S}_1 \gets   
        \frac{1}{K}\sum_{k=1}^K 
        \exp(-Q_1(s,A_k) / \tau)$}
        \State{$\hat{S}_2 \gets   
        \frac{1}{K}\sum_{k=1}^K 
        \exp(-Q_2(s,A_k) / \tau)$}
         \State\Return $-\tau \log \hat{S}_1  + \tau \log \hat{S}_2$
        \EndProcedure
    \end{algorithmic}
\end{algorithm}

\begin{algorithm}
    \caption{Approximation of $T$ based on the Blanchet--Glynn estimator }\label{alg:bg}
    \begin{algorithmic}
        \Require{$\tau > 0, r \in (1/2, 3/4), \mu \in \mathcal P (A)$},
        \Procedure{$T_{\mathrm{BG}}$}{$Q,s$}
        \State{generate $K \sim \mathrm{Geometric}(r)$}
        \State{$p_K \gets r (1 -r)^{K}$}
        \State{generate $2^{K} + 1$ i.i.d.~samples from $\mu$: $A_0, A_1, \ldots, A_{2^K}$}
        \State{$\hat{S}_E \gets  
        \frac{1}{2^{K - 1}}\sum_{k = 1}^{2^{K-1}}\exp(-Q(s,A_{2k}) / \tau)$}
        \State{$\hat{S}_O \gets   
        \frac{1}{2^{K - 1}}\sum_{k = 1}^{2^{K-1}}\exp(-Q(s,A_{2k - 1}) / \tau)$}
        \State{$\hat{S} \gets \frac{\hat{S}_E  + \hat{S}_O}{2}$}
        \State\Return $\frac{1}{p_K}\left(-\tau \log \hat{S} - \frac{1}{2}(-\tau \log \hat{S}_E  - \tau \log \hat{S}_O)\right) + Q(s, A_0)$
        \EndProcedure
        \Procedure{$DT_{\mathrm{BG}}$}{$Q_1, Q_2,s$}
        \State{generate $K \sim \mathrm{Geometric}(r)$}
        \State{$p_K \gets r (1 -r)^{K}$}
        \State{generate $2^{K} + 1$ i.i.d.~samples from $\mu$: $A_0, A_1, \ldots, A_{2^K}$}
        \For{$i = 1,2$}
       \State{$\hat{S}_{E,i} \gets 
        \frac{1}{2^{K - 1}}\sum_{k = 1}^{2^{K-1}}\exp(-Q_i(s,A_{2k}) / \tau)$}
        \State{$\hat{S}_{O,i} \gets 
        \frac{1}{2^{K - 1}}\sum_{k = 1}^{2^{K-1}}\exp(-Q_i(s,A_{2k - 1}) / \tau)$}
        \State{$\hat{S}_i \gets \frac{\hat{S}_{E,i} + \hat{S}_{O,i}}{2}$}
        \State{$t_i \gets \frac{1}{p_K}\left(-\tau \log \hat{S}_i  - \frac{1}{2}(-\tau \log \hat{S}_{E,i} - \tau \log \hat{S}_{O,i})\right) + Q_i(s, A_0)$}
        \EndFor
        \State\Return $t_1 - t_2$
        \EndProcedure
    \end{algorithmic}
\end{algorithm}

\section{Additional Numerical Results}
\label{appendix:numerical-results}

In Figures \ref{fig:gamma-06-avg-q-all} and \ref{fig:gamma-06-rmse-time} we present numerical results for the linear quadratic Gaussian control problem presented in Section \ref{subsection:experimental-setup} for $d=20, \gamma = 0.6$. We clearly see the numerical instability of the MLMCu estimator, whereas the MLMCb estimator behaves as expected. For $r = 0.6$, only the level $6$ estimate is unstable, for $r = 1 - 2^{-3/2}$ both level $5$ and $6$ are unstable, which further confirms Remark \ref{remark:r}.
\begin{figure}
    \centering
    \includegraphics[width=0.8\linewidth]{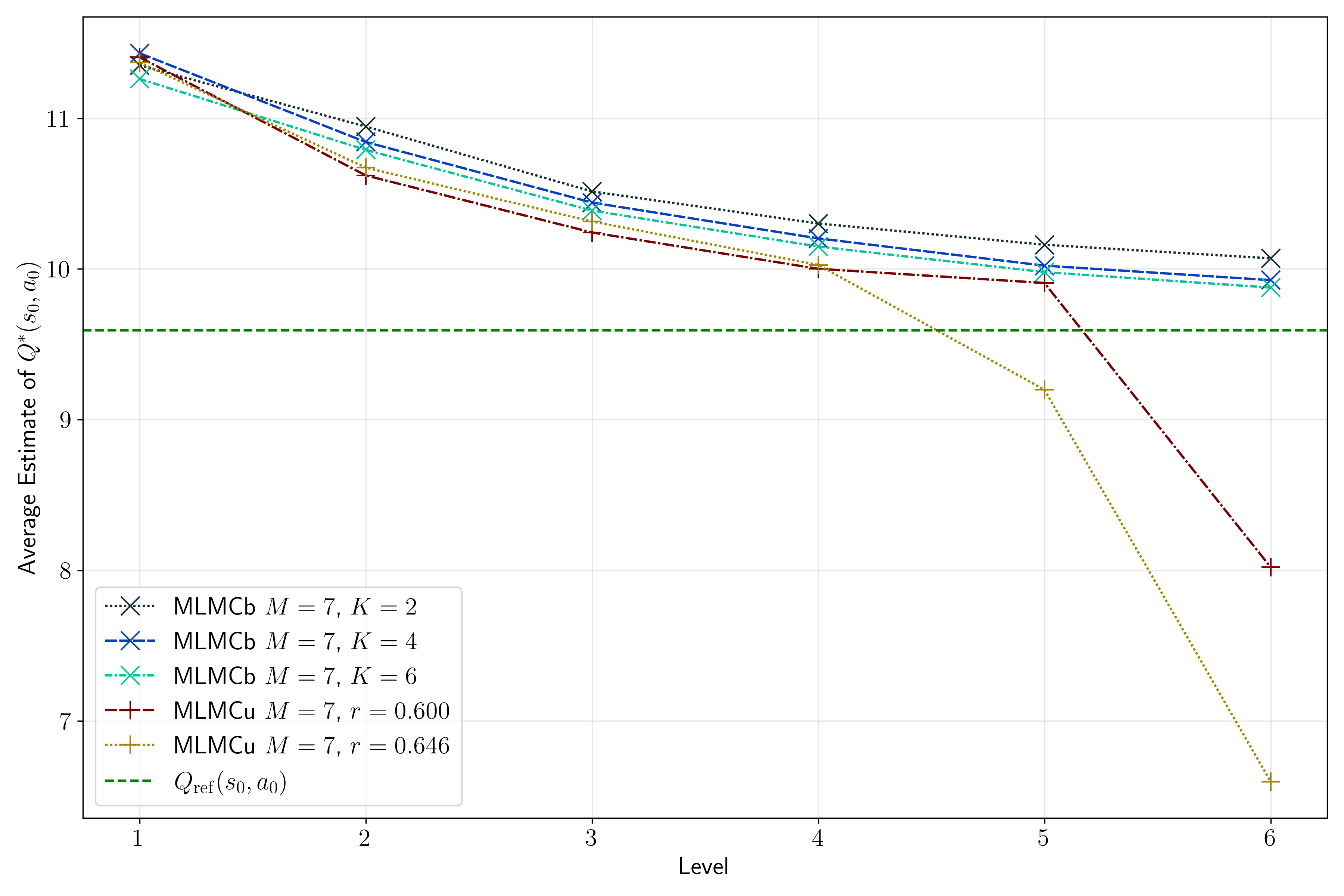}
    \caption{Average estimate of $Q^\star(s_0,a_0)$ over 20 runs for $d=20, \gamma=0.6$.}
    \label{fig:gamma-06-avg-q-all}
\end{figure}

\begin{figure}
    \centering
    \begin{subfigure}[b]{0.49\linewidth}
        \centering
        \includegraphics[width=\linewidth]{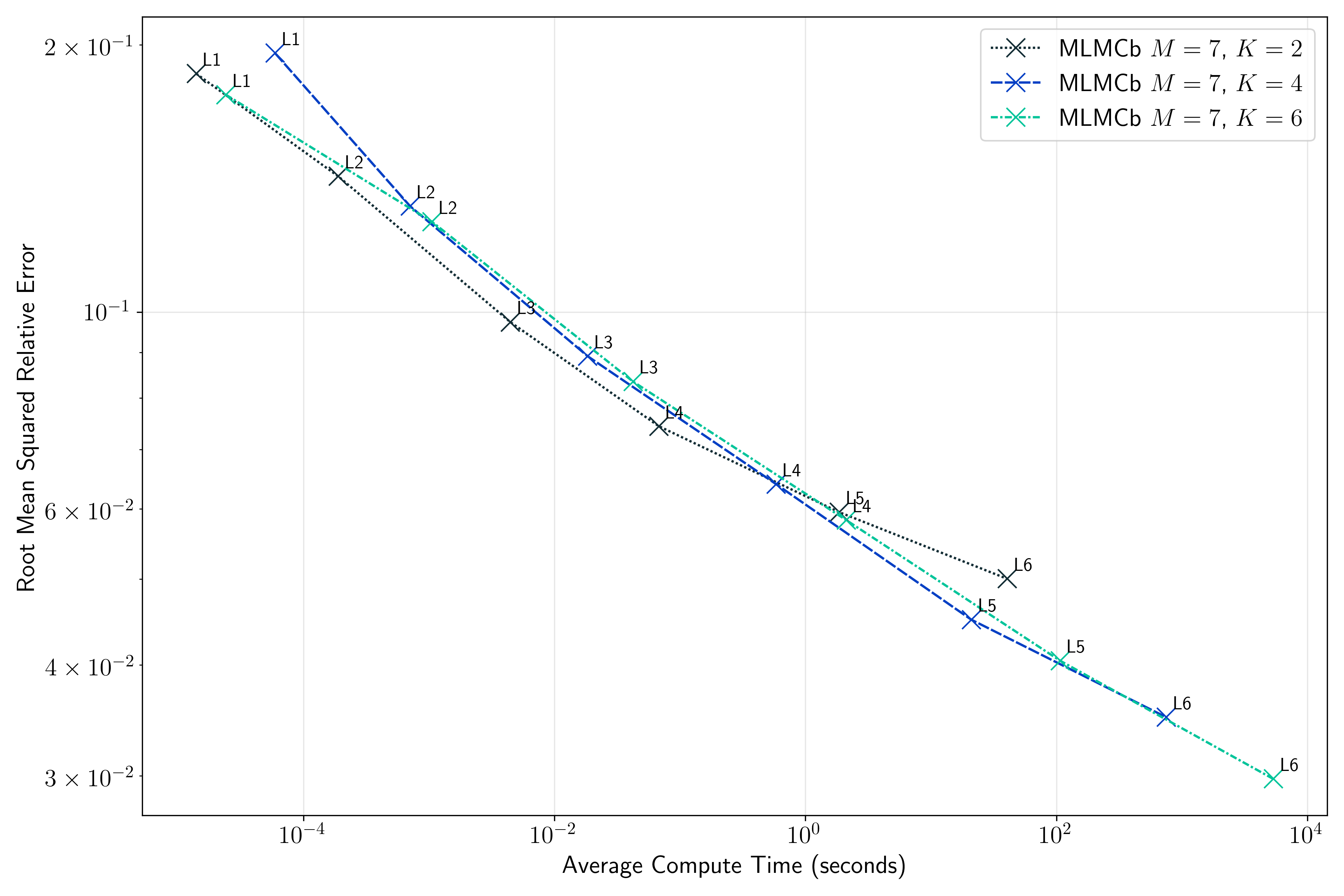}
        \caption{Plain Monte Carlo (MLMCb).}
        \label{fig:gamma-06-time-error-mc}
    \end{subfigure}
    \hfill
    \begin{subfigure}[b]{0.49\linewidth}
        \centering
        \includegraphics[width=\linewidth]{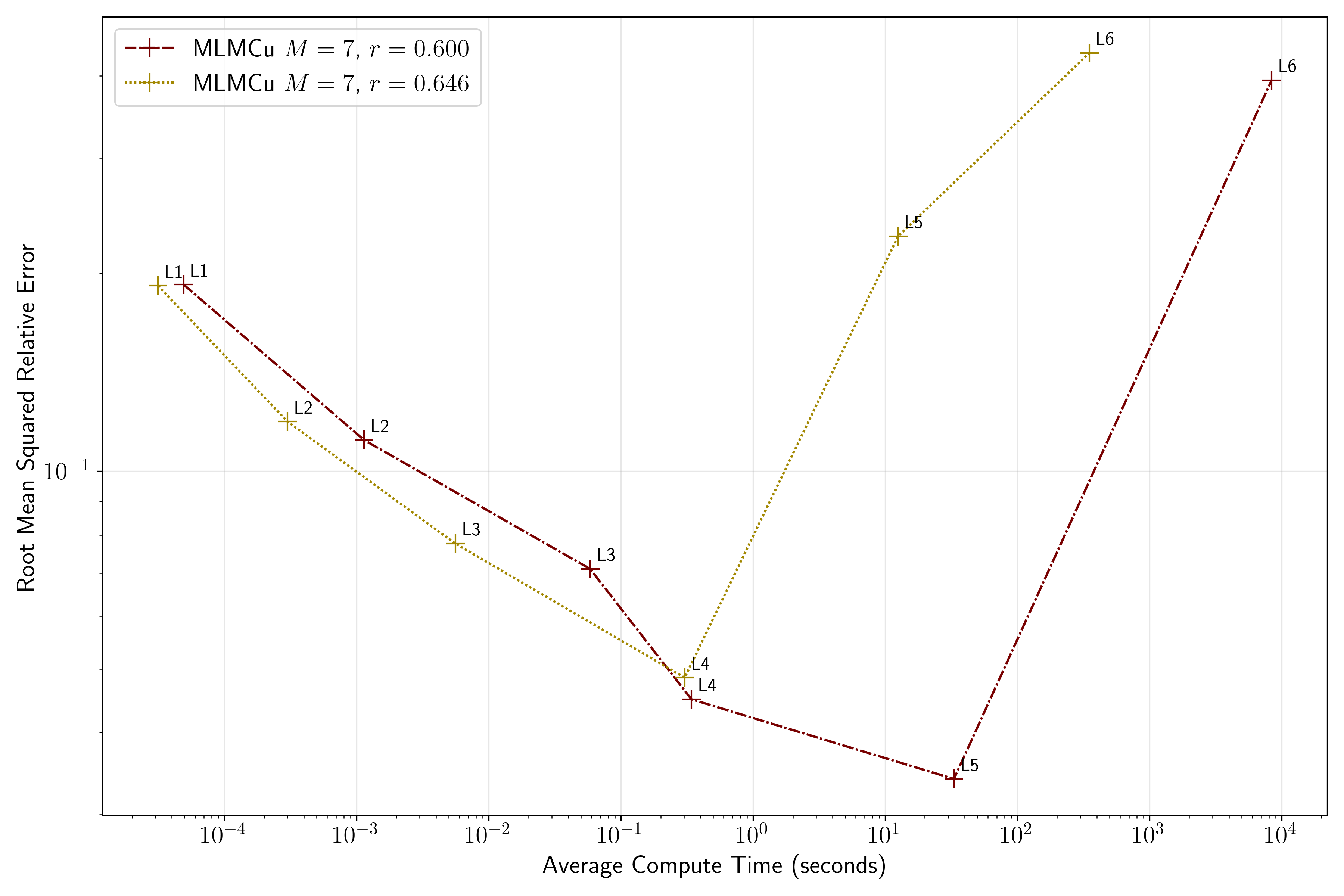}
        \caption{Unbiased Monte Carlo (MLMCu).}
        \label{fig:gamma-06-time-error-bg}
    \end{subfigure}
    \caption{RMSRE as a function of average compute time over 20 runs for $d=20, \gamma = 0.6$.}
    \label{fig:gamma-06-rmse-time}
\end{figure}

In Tables \ref{tab:results_gamma04}, \ref{tab:results_gamma05} and \ref{tab:results_gamma06}, we display the numerical results of all the considered configurations of the MLMC estimator in the entropy-regularized LGQ problem presented in Section \ref{subsection:experimental-setup} with $d = 20$ and for three problem settings with values of the discount factor $\gamma \in \{0.4,0.5,0.6\}$.

\begin{table}[htbp]
\centering
\begin{tabular}{
p{2cm}
p{2cm}
p{3cm}
p{3cm}
p{3cm}
}
\toprule
Estimator type & Parameter value &  Average compute time (seconds) & RMSRE & Average estimate of $Q^\star(s_0,a_0)$ \\
\midrule
\multirow{3}{*}{MC, $M = 7$} & $K = 2$ & 4.086e+01 & 0.0154 & 3.983\\
 & $K = 4$ & 7.892e+02 & 0.00869 & 3.957 \\
 & $K = 6$ & 5.381e+03 & 0.00655& 3.948\\
\midrule
\multirow{2}{*}{BG, $M = 7$} & $r = 0.6$ & 1.392e+04 & 0.00325 & 3.929 \\
 & $r = 1 - \frac{1}{2^{3/2}}$ & 1.015e+03 & 0.00392 & 3.935 \\
\bottomrule
\end{tabular}
\caption{Level $n=6$ MLMC perfomance, $d=20, \gamma=0.4$, reference value is $Q_{\mathrm{ref}} (s_0, a_0) = 3.923$ (MC corresponds to MLMCb, BG corresponds to MLMCu).}
\label{tab:results_gamma04}


\begin{tabular}{
p{2cm}
p{2cm}
p{3cm}
p{3cm}
p{3cm}
}
\toprule
Estimator type & Parameter value &  Average compute time (seconds) & RMSRE & Average estimate of $Q^\star(s_0,a_0)$ \\
\midrule
\multirow{3}{*}{MC, $M = 7$} & $K = 2$ & 4.121e+01 & 0.0284 & 6.110\\
 & $K = 4$ & 7.779e+02 & 0.0175 & 6.046 \\
 & $K = 6$ & 5.377e+03 & 0.0143 & 6.026 \\
\midrule
\multirow{2}{*}{BG, $M = 7$} & $r = 0.6$ & 3.751e+03 & 0.00851 & 5.983 \\
 & $r = 1 - \frac{1}{2^{3/2}}$ & 1.481e+03 & 0.317 & 5.290 \\
\bottomrule
\end{tabular}
\caption{Level $n=6$ MLMC perfomance, $d=20, \gamma=0.5$, reference value is $Q_{\mathrm{ref}} (s_0, a_0) = 5.942$ (MC corresponds to MLMCb, BG corresponds to MLMCu).}
\label{tab:results_gamma05}


\begin{tabular}{
p{2cm}
p{2cm}
p{3cm}
p{3cm}
p{3cm}
}
\toprule
Estimator type & Parameter value &  Average compute time (seconds) & RMSRE & Average estimate of $Q^\star(s_0,a_0)$ \\
\midrule
\multirow{3}{*}{MC, $M = 7$} & $K = 2$ & 4.056e+01 & 0.0500 & 10.071\\
 & $K = 4$ & 7.527e+02 & 0.0349 & 9.926 \\
 & $K = 6$ & 5.357e+03 & 0.0298 & 9.876\\
\midrule
\multirow{2}{*}{BG, $M = 7$} & $r = 0.6$ & 8.391e+03 & 0.393 & 8.021  \\
 & $r = 1 - \frac{1}{2^{3/2}}$ & 3.502e+02 & 0.433 & 6.598 \\
\bottomrule
\end{tabular}
\caption{Level $n=6$ MLMC perfomance, $d=20, \gamma=0.6$, reference value is $Q_{\mathrm{ref}} (s_0, a_0) = 9.591$ (MC corresponds to MLMCb, BG corresponds to MLMCu).}
\label{tab:results_gamma06}
\end{table}

\bibliographystyle{plain}
\bibliography{references}
\end{document}